\documentclass[twoside,11pt]{article}

\usepackage{jmlr2e}

\usepackage[utf8]{inputenc}
\usepackage[T1]{fontenc}
\usepackage{amsmath}          
\usepackage{mathtools}
\usepackage{bm}
\usepackage[caption=false]{subfig}
\usepackage{booktabs}
\usepackage{multirow}
\usepackage{array}
\usepackage{xcolor}
\usepackage{listings}
\usepackage{algorithm}
\usepackage{algorithmic}
\usepackage{float}
\usepackage{etoolbox}
\usepackage{lastpage}
\usepackage{tcolorbox}
\tcbuselibrary{skins,breakable,listings}
\usepackage{tikz}
\usepackage{pgfplots}
\pgfplotsset{compat=1.17}
\usetikzlibrary{
    shapes.geometric, 
    arrows.meta, 
    positioning, 
    calc, 
    decorations.pathmorphing, 
    patterns, 
    shadows.blur, 
    backgrounds, 
    fit, 
    matrix,
    fadings,
    shadings,
    decorations.markings
}

\definecolor{NatureBlue}{HTML}{0072B2}
\definecolor{NatureOrange}{HTML}{E69F00}
\definecolor{NatureGreen}{HTML}{009E73}
\definecolor{NatureRed}{HTML}{D55E00}
\definecolor{NaturePurple}{HTML}{CC79A7}
\definecolor{NatureCyan}{HTML}{56B4E9}
\definecolor{NatureYellow}{HTML}{F0E442}
\definecolor{DarkGray}{HTML}{333333}
\definecolor{LightGray}{HTML}{F5F5F5}
\definecolor{MediumGray}{HTML}{888888}

\definecolor{codebg}{HTML}{F8FAFD}
\definecolor{codeframe}{HTML}{C8D1DC}
\definecolor{codehead}{HTML}{144A75}
\definecolor{codeheadtext}{HTML}{FFFFFF}
\definecolor{codekw}{HTML}{0B5394}
\definecolor{codestring}{HTML}{A31515}
\definecolor{codecomment}{HTML}{2F7D32}
\definecolor{codenumber}{HTML}{6E7781}
\definecolor{codeidentifier}{HTML}{1F2328}
\definecolor{codedecorator}{HTML}{7A3E9D}
\lstdefinestyle{mystyle}{
	language=Python,
	backgroundcolor=\color{codebg},
	basicstyle=\ttfamily\small,
	identifierstyle=\color{codeidentifier},
	keywordstyle=\bfseries\color{codekw},
	keywordstyle=[2]\bfseries\color{codedecorator},
	commentstyle=\itshape\color{codecomment},
	stringstyle=\color{codestring},
	numberstyle=\scriptsize\color{codenumber},
	numbers=left,
	stepnumber=1,
	numbersep=10pt,
	xleftmargin=0.8em,
	frame=none,
	breaklines=true,
	breakatwhitespace=false,
	postbreak=\mbox{\textcolor{NatureBlue}{$\hookrightarrow$}\space},
	keepspaces=true,
	showspaces=false,
	showstringspaces=false,
	showtabs=false,
	tabsize=4,
	aboveskip=0.2em,
	belowskip=0.2em,
	columns=fullflexible,
	upquote=true,
	morekeywords=[2]{self,super,torch,nn,F},
	emph={RMNDirect,get_exponents,log_primitive,forward},
	emphstyle=\bfseries\color{NatureRed}
}
\newtcblisting{pythonlisting}{
	enhanced,
	breakable,
	notitle,
	colback=codebg,
	colframe=codeframe,
	boxrule=0.8pt,
	arc=2pt,
	outer arc=2pt,
	left=6pt,
	right=6pt,
	top=6pt,
	bottom=6pt,
	listing only,
	listing options={style=mystyle,language=Python}
}
\newtcolorbox{theoremblock}{
    enhanced,
    breakable,
    colback=LightGray,
    colframe=LightGray,
    boxrule=0pt,
    sharp corners,
    left=0pt,
    right=0pt,
    top=0pt,
    bottom=0pt,
    boxsep=0pt,
    before skip=6pt,
    after skip=6pt
}
\BeforeBeginEnvironment{theorem}{\begin{theoremblock}}
\AfterEndEnvironment{theorem}{\end{theoremblock}}
\BeforeBeginEnvironment{lemma}{\begin{theoremblock}}
\AfterEndEnvironment{lemma}{\end{theoremblock}}
\BeforeBeginEnvironment{proposition}{\begin{theoremblock}}
\AfterEndEnvironment{proposition}{\end{theoremblock}}
\BeforeBeginEnvironment{corollary}{\begin{theoremblock}}
\AfterEndEnvironment{corollary}{\end{theoremblock}}
\BeforeBeginEnvironment{definition}{\begin{theoremblock}}
\AfterEndEnvironment{definition}{\end{theoremblock}}
\BeforeBeginEnvironment{remark}{\begin{theoremblock}}
\AfterEndEnvironment{remark}{\end{theoremblock}}
\BeforeBeginEnvironment{example}{\begin{theoremblock}}
\AfterEndEnvironment{example}{\end{theoremblock}}

\newcommand{\R}{\mathbb{R}}
\newcommand{\N}{\mathbb{N}}

\newcommand{\E}{\mathbb{E}}
\newcommand{\bx}{\bm{x}}
\newcommand{\bc}{\bm{c}}

\newcommand{\norm}[1]{\left\|#1\right\|}

\DeclareMathOperator{\softplus}{softplus}
\renewcommand{\eqref}[1]{Eq.~\ref{#1}}

\hypersetup{
    hypertexnames=false,
    pdftitle={Radial Müntz-Szász Networks: Neural Architectures with Learnable Power Bases for Multidimensional Singularities},
    pdfauthor={N'guessan, Kim}
}

\newtcolorbox{keyinsight}{
    enhanced,
    breakable,
    colback=NatureBlue!5,
    colframe=NatureBlue,
    boxrule=1pt,
    left=10pt,
    right=10pt,
    top=8pt,
    bottom=8pt,
    fonttitle=\bfseries,
    title=Key Insight
}

\newtcolorbox{mainresult}{
    enhanced,
    breakable,
    colback=NatureGreen!8,
    colframe=NatureGreen!80!black,
    boxrule=1.5pt,
    left=10pt,
    right=10pt,
    top=8pt,
    bottom=8pt
}

\jmlrheading{0}{2026}{1-\pageref{LastPage}}{3/26}{N/A}{00-0000}{Gnankan Landry Regis N'guessan and Bum Jun Kim}

\ShortHeadings{Radial M\"untz-Sz\'asz Networks}{N'guessan and Kim}
\firstpageno{1}

\begin{document}

\title{Radial M\"untz-Sz\'asz Networks: Neural Architectures with Learnable Power Bases for Multidimensional Singularities}

\author{\name Gnankan Landry Regis N'guessan \email rnguessan@aimsric.org \\
	\addr Axiom Research Group\\
	Department of Applied Mathematics and Computational Science, NM-AIST, Tanzania\\
	African Institute for Mathematical Sciences (AIMS), Research and Innovation Centre, Rwanda
	\AND
	\name Bum Jun Kim\thanks{Corresponding author} \email bumjun.kim@weblab.t.u-tokyo.ac.jp \\
	\addr Graduate School of Engineering, The University of Tokyo, Japan}

\editor{To be assigned}

\maketitle

\begin{abstract}
		Radial singular fields---such as $1/r$, $\log r$, and crack-tip profiles---are difficult to model with current coordinate-separable neural architectures. We formally establish this result: any $C^2$ function that is both radial and additively separable must be quadratic, establishing a fundamental obstruction for coordinate-wise power-law models. Motivated by this result, we introduce Radial M\"untz-Sz\'asz Networks (RMN), which represent fields as linear combinations of learnable radial powers $r^\mu$, including negative exponents, together with a limit-stable log-primitive for exact $\log r$ behavior. RMN admits closed-form spatial gradients and Laplacians, enabling physics-informed learning on punctured domains. Across ten 2D and 3D benchmarks, RMN achieves between 1.5 and 51 times lower RMSE than MLPs and between 10 and 100 times lower RMSE than SIREN, while using only 27 parameters, compared with 33,537 for MLPs and 8,577 for SIREN. We extend RMN to incorporate angular dependence (RMN-Angular) and to handle multiple sources with learnable centers (RMN-MC), whose source-center recovery errors fall below $10^{-4}$. We also report controlled failures on smooth, strongly non-radial targets to delineate RMN's operating regime.
\end{abstract}


\section{Introduction}
\label{sec:introduction}

Many physical singularities in multiple spatial dimensions are radial: they depend primarily on the distance from a source. For Laplace's equation in $\R^d$, the fundamental solution scales as $\norm{\bx}^{2-d}$ for $d \geq 3$ and as $\log\norm{\bx}$ for $d = 2$ \citep{evans2010partial}. This distance-to-source scaling underlies the Coulomb and gravitational potentials, line-charge fields, and fracture-mechanics crack-tip asymptotics, where stress fields scale as $r^{-1/2}$ while displacement fields scale as $r^{1/2}$ \citep{williams1957stress}. Boundary layers in fluid dynamics can exhibit similarly steep, fractional-power behavior.

These radial singularities pose fundamental challenges for neural network approximation. Standard multilayer perceptrons (MLPs) with smooth activation functions exhibit a spectral bias toward low-frequency components \citep{rahaman2019spectral}, so capturing singular behaviors such as $1/r$ or $\log r$ typically requires large models and long training. In our benchmarks, a 4-layer, width-128 MLP with 33,537 parameters approaches the accuracy of Radial M\"untz-Sz\'asz Networks, which achieve comparable accuracy with only 27 parameters. Sinusoidal representation networks (SIREN) \citep{sitzmann2020siren} broaden frequency support for smooth implicit representations, but they still lack an explicit mechanism for power-law divergence near singularities and require 8,577 parameters to reach accuracy comparable to RMN's 27.

\begin{figure}[H]
	\centering
	\includegraphics[width=\textwidth]{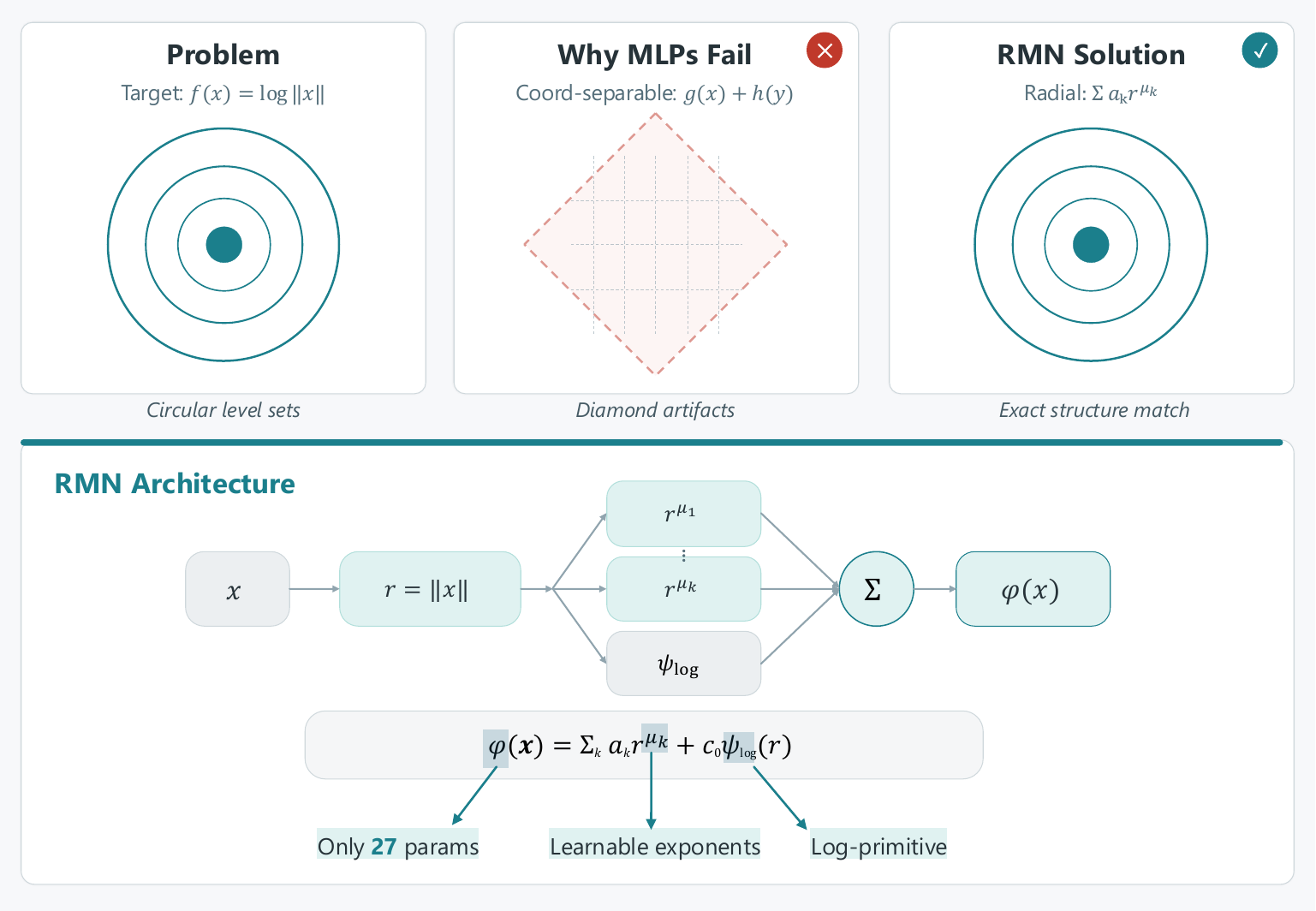}
	\caption{Graphical Abstract. Standard coordinate-separable networks fail on radial singularities due to structural mismatch. RMN directly parameterizes the radial power structure with learnable exponents, achieving orders-of-magnitude improvements with far fewer parameters.}
	\label{fig:graphical_abstract}
\end{figure}

In recent work, we introduced M\"untz-Sz\'asz Networks (MSN) \citep{nguessan2025msn}, which replace fixed activation functions with learnable power-law bases. The classical M\"untz-Sz\'asz theorem \citep{muntz1914, szasz1916} characterizes when systems of monomials $\{x^{\mu_k}\}$ form complete bases in continuous function spaces, providing a theoretical foundation for power-law neural architectures. MSN improves approximation for one-dimensional singular functions. However, extending MSN to multiple dimensions by applying it coordinate-wise requires an additively separable model:
\begin{equation}
	\phi_{\text{MSN-coord}}(\bx) = \sum_{i=1}^{d} \sum_{k=1}^{K} a_{ik} |x_i|^{\mu_k}.
\end{equation}
Such separable functions fundamentally cannot represent radial structure. Specifically, we prove a separability obstruction theorem: any $C^2$ function that is both radial and additively separable must be quadratic (Theorem~\ref{thm:separability}). This explains why coordinate-wise MSN systematically fails on radial singularities by factors of 72$\times$ to 1,652$\times$, regardless of capacity.

This observation motivates RMN, which directly parameterizes radial power functions:
\begin{equation}
	\phi_{\text{RMN}}(\bx) = \sum_{k=1}^{K} a_k \norm{\bx}^{\mu_k} + c_0 \psi_{\log}(r;\mu_{\log}) + b_0,
\end{equation}
where $r = \norm{\bx} = \sqrt{\sum_{i=1}^d x_i^2}$, $\{\mu_k\}_{k=1}^K$ are learnable exponents in $[\mu_{\min}, \mu_{\max}]$, $\psi_{\log}(r;\mu_{\log})$ is a log-primitive with a learnable log-exponent $\mu_{\log}$, and in the limit $\mu_{\log}\to 0$ it recovers $\log r$; $b_0$ is a bias term. Crucially, RMN allows negative exponents: $1/r = r^{-1}$ requires $\mu = -1$, which coordinate-wise approaches with positive exponents cannot achieve.

\subsection{Contributions}
In this study, we make the following contributions:

\paragraph{A Theoretical Obstruction Result} Theorem~\ref{thm:separability} proves that any $C^2$ function that is both radial and additively separable must be quadratic, explaining why coordinate-separable architectures systematically fail on radial singularities.

\paragraph{The RMN Architecture} We introduce a principled architecture parameterizing radial functions as linear combinations of learnable power bases $\{r^{\mu_k}\}_{k=1}^K$ with negative exponents enabled.

\paragraph{Logarithmic Singularities} We provide a mathematically correct treatment via a limit-stable primitive exploiting $\lim_{\mu \to 0} (r^\mu - 1)/\mu = \log r$, with closed-form gradients and Laplacians.

\paragraph{Angular and Multi-Center Extensions} We extend RMN to incorporate angular dependence through RMN-Angular using spherical harmonics and to handle multi-center singularities through RMN-MC with learnable source locations that recover centers to $<10^{-4}$ precision.

\paragraph{Interpretability} Learned exponent spectra recover physically meaningful singularity orders, enabling scientific insight unavailable from black-box approximators.

\subsection{Radial Singularities in Mathematical Physics}

Table~\ref{tab:singularities} catalogs the radial singularities that motivate this work. The prevalence of power-law and logarithmic forms reflects fundamental physics: the Laplacian is rotationally invariant, so its fundamental solutions are radial; point sources spread over spherical shells with area $\propto r^{d-1}$, leading to $1/r^{d-1}$ flux scaling; and asymptotic analysis near geometric singularities yields geometry-dependent power laws. Here $\omega_d = |S^{d-1}|$ denotes the surface area of the unit sphere in $\R^d$.

\begin{table}[t]
	\centering
	\resizebox{\textwidth}{!}{%
		\begin{tabular}{llll}
			\toprule
			Physical Problem                    & Dim        & Type             & Form                                                                           \\
			\midrule
			Laplace fundamental solution        & $d = 2$    & Logarithmic      & $G(\bx) = -\frac{1}{2\pi}\log\norm{\bx}$                                       \\
			Laplace fundamental solution        & $d = 3$    & Inverse          & $G(\bx) = \frac{1}{4\pi\norm{\bx}}$                                            \\
			Laplace fundamental solution        & $d \geq 3$ & Power law        & $G(\bx) = \frac{1}{(d-2)\omega_d}\norm{\bx}^{2-d}$                             \\
			Coulomb and gravitational potential & $d = 3$    & Inverse          & $\Phi(\bx) = \frac{q}{4\pi\varepsilon_0\norm{\bx}}$                            \\
			Electric field of a point charge    & $d = 3$    & Inverse square   & $\mathbf{E}(\bx) \propto \frac{\hat{\bx}}{\norm{\bx}^2}$                       \\
			Dipole potential                    & $d = 3$    & Inverse square   & $\Phi(\bx) = \frac{\mathbf{p} \cdot \hat{\bx}}{4\pi\varepsilon_0\norm{\bx}^2}$ \\
			Mode I crack-tip stress             & $d = 2$    & Inverse sqrt     & $\sigma_{ij} \propto \frac{K_I}{\sqrt{2\pi r}}f_{ij}(\theta)$                  \\
			Crack-tip displacement              & $d = 2$    & Square root      & $u_i \propto \sqrt{\frac{r}{2\pi}}g_i(\theta)$                                 \\
			Hydrogen atom ground state          & $d = 3$    & Exponential cusp & $\psi(\bx) \propto e^{-\norm{\bx}/a_0}$                                        \\
			Short-time heat kernel              & $d$        & Gaussian         & $K(\bx,t) \propto t^{-d/2}e^{-\norm{\bx}^2/4t}$                                \\
			Stokeslet in viscous flow           & $d = 3$    & Inverse          & $\mathbf{u} \propto \frac{1}{\norm{\bx}}(\mathbf{I} + \hat{\bx}\hat{\bx}^T)$   \\
			\bottomrule
		\end{tabular}}
	\caption{Radial singularities in mathematical physics. All depend on $r = \norm{\bx - \bx_0}$ with specific power-law or logarithmic scaling.}
	\label{tab:singularities}

\end{table}

\section{Related Work}
\label{sec:related_work}

RMN draws on and extends several research threads: physics-informed neural networks (PINNs), specialized architectures for implicit representations, classical approximation theory, and numerical methods for handling singularities. We position our contribution relative to each.

\subsection{PINNs}

PINNs \citep{raissi2019physics} embed partial differential equation (PDE) residuals directly into the loss function, enabling mesh-free solutions to forward and inverse problems. This framework has been extensively reviewed \citep{karniadakis2021physics, cuomo2022scientific} and implemented in libraries such as DeepXDE \citep{lu2021deepxde}. However, PINNs with standard MLP architectures struggle near singularities due to spectral bias \citep{rahaman2019spectral}: these networks preferentially learn low-frequency components, requiring many parameters and training iterations to accurately capture sharp gradients.

Several approaches have been developed to address PINN training difficulties. \citet{wang2021understanding} analyzed gradient pathologies and proposed adaptive weighting schemes. The Deep Ritz method \citep{weinan2018deep} reformulates PDEs variationally, while the Deep Galerkin Method \citep{sirignano2018dgm} uses random sampling for high-dimensional problems. Recent work by \citet{arzani2023theory} specifically addressed boundary layer singularities in fluid mechanics through tailored network architectures. Although these methods improve training performance, they do not address the fundamental architectural mismatch between smooth basis functions and singular targets.

\paragraph{RMN's Contribution} Rather than modifying training procedures, RMN modifies the architecture itself to match the target structure. The learnable power-law basis directly represents singularities that MLPs can only approximate inefficiently.

\subsection{Specialized Architectures for Coordinate Networks}

The spectral bias problem has motivated architectures with broader frequency support. SIREN \citep{sitzmann2020siren} uses periodic activations $\sin(\omega x)$ with careful initialization, achieving strong results in implicit neural representations. Fourier feature networks \citep{tancik2020fourier} map inputs through random Fourier features before a standard MLP, enabling high-frequency learning. Multiplicative filter networks \citep{fathony2021multiplicative} combine sinusoidal and Gabor filters for improved expressiveness.

Kolmogorov-Arnold Networks (KAN) \citep{liu2024kan, liu2024kan2} replace linear weights with learnable univariate functions on edges, inspired by the Kolmogorov-Arnold representation theorem. KAN demonstrates improved interpretability and efficiency in scientific problems, as the learned edge functions reveal the underlying structure.

Rational neural networks \citep{boulle2020rational, boulle2022data} represent another approach to learning singular behavior by parameterizing networks with rational functions. These networks can approximate functions with poles more efficiently than polynomial-based activations, providing a complementary strategy to RMN's power-law basis.

\paragraph{RMN's Contribution} While SIREN and Fourier features address frequency content, they do not capture power-law structure. KAN's learnable edge functions are related to our approach, but RMN specializes in radial power laws with learnable exponents, enabling the exact representation of fundamental solutions that coordinate-based methods cannot efficiently approximate, as shown in Theorem~\ref{thm:separability}. Unlike rational networks, which approximate singularities via poles, RMN directly parameterizes the singularity structure.

\subsection{Neural Operators}

Neural operators learn mappings between function spaces rather than pointwise evaluations. The Fourier Neural Operator (FNO) \citep{li2020fourier} parameterizes integral kernels in the Fourier space, achieving discretization-invariant learning. DeepONet \citep{lu2021learning} uses branch and trunk networks based on the universal approximation theorem for operators.

These methods excel at learning solution operators for families of PDEs but inherit the spectral properties of their underlying architectures. Near singularities, both FNO and DeepONet require fine discretization or many modes.

\paragraph{RMN's Contribution} RMN operates at the pointwise level with built-in singularity representation. For problems with a known radial structure, RMN provides a more parameter-efficient representation compared to discretizing singular kernels.

\subsection{Radial Basis Function Networks}

Radial basis function (RBF) networks \citep{powell1987radial, broomhead1988radial, park1991universal} represent functions as weighted sums of radially symmetric basis functions centered at data points:
\begin{equation}
	f(\bx) = \sum_{j=1}^N w_j \varphi(\norm{\bx - \bc_j}),
\end{equation}
where $\varphi$ is typically a Gaussian, multiquadric, or thin-plate spline. RBF networks have strong approximation properties \citep{wendland2004scattered} and have connections to Gaussian processes.

\paragraph{RMN's Contribution} Classical RBF networks use fixed functional forms $\varphi$, for example, $\varphi(r) = e^{-r^2}$ or $\varphi(r) = r^2 \log r$. RMN instead learns the functional form itself through learnable exponents. This enables adaptation to the specific singularity structure of the target, rather than relying on a predetermined basis.

\subsection{Numerical Methods for Singularities}

Classical numerical analysis has developed specialized techniques for singular problems. Extended finite element methods (XFEM) \citep{moes1999finite, belytschko1999elastic} enrich standard polynomial bases with singular functions matching crack-tip asymptotics \citep{williams1957stress}. Singular function enrichment \citep{fix1973singular} adds terms like $r^{\lambda} \sin(\lambda\theta)$ to capture corner singularities, with the mathematical theory developed extensively by \citet{grisvard2011elliptic}. Adaptive $hp$-refinement \citep{babuska1986hp, schwab1998p} combines mesh refinement ($h$) with polynomial degree increase ($p$) to achieve exponential convergence even for singular solutions.

The Method of Fundamental Solutions \citep{fairweather1998method} represents solutions as superpositions of fundamental solutions centered at source points located outside the domain. This classical meshless method is conceptually related to our RMN-MC, which learns both the source locations and the singularity structure.

Mesh-free methods \citep{belytschko1996meshless, liu2003mesh} completely avoid mesh generation, with some variants incorporating singular enrichments. Boundary element methods naturally handle fundamental solution singularities through careful quadrature \citep{sauter2010boundary}. For boundary layer problems in fluid mechanics, specialized techniques are required \citep{schlichting2017boundary}.

\paragraph{RMN's Contribution} These classical methods require the a priori specification of the singularity structure. RMN learns the appropriate exponents directly from data, automatically discovering singularity orders. This is particularly valuable for inverse problems where the singularity structure is unknown or for complex problems where multiple types interact.

\subsection{Equivariant and Geometric Neural Networks}

Equivariant neural networks \citep{cohen2016group, weiler2019general} build symmetry constraints directly into the network architecture. For problems with rotational symmetry, spherical convolutional neural networks \citep{cohen2018spherical} and tensor field networks \citep{thomas2018tensor} represent functions using spherical harmonics, thereby ensuring $SO(3)$ equivariance.

\paragraph{RMN's Contribution} RMN-Angular, described in Section~\ref{subsec:rmn_angular}, incorporates spherical harmonics for angular dependence, but the key innovation is learnable radial exponents rather than fixed polynomial or smooth radial functions. This combination enables the representation of multipole-like fields with fractional orders.

\subsection{Approximation Theory}

Universal approximation theorems for neural networks \citep{cybenko1989approximation, hornik1989multilayer} guarantee that sufficiently wide networks can approximate any continuous function. However, these results do not address the efficiency in terms of the number of parameters required for a given accuracy.

Barron's theorem \citep{barron1993universal} provides efficiency guarantees for functions with bounded Fourier moments, but singular functions violate these regularity assumptions. The Müntz-Szász theorem \citep{muntz1914, szasz1916} characterizes which exponent sequences yield dense spans in $C[0,1]$, providing the classical foundation for our approach. Modern treatments appear in \citet{borwein1995polynomials}.

\paragraph{RMN's Contribution} We extend Müntz-Szász theory to neural network architectures with learnable exponents, negative powers for singularities, and multidimensional radial structure. We establish density in radial $L^2$ spaces under appropriate exponent conditions.

\section{Foundations: Preliminaries and Separability Obstruction}
\label{sec:preliminaries}

This section establishes the mathematical framework for analyzing radial singularities and their approximation. We work primarily in $\R^d$ with the Euclidean norm $\norm{\bx} = (\sum_{i=1}^d x_i^2)^{1/2}$, review the calculus and integrability of radial powers, and then prove the obstruction result used to justify RMN.

\subsection{Radial Functions in \texorpdfstring{$\R^d$}{Rd}}
\label{subsec:radial_functions}

\begin{definition}[Radial function]
	A function $f: \R^d \to \R$ is called radial, also called spherically symmetric, if there exists a function $h: [0, \infty) \to \R$ such that $f(\bx) = h(\norm{\bx})$ for all $\bx \in \R^d$. We write $f(\bx) = h(r)$ where $r = \norm{\bx}$.
\end{definition}

Radial functions are invariant under orthogonal transformations: $f(O\bx) = f(\bx)$ for all $O \in O(d)$. This invariance serves as both a geometric constraint and a source of efficiency because radial functions on $\R^d$ are effectively one-dimensional.

\paragraph{Calculus of Radial Functions} For a differentiable radial function $f(\bx) = h(r)$:
\begin{align}
	\nabla f(\bx) & = h'(r) \frac{\bx}{r}, \label{eq:radial_gradient}           \\
	\Delta f(\bx) & = h''(r) + \frac{d-1}{r} h'(r). \label{eq:radial_laplacian}
\end{align}

The Laplacian formula shows that the $(d-1)/r$ term couples the radial derivative to the dimension. For $h(r) = r^\mu$:
\begin{equation}
	\Delta r^\mu = \mu(\mu - 1) r^{\mu - 2} + (d-1) \mu r^{\mu - 2} = \mu(\mu + d - 2) r^{\mu - 2}.
	\label{eq:laplacian_power}
\end{equation}

\begin{corollary}
	The radial power $r^\mu$ is harmonic $\Delta r^\mu = 0$ in $\R^d \setminus \{0\}$ if and only if $\mu = 0$ or $\mu = 2 - d$.
\end{corollary}

Thus $r^{-1}$ is harmonic in $d = 3$ and corresponds to the Coulomb and Newtonian potentials, $r^0 = 1$ is trivially harmonic everywhere, and $\log r$ is harmonic in $d = 2$ as the limit $\mu \to 0$ of $(r^\mu - 1)/\mu$.

\subsection{Integrability and Singular Orders}
\label{subsec:integrability}

Not all radial singularities are equally severe. The integrability of $r^\alpha$ near the origin determines whether the singularity is mild or strong, that is, integrable or non-integrable.

\begin{proposition}[$L^p$ integrability of radial powers]
	\label{prop:integrability}
	Let $B_1 = \{\bx \in \R^d : \norm{\bx} \leq 1\}$ be the unit ball. Then $r^\alpha \in L^p(B_1)$ if and only if $\alpha > -d/p$.
\end{proposition}

\begin{proof}
	Switching to polar coordinates,
	\[
		\int_{B_1} |r^\alpha|^p d\bx = \omega_d \int_0^1 r^{\alpha p} r^{d-1} dr = \omega_d \int_0^1 r^{\alpha p + d - 1} dr.
	\]
	This integral converges if and only if $\alpha p + d - 1 > -1$, that is, $\alpha > -d/p$.
\end{proof}

For $p = 2$, corresponding to the mean squared error (MSE) loss, this criterion yields dimension-dependent thresholds. In $d = 2$, $r^\alpha \in L^2(B_1)$ if and only if $\alpha > -1$, so $\log r \in L^2(B_1)$ but $r^{-1} \notin L^2(B_1)$. In $d = 3$, $r^\alpha \in L^2(B_1)$ if and only if $\alpha > -3/2$, so $r^{-1} \in L^2(B_1)$ and the Coulomb potential is square-integrable in 3D.

This distinction matters for training: when the $L^2$ norm diverges, minimizing the unweighted MSE on the full domain is ill-posed. In our benchmarks, we therefore fit and evaluate on punctured domains $\{\bx : \norm{\bx} \geq r_{\min}\}$ with $r_{\min} > 0$, which ensures that the loss is well-defined.

\begin{remark}[Punctured Domains and Evaluation Protocol]
	\label{rem:punctured_domains}
	Let $\varepsilon > 0$ and define the punctured domain $\Omega_\varepsilon = \{\bx : \norm{\bx} \geq \varepsilon\}$. Such punctured domains are not merely a numerical convenience but a mathematical necessity: for singularities with $\alpha \leq -d/2$, the $L^2$ norm $\|f\|_{L^2(B_1)}$ diverges, so standard MSE training is ill-defined on the full domain. For example, fitting $r^{-1}$ in 2D with $L^2$ loss on $B_1$ asks for a finite approximation to an infinite quantity. In our experiments, we take $\varepsilon = r_{\min} = 0.01$. We sample all training points uniformly from $\Omega_\varepsilon$, evaluate on a separate test grid within the same punctured domain, and report the root mean squared error (RMSE) computed on the punctured test domain. This is standard practice in singular function approximation and matches evaluation protocols in XFEM \citep{moes1999finite} and related literature. Physically, this is also realistic: real systems have finite extent, with no true point charges and finite crack-tip radii.
\end{remark}

\subsection{The Müntz-Szász Theorem}
\label{subsec:muntz_szasz}

The classical Müntz-Szász theorem characterizes when power functions form a complete basis for continuous functions.

\begin{theorem}[Müntz-Szász \citep{muntz1914, szasz1916}]
	\label{thm:muntz_szasz}
	Let $0 < \mu_1 < \mu_2 < \cdots$ be a strictly increasing sequence of positive real numbers. The linear span of $\{1, x^{\mu_1}, x^{\mu_2}, \ldots\}$ is dense in $C[0,1]$ equipped with the supremum norm, if and only if
	\begin{equation}
		\sum_{k=1}^\infty \frac{1}{\mu_k} = \infty.
		\label{eq:muntz_condition}
	\end{equation}
\end{theorem}

This remarkable theorem shows that even sparse sequences of exponents can yield universal approximation, provided the exponents do not grow too quickly.

\begin{example}[Satisfying and violating the Müntz condition]
	\label{ex:muntz_examples}
	Consider several representative exponent sets. For standard polynomials, $\Lambda = \{0, 1, 2, 3, \ldots\}$, the harmonic series $\sum_k 1/k$ diverges, so \eqref{eq:muntz_condition} holds, and density is consistent with the Weierstrass approximation theorem. Likewise, for half-integer exponents $\Lambda = \{0, 0.5, 1, 1.5, 2, \ldots\}$, one has $\sum_k 2/(k+1) = \infty$, and the span is dense in $C[0,1]$. In contrast, for squared exponents $\Lambda = \{0, 1, 4, 9, 16, \ldots\} = \{k^2\}$, the series $\sum_k 1/k^2 = \pi^2/6$ converges, so this set is not dense and cannot approximate functions such as $x^{1/2}$. A similar failure occurs for $\Lambda = \{0, 2, 4, 8, 16, \ldots\} = \{2^k\}$, since $\sum_k 2^{-k} = 2 < \infty$.
\end{example}

\paragraph{Limitations of the Müntz-Szász Theorem for Singular Approximation} The Müntz-Szász theorem, although foundational, has several limitations in our purposes. First, it concerns positive exponents on $[0,1]$, whereas singular functions require negative exponents. Second, the theorem guarantees density but does not ensure efficiency. Third, the logarithm $\log x$ cannot be represented as a finite linear combination of powers. Finally, the theorem is one-dimensional.

We now establish a density result for radial Müntz-type expansions that addresses these limitations, extending the classical theorem to multidimensional radial $L^2$ spaces with negative exponents.

\begin{definition}[Radial Müntz space]
	\label{def:radial_muntz_space}
	Let $\Lambda = \{\mu_k\}_{k=1}^\infty \subset \R$ be a sequence of distinct real exponents. Define the radial Müntz space on the annulus $A_{a,b} = \{\bx \in \R^d : a \leq \norm{\bx} \leq b\}$ with $0 < a < b$ as:
	\begin{equation}
		\mathcal{M}_\Lambda^{\mathrm{rad}}(A_{a,b}) = \overline{\operatorname{span}\{r^{\mu_k} : k \in \N\}}^{L^2_{\mathrm{rad}}(A_{a,b})},
	\end{equation}
	where $L^2_{\mathrm{rad}}(A_{a,b}) = \{f \in L^2(A_{a,b}) : f(\bx) = h(\norm{\bx}) \text{ a.e.}\}$ denotes the closed subspace of radial functions.
\end{definition}

\begin{theorem}[Radial Müntz-Szász Density]
	\label{thm:radial_density}
	Let $0 < a < b < \infty$, and let $\Lambda^+ = \{\mu_k^+ : k \in \N\}$ and $\Lambda^- = \{\mu_k^- : k \in \N\}$ be sequences of positive and negative exponents, respectively, with $\mu_k^+ \to +\infty$ and $\mu_k^- \to -\infty$. Define $\Lambda = \Lambda^+ \cup \Lambda^- \cup \{0\}$. If
	\begin{equation}
		\sum_{k=1}^\infty \frac{1}{\mu_k^+} = \infty \quad \text{and} \quad \sum_{k=1}^\infty \frac{1}{|\mu_k^-|} = \infty,
		\label{eq:bilateral_muntz}
	\end{equation}
	then $\mathcal{M}_\Lambda^{\mathrm{rad}}(A_{a,b}) = L^2_{\mathrm{rad}}(A_{a,b})$.
\end{theorem}

\begin{proof}
	We reduce the multidimensional radial problem to the classical Müntz-Szász theorem through successive changes of variables.

	First, for a radial function $f(\bx) = h(r)$ with $r = \norm{\bx}$, the $L^2$ norm on the annulus $A_{a,b}$ satisfies
	\begin{equation}
		\|f\|_{L^2(A_{a,b})}^2 = \omega_d \int_a^b |h(r)|^2 r^{d-1} \, dr,
	\end{equation}
	where $\omega_d = 2\pi^{d/2}/\Gamma(d/2)$ is the surface area of $S^{d-1}$. Thus $L^2_{\mathrm{rad}}(A_{a,b})$ is isometrically isomorphic to the weighted space $L^2([a,b], r^{d-1} dr)$.

	Next, the substitution $r = a + (b-a)t$ maps $[a,b]$ to $[0,1]$, and the power $r^\mu$ becomes $(a + (b-a)t)^\mu$. Density of the powers $\{r^\mu\}$ in $L^2([a,b], r^{d-1} dr)$ is equivalent to the density of $\{(a + (b-a)t)^\mu\}$ in $L^2([0,1], w(t) dt)$ for a strictly positive continuous weight $w(t) = (a + (b-a)t)^{d-1}(b-a)$.

	We then observe that on $[0,1]$ with any continuous positive weight, density questions reduce to the unweighted case by standard arguments \citep[Chapter II]{borwein1995polynomials}. The binomial expansion of $(a + (b-a)t)^\mu$ shows that each such function lies in the closed span of $\{t^n : n \geq 0\} \cup \{t^\mu : \mu \in \Lambda\}$ over $[0,1]$.

	More directly, we apply the substitution $t = x^\gamma$ for an appropriate $\gamma > 0$ to reduce to the classical setting. The key observation is that, on any interval $[a,b]$ with $0 < a < b$, the system $\{r^\mu : \mu \in \Lambda\}$ is complete if and only if the bilateral Müntz condition \eqref{eq:bilateral_muntz} holds. This follows from the classical Müntz-Szász theorem applied separately to the positive and negative exponent subsequences, using the fact that on $[a,b]$ with $a > 0$, negative powers are bounded continuous functions \citep{borwein1995polynomials,almira2007muntz}.

	Finally, by hypothesis \eqref{eq:bilateral_muntz}, the exponent sequence $\Lambda$ satisfies the bilateral Müntz condition. Hence $\{r^{\mu_k}\}$ is complete in $L^2([a,b], r^{d-1} dr)$, and therefore, in $L^2_{\mathrm{rad}}(A_{a,b})$.
\end{proof}

\begin{remark}[Relation to the classical theorem]
	\label{rem:classical_relation}
	When restricted to positive exponents and $d = 1$, Theorem~\ref{thm:radial_density} recovers the classical Müntz-Szász theorem on $[a,b] \subset (0,\infty)$. The extension to negative exponents is essential for singular approximation: the condition $\sum 1/|\mu_k^-| = \infty$ ensures that arbitrarily strong singularities can be approximated. The annular domain $a > 0$ avoids integrability issues at the origin.
\end{remark}

\begin{remark}[Role of negative exponents]
	\label{rem:negative_exponents}
	The bilateral condition \eqref{eq:bilateral_muntz} requires that both positive and negative exponents diverge slowly enough. This is essential for approximating singular functions: positive exponents alone cannot efficiently represent $r^{-1}$ near zero, while negative exponents alone cannot capture smooth growth at large $r$. The RMN architecture learns exponents from both regimes, automatically discovering the appropriate balance for the target function.
\end{remark}

\begin{remark}[Logarithmic terms]
	\label{rem:log_extension}
	Theorem~\ref{thm:radial_density} does not directly include $\log r$. However, for any $\mu \neq 0$, the identity $(r^\mu - 1)/\mu \to \log r$ as $\mu \to 0$ shows that $\log r$ lies in the $L^2$ closure of Müntz spans with exponents accumulating at zero. In practice, we include a dedicated log-primitive $\psi_{\log}(r;\mu_{\log})$ for numerical stability, rather than relying solely on this limiting argument.
\end{remark}

\begin{corollary}[Density for RMN hypothesis class]
	\label{cor:rmn_density}
	Fix $0<a<b$. Suppose that, in the large-$K$ limit, the learnable exponents can realize a sequence $\Lambda=\{\mu_k\}_{k=1}^\infty$ whose positive and negative subsequences satisfy the hypotheses of Theorem~\ref{thm:radial_density}, in particular, $\mu_k^+\to+\infty$, $\mu_k^-\to-\infty$, and \eqref{eq:bilateral_muntz} holds. This requires that the allowable exponent range is not fixed a priori and permits arbitrarily large $|\mu|$. Then the RMN-Direct hypothesis class is dense in $L^2_{\mathrm{rad}}(A_{a,b})$.
\end{corollary}

\begin{proof}
	Fix $0<a<b$. By assumption, in the large-$K$ limit, the learnable exponent set of the RMN can realize a sequence $\Lambda=\{\mu_k\}_{k=1}^\infty$ satisfying the hypotheses of Theorem~\ref{thm:radial_density}, in particular the bilateral Müntz condition \eqref{eq:bilateral_muntz}.
	For each $K$, consider the RMN-Direct model with $c_0=0$ and exponents $\mu_1,\ldots,\mu_K$. As the coefficients $a_1,\ldots,a_K$ and the bias $b_0$ vary freely, the resulting hypothesis class contains $\operatorname{span}\{1,r^{\mu_1},\ldots,r^{\mu_K}\}$.
	Therefore, the $L^2_{\mathrm{rad}}(A_{a,b})$-closure of the union over $K$ contains
	\[
		\overline{\operatorname{span}\{1,r^{\mu_k}:k\in\N\}}^{L^2_{\mathrm{rad}}(A_{a,b})}
		=\mathcal{M}_{\Lambda \cup \{0\}}^{\mathrm{rad}}(A_{a,b}),
	\]
	which equals $L^2_{\mathrm{rad}}(A_{a,b})$ by Theorem~\ref{thm:radial_density}.
\end{proof}

\subsection{Spherical Harmonics}
\label{subsec:spherical_harmonics}

For functions with angular dependence, we employ spherical harmonics. In $\R^d$, the spherical harmonics $\{Y_\ell^m\}_{\ell \geq 0, |m| \leq \ell}$ form an orthonormal basis for $L^2(S^{d-1})$.

Any function defined on $\R^d \setminus \{0\}$ can be expanded as
\begin{equation}
	f(\bx) = \sum_{\ell=0}^\infty \sum_{m=-\ell}^\ell f_{\ell m}(r) Y_\ell^m(\hat{\bx}),
	\label{eq:spherical_expansion}
\end{equation}
where $\hat{\bx} = \bx/\norm{\bx}$ is the unit direction and $f_{\ell m}(r)$ are the radial coefficient functions.

\paragraph{2D Specialization and Mode Counting} In $d=2$, $S^{d-1} = S^1$ and the angular basis reduces to Fourier modes. Writing $\bx = (r\cos\theta, r\sin\theta)$, we expand
\begin{equation*}
	f(r,\theta) = \sum_{m\in\mathbb{Z}} f_m(r) e^{im\theta} \approx \sum_{m=-M_{\max}}^{M_{\max}} f_m(r) e^{im\theta}.
\end{equation*}
This symmetric truncation contains $2M_{\max}+1$ angular modes. If $f$ is real-valued, then $f_{-m}(r)=f_m(r)^\ast$, or equivalently, a cosine--sine representation with $2M_{\max}+1$ real angular components.

\subsection{Separability Obstruction}
\label{sec:obstruction}

This subsection establishes the core theoretical result: coordinate-separable architectures are fundamentally incompatible with non-quadratic radial functions.

\begin{definition}[Additive separability]
	A function $f: \R^d \to \R$ is additively separable if there exist functions $g_i: \R \to \R$ such that
	\begin{equation}
		f(\bx) = \sum_{i=1}^d g_i(x_i).
		\label{eq:additive_separable}
	\end{equation}
\end{definition}

\begin{definition}[Multiplicative separability]
	A function $f: \R^d \to \R$ is multiplicatively separable if there exist functions $g_i: \R \to \R$ such that
	\begin{equation}
		f(\bx) = \prod_{i=1}^d g_i(x_i).
		\label{eq:multiplicative_separable}
	\end{equation}
\end{definition}

Many neural architectures produce separable representations, such as coordinate-wise MSN, generalized additive models, tensor product bases, and the first layer of MLPs.

\begin{theorem}[Separability Obstruction]
	\label{thm:separability}
	Let $f: \R^2 \setminus \{0\} \to \R$ be $C^2$. Suppose that $f$ is radial, that is, $f(x, y) = h(r)$ for some $h: (0, \infty) \to \R$ with $r = \sqrt{x^2 + y^2}$, and also additively separable, that is, $f(x, y) = g(x) + k(y)$ for some $g, k: \R \to \R$.
	Then $h(r) = cr^2 + b$ for some constants $c, b \in \R$.
\end{theorem}

\begin{proof}
	From radiality, we compute the mixed partial derivative using the chain rule:
	\begin{align}
		\frac{\partial f}{\partial x}              & = h'(r) \frac{\partial r}{\partial x} = h'(r) \frac{x}{r},     \\
		\frac{\partial^2 f}{\partial x \partial y} & = \frac{\partial}{\partial y}\left( h'(r) \frac{x}{r} \right).
	\end{align}

	Applying the product rule, we obtain
	\begin{align}
		\frac{\partial^2 f}{\partial x \partial y} & = h''(r) \frac{\partial r}{\partial y} \cdot \frac{x}{r} + h'(r) \frac{\partial}{\partial y}\left(\frac{x}{r}\right) \\
		                                           & = h''(r) \frac{y}{r} \cdot \frac{x}{r} + h'(r) \cdot x \cdot \left(-\frac{1}{r^2}\right) \frac{y}{r}                 \\
		                                           & = \frac{xy}{r^2} h''(r) - \frac{xy}{r^3} h'(r)                                                                       \\
		                                           & = \frac{xy}{r^2} \left( h''(r) - \frac{h'(r)}{r} \right).
	\end{align}

	From additive separability, $f(x, y) = g(x) + k(y)$ implies
	\[
		\frac{\partial^2 f}{\partial x \partial y} = \frac{\partial}{\partial y} g'(x) = 0.
	\]

	Combining these results: for all $(x, y)$ with $xy \neq 0$,
	\begin{equation}
		\frac{xy}{r^2} \left( h''(r) - \frac{h'(r)}{r} \right) = 0.
		\label{eq:obstruction_ode}
	\end{equation}

	Because $xy \neq 0$ and $r > 0$, we conclude $h''(r) - h'(r)/r = 0$ for all $r$ such that there exists $(x,y)$ with $x^2 + y^2 = r^2$ and $xy \neq 0$.

	Now we extend to all $r > 0$. For any $r > 0$, we can choose $x = y = r/\sqrt{2}$, which satisfies $x^2 + y^2 = r^2$ and $xy = r^2/2 \neq 0$. Thus the ordinary differential equation (ODE) $h''(r) = h'(r)/r$ holds for all of $(0, \infty)$. Alternatively, since $h \in C^2$, both $h''(r)$ and $h'(r)/r$ are continuous on $(0, \infty)$. Two continuous functions that agree on a dense subset must agree everywhere on their common domain.

	This is a second-order linear ODE\@. Let $\psi(r) = h'(r)$. Then $\psi'(r) = \psi(r)/r$, which gives:
	\begin{align*}
		\frac{d\psi}{\psi} & = \frac{dr}{r}, \\
		\log|\psi(r)|      & = \log r + C_1, \\
		\psi(r)            & = Ar,
	\end{align*}
	for some constant $A$. Integrating, $h(r) = \frac{A}{2} r^2 + B = cr^2 + b$.
\end{proof}

\begin{remark}
	The quadratic $r^2 = x^2 + y^2$ is the only nontrivial radial function that is additively separable. This is because $x^2 + y^2 = g(x) + k(y)$ with $g(x) = x^2$ and $k(y) = y^2$.
\end{remark}

This result formalizes the empirical observation that coordinate-wise networks systematically introduce axis-aligned artifacts when approximating radial structure. Figure~\ref{fig:obstruction_visual} illustrates a representative example.

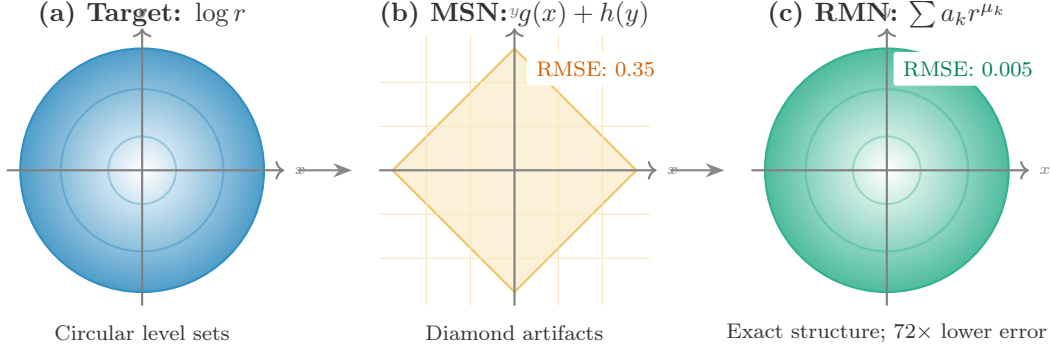
\begin{figure}[!htbp]
	\centering
	\begin{tikzpicture}[scale=0.9]
		\begin{scope}[xshift=0cm]
			\node[font=\bfseries\small, color=DarkGray] at (2.2, 4.5) {(a) Target: $\log r$};
			\shade[inner color=white, outer color=NatureBlue!70] (2.2,2.2) circle (1.8);
			\draw[NatureBlue!80, thick] (2.2,2.2) circle (1.8);
			\draw[NatureBlue!60, thick] (2.2,2.2) circle (1.2);
			\draw[NatureBlue!40, thick] (2.2,2.2) circle (0.5);
			\fill[white] (2.2,2.2) circle (0.08);
			\draw[->, thick, gray] (0.2,2.2) -- (4.3,2.2) node[right, font=\tiny] {$x$};
			\draw[->, thick, gray] (2.2,0.2) -- (2.2,4.3) node[above, font=\tiny] {$y$};
			\node[font=\scriptsize, color=DarkGray] at (2.2, -0.2) {Circular level sets};
		\end{scope}

		\begin{scope}[xshift=5.5cm]
			\node[font=\bfseries\small, color=DarkGray] at (2.2, 4.5) {(b) MSN: $g(x)+h(y)$};
			\fill[NatureOrange!15] (0.4,2.2) -- (2.2,4.0) -- (4.0,2.2) -- (2.2,0.4) -- cycle;
			\draw[NatureOrange!60, thick] (0.4,2.2) -- (2.2,4.0) -- (4.0,2.2) -- (2.2,0.4) -- cycle;
			\foreach \x in {0.9, 1.55, 2.2, 2.85, 3.5} {
					\draw[NatureOrange!25] (\x, 0.2) -- (\x, 4.2);
				}
			\foreach \y in {0.9, 1.55, 2.2, 2.85, 3.5} {
					\draw[NatureOrange!25] (0.2, \y) -- (4.2, \y);
				}
			\draw[->, thick, gray] (0.2,2.2) -- (4.3,2.2) node[right, font=\tiny] {$x$};
			\draw[->, thick, gray] (2.2,0.2) -- (2.2,4.3) node[above, font=\tiny] {$y$};
			\node[font=\scriptsize, color=NatureRed, fill=white, rounded corners=1pt] at (3.4, 3.7) {RMSE: 0.35};
			\node[font=\scriptsize, color=DarkGray] at (2.2, -0.2) {Diamond artifacts};
		\end{scope}

		\begin{scope}[xshift=11cm]
			\node[font=\bfseries\small, color=DarkGray] at (2.2, 4.5) {(c) RMN: $\sum a_k r^{\mu_k}$};
			\shade[inner color=white, outer color=NatureGreen!70] (2.2,2.2) circle (1.8);
			\draw[NatureGreen!80, thick] (2.2,2.2) circle (1.8);
			\draw[NatureGreen!60, thick] (2.2,2.2) circle (1.2);
			\draw[NatureGreen!40, thick] (2.2,2.2) circle (0.5);
			\fill[white] (2.2,2.2) circle (0.08);
			\draw[->, thick, gray] (0.2,2.2) -- (4.3,2.2) node[right, font=\tiny] {$x$};
			\draw[->, thick, gray] (2.2,0.2) -- (2.2,4.3) node[above, font=\tiny] {$y$};
			\node[font=\scriptsize, color=NatureGreen!80!black, fill=white, rounded corners=1pt] at (3.4, 3.7) {RMSE: 0.005};
			\node[font=\scriptsize, color=DarkGray] at (2.2, -0.2) {Exact structure; 72$\times$ lower error};
		\end{scope}

		\draw[-{Stealth[length=2.5mm]}, thick, MediumGray] (4.5, 2.2) -- (5.3, 2.2);
		\draw[-{Stealth[length=2.5mm]}, thick, MediumGray] (10.0, 2.2) -- (10.8, 2.2);
	\end{tikzpicture}
	\caption{The separability obstruction visualized. (a) The target $\log r$ has circular level sets. (b) The coordinate-separable MSN produces diamond-shaped artifacts due to its axis-aligned structure. (c) RMN precisely matches the radial structure, achieving a 72$\times$ lower error.}
	\label{fig:obstruction_visual}
\end{figure}

We now discuss the extension to multiplicative separability and higher dimensions.

\begin{proposition}
	\label{prop:multiplicative}
	Let $f: \R^2 \setminus \{0\} \to (0, \infty)$ be $C^2$, radial, and multiplicatively separable. Then there exist constants $a, b \in \R$ such that
	\[
		f(x, y) = \exp(a(x^2 + y^2) + b).
	\]
\end{proposition}

\begin{proof}
	Since $f > 0$, we can take logarithms: $\log f(x, y) = \log g(x) + \log h(y)$. Thus $\log f$ is additively separable. Because $f$ is radial, $\log f$ is also radial. By Theorem~\ref{thm:separability}, $\log f(x, y) = a(x^2 + y^2) + b$ for some constants $a, b$. Hence $f(x, y) = e^b \cdot e^{ar^2} = \exp(ar^2 + b)$.
\end{proof}

\begin{remark}
	Gaussian-type functions $\exp(ar^2)$ are smooth at the origin and do not exhibit singular behavior. Thus multiplicative separability cannot represent radial singularities such as $r^{-1}$, $r^{1/2}$, or $\log r$. This complements Theorem~\ref{thm:separability}: neither additive nor multiplicative separability can capture non-quadratic radial structure.
\end{remark}

\begin{proposition}
	\label{prop:higher_dim}
	Let $f: \R^d \setminus \{0\} \to \R$ be $C^2$, radial, and additively separable: $f(\bx) = \sum_{i=1}^d g_i(x_i)$. Then $f(\bx) = c\norm{\bx}^2 + b$ for constants $c, b$.
\end{proposition}

\begin{proof}
	For any pair of indices $i \neq j$, consider the restriction of $f$ to the $(x_i, x_j)$-plane. This restriction is both radial in the 2D sense and additively separable. By Theorem~\ref{thm:separability}, $f(x_i, x_j) = c_{ij}(x_i^2 + x_j^2) + b_{ij}$. Consistency across all pairs requires $c_{ij} = c$ and $b_{ij} = b$ for all $i, j$, giving $f(\bx) = c\norm{\bx}^2 + b$.
\end{proof}

The separability obstruction has immediate consequences for neural network architectures:
\paragraph{Coordinate-Wise MSN Cannot Efficiently Approximate Radial Singularities} Because MSN-coord produces additively separable functions, Theorem~\ref{thm:separability} implies that it can only exactly represent quadratic radial functions. The failure is structural, not capacity-limited. Adding more parameters in MSN-coord does not resolve the obstruction. The obstruction quantifies the capacity gap. To approximate $\log r$ using coordinate-wise bases, the required number of terms grows as $\varepsilon^{-\gamma}$ for a given tolerance $\varepsilon$. RMN with its log-primitive achieves a numerically exact representation through a stable limiting mechanism.

\begin{table}[t]
	\centering
		\begin{tabular}{lcccc}
			\toprule
			Target        & RMN                   & MSN coord             & Improvement              & Params for RMN and MSN \\
			\midrule
			$\log r$ (2D) & $4.85 \times 10^{-3}$ & $3.47 \times 10^{-1}$ & $\mathbf{72\times}$      & 27 / 49                \\
			$r^{-1}$ (2D) & $7.31 \times 10^{-3}$ & $7.34 \times 10^{0}$  & $\mathbf{1{,}004\times}$ & 27 / 49                \\
			$r^{-1}$ (3D) & $4.61 \times 10^{-3}$ & $7.62 \times 10^{0}$  & $\mathbf{1{,}652\times}$ & 27 / 73                \\
			\bottomrule
		\end{tabular}
	\caption{Validation of separability obstruction. RMN improves over coordinate-separable MSN by 72$\times$ to 1,652$\times$ on non-quadratic radial targets. All results represent the mean RMSE over 5 random seeds.}
	\label{tab:obstruction_validation}

\end{table}

\paragraph{Note on MLP Comparison} While the MSN comparison validates the separability obstruction theorem, practitioners care more about MLP baselines. On these same benchmarks, RMN outperforms a 33,537-parameter MLP by 1.5$\times$ on $\log r$, 24$\times$ on 2D $r^{-1}$, and 51$\times$ on 3D Coulomb. The full comparison appears in Table~\ref{tab:main_results_full}.

\section{RMN: Architecture and Variants}
\label{sec:architecture}

We first define RMN-Direct and its exponent parameterization, then introduce the log-primitive and numerical stabilization, and finally present two extensions: RMN-Angular for angular dependence and RMN-MC for multiple singularity centers.

\subsection{RMN-Direct: Architecture and Exponent Parameterization}
\label{subsec:rmn_definition}

\begin{definition}[RMN-Direct]
	For an input $\bx \in \R^d$, the RMN-Direct network computes:
	\begin{equation}
		\boxed{\phi_{\text{RMN}}(\bx) = \sum_{k=1}^{K} a_k r^{\mu_k} + c_0 \psi_{\log}(r;\mu_{\log}) + b_0},
		\label{eq:rmn_direct}
	\end{equation}
	where $r = \norm{\bx}$ is the radial coordinate, $\{a_k\}_{k=1}^K$ are learnable coefficients, $\{\mu_k\}_{k=1}^K \subset [\mu_{\min}, \mu_{\max}]$ are learnable exponents including negative values, $c_0$ is the coefficient for the log-primitive, $\mu_{\log}$ is a learnable log-exponent, $\psi_{\log}(r;\mu_{\log})$ is the log-primitive (Definition~\ref{def:log_primitive}), and $b_0$ is a bias term.
\end{definition}

\paragraph{Parameter Count} RMN-Direct has $K$ coefficients $+ K$ exponents $+ 1$ log coefficient $+ 1$ log exponent $+ 1$ bias $= 2K + 3$ parameters. For the default $K = 12$, this gives 27 parameters, compared to 33,537 for a 4-layer MLP with 128 neurons per layer.

\paragraph{Exact Representability} We record two basic exact-representability propositions in Appendix~\ref{app:moved_statements}. These facts show that RMN has exact representation capacity for finite power sums, unlike MLPs, which must approximate.

\paragraph{Exponent Parameterization} We parameterize exponents using a cumulative-gap scheme that ensures (1) exponents remain in $[\mu_{\min}, \mu_{\max}]$, (2) exponents are ordered, and (3) gradients flow smoothly. Given raw parameters $\{s_k\}_{k=1}^K$, we compute:
\begin{align}
	\delta_k & = \softplus(s_k) + \epsilon \quad (\epsilon = 0.01), \\
	\sigma_k & = \sum_{j=1}^k \delta_j,                             \\
	u_k      & = \sigma_k / \sigma_K,                               \\
	\mu_k    & = \mu_{\min} + (\mu_{\max} - \mu_{\min}) u_k.
\end{align}
The default exponent range is $[\mu_{\min}, \mu_{\max}] = [-2, 4]$, which covers singularities from $r^{-2}$ to regular functions like $r^4$.

\subsection{Logarithmic Singularities: The Log-Primitive}
\label{subsec:log_treatment}

\begin{proposition}
	\label{prop:log_not_power}
	The function $\log r$ cannot be represented as a finite sum $\sum_k a_k r^{\mu_k}$ for any choice of $\mu_k \in \R$.
\end{proposition}

\begin{proof}
	Suppose $\log r = \sum_{k=1}^K a_k r^{\mu_k}$. Differentiating: $1/r = \sum_k a_k \mu_k r^{\mu_k - 1}$. The left side is $r^{-1}$; the right side is a finite sum of terms $r^{\mu_k - 1}$. For equality, we need exactly one $\mu_j = 0$ with $a_j \mu_j = 1$, but $\mu_j = 0$ gives $a_j \cdot 0 = 0 \neq 1$. Contradiction.
\end{proof}

We include an explicit log-primitive based on the identity:
\begin{equation}
	\lim_{\mu \to 0} \frac{r^\mu - 1}{\mu} = \log r.
	\label{eq:log_primitive_limit}
\end{equation}

\begin{definition}[Log-primitive]
	\label{def:log_primitive}
	For $r>0$, define
	\begin{equation}
		\psi_{\log}(r; \mu) = \begin{cases}
			\displaystyle \frac{r^\mu - 1}{\mu} & \text{if } \mu \neq 0, \\[8pt]
			\displaystyle \log r                & \text{if } \mu = 0.
		\end{cases}
	\end{equation}
	For numerical stability when $|\mu|$ is small, we evaluate this expression using its Taylor expansion and switch to that series below the threshold $\epsilon_{\log}=10^{-4}$.
\end{definition}

\begin{proposition}[Pointwise convergence]
	For all $r > 0$, $\psi_{\log}(r; \mu) \to \log r$ as $\mu \to 0$.
\end{proposition}

\begin{proof}
	Write $r^\mu = e^{\mu \log r} = 1 + \mu \log r + \frac{\mu^2 (\log r)^2}{2} + O(\mu^3)$. Then
	\[
		\frac{r^\mu - 1}{\mu} = \log r + \frac{\mu (\log r)^2}{2} + O(\mu^2) \to \log r \text{ as } \mu \to 0.
	\]
\end{proof}

\paragraph{Numerical Stability} Direct computation of $(r^\mu - 1)/\mu$ near $\mu = 0$ suffers from catastrophic cancellation. The Taylor expansion provides a numerically stable alternative. We use a smooth sigmoid-based blend between the two formulas to ensure differentiability.

\paragraph{Learning Logarithmic Behavior} In practice, we include a separate learnable exponent $\mu_{\log}$ for the log-primitive, initialized near zero. The network can learn to either (a) keep $\mu_{\log} \approx 0$ to represent $\log r$, or (b) move $\mu_{\log}$ away from zero to represent a power law $(r^{\mu_{\log}} - 1)/\mu_{\log}$.

\subsection{Optimization and Numerical Stability}
\label{subsec:optimization}

For small $r$ and negative $\mu$, direct computation of $r^\mu$ can overflow. We therefore compute $r^\mu$ in the stable form $r^\mu = \exp(\mu \log r)$ and floor $r$ at $r_{\mathrm{floor}} = 10^{-12}$. The exponent gradient satisfies $\partial_{\mu_k} r^{\mu_k} = r^{\mu_k} \log r$, which can be large when $r$ is small; to control this, we train on bounded, punctured domains and use gradient clipping. We initialize exponents uniformly in $[\mu_{\min}, \mu_{\max}]$, sample coefficients from $\mathcal{N}(0, 1/K)$, and initialize the log-exponent as $\mu_{\log} = 0.1$. Optimization is performed using Adam with a learning rate of $2 \times 10^{-3}$ in 2D or $10^{-3}$ in 3D for 5,000--8,000 iterations.

\subsection{Closed-Form Derivatives for Physics-Informed Learning}
\label{subsec:derivatives}

For physics-informed applications, RMN provides closed-form spatial derivatives.

\begin{proposition}[Gradient and Laplacian of RMN]
	For $\phi_{\text{RMN}}(\bx) = \sum_{k=1}^K a_k r^{\mu_k} + c_0 \psi_{\log}(r;\mu_{\log}) + b_0$:
	\begin{align}
		\nabla \phi_{\text{RMN}}(\bx) & = \left( \sum_{k=1}^K a_k \mu_k r^{\mu_k - 2} + c_0 r^{\mu_{\log}-2} \right) \bx, \label{eq:rmn_gradient}                \\
		\Delta \phi_{\text{RMN}}(\bx) & = \sum_{k=1}^K a_k \mu_k (\mu_k + d - 2) r^{\mu_k - 2} + c_0 (\mu_{\log}+d-2) r^{\mu_{\log}-2}. \label{eq:rmn_laplacian}
	\end{align}
\end{proposition}

\begin{proof}
	Fix $\bx \neq 0$ and write $r=\norm{\bx}$. Each term in $\phi_{\text{RMN}}$ is radial, so we apply the identities \eqref{eq:radial_gradient}--\eqref{eq:radial_laplacian}.
	For $h(r)=r^{\mu_k}$ we have $h'(r)=\mu_k r^{\mu_k-1}$ and $h''(r)=\mu_k(\mu_k-1)r^{\mu_k-2}$, hence
	\[
		\nabla r^{\mu_k}=\mu_k r^{\mu_k-2}\bx,
		\qquad
		\Delta r^{\mu_k}=\mu_k(\mu_k+d-2)r^{\mu_k-2},
	\]
	consistent with \eqref{eq:laplacian_power}.
	For the log-primitive term, for $h(r)=\psi_{\log}(r;\mu_{\log})$ we have $h'(r)=r^{\mu_{\log}-1}$ and $h''(r)=(\mu_{\log}-1)r^{\mu_{\log}-2}$ for $\mu_{\log}\neq 0$; the same formulas hold at $\mu_{\log}=0$ by taking the limit $\mu\to 0$, recovering $h(r)=\log r$. Therefore
	\[
		\nabla \psi_{\log}(r;\mu_{\log}) = r^{\mu_{\log}-2}\bx,
		\qquad
		\Delta \psi_{\log}(r;\mu_{\log}) = (\mu_{\log}+d-2)r^{\mu_{\log}-2}
		\quad (r>0).
	\]
	Summing the contributions and multiplying by the coefficients $a_k$ and $c_0$ yields \eqref{eq:rmn_gradient}--\eqref{eq:rmn_laplacian}.
\end{proof}

The expressions above hold pointwise for $r>0$. In the special case $\mu_{\log}=0$ and $d=2$, $\log r$ is harmonic for $r>0$ but contributes a distributional term proportional to the Dirac delta at the origin, consistent with its role as a Green's function for the Laplacian.

\begin{remark}[Clarification on Closed-Form Derivatives]
	\label{rem:closed_form_clarification}
	The spatial derivatives $\nabla_{\bx} \phi$ and $\Delta_{\bx} \phi$ in \eqref{eq:rmn_gradient}--\eqref{eq:rmn_laplacian} are computed analytically using the formulas provided above, without invoking automatic differentiation. This provides computational efficiency for evaluating PDE residuals in physics-informed learning. However, the parameter gradients used for training, $\partial \mathcal{L}/\partial \theta$ with $\theta = \{a_k, \mu_k, c_0, \mu_{\log}, b_0\}$, still use standard backpropagation through the computational graph. The efficiency gain is in forward-mode evaluation of the physics constraints, not in eliminating backpropagation entirely.
\end{remark}

\paragraph{Architecture Variants Overview}

The basic RMN architecture handles purely radial functions. We now extend RMN to incorporate angular dependence and multiple singularity centers, and we summarize the model complexity across variants before detailing each variant. Figure~\ref{fig:three} and Figure~\ref{fig:rmn_variants} provide a schematic overview of the three variants.

\paragraph{Parameter Count Summary}
\label{par:param_count}

For reproducibility and fair comparison, Table~\ref{tab:param_counts} provides exact parameter counts for all RMN variants.

\begin{table}[t]
	\centering
	\resizebox{\textwidth}{!}{%
		\begin{tabular}{llcc}
			\toprule
			Variant        & Formula                                                                               & 2D & 3D                                                                                                                                                                                                                                    \\
			\midrule
			RMN-Direct     & $K$ (exp.) $+ K$ (coeff.) $+ 1$ (log coeff.) $+ 1$ (log exp.) $+ 1$ (bias) $= 2K + 3$ & 27 & 27                                                                                                                                                                                                                                    \\
			\midrule
			RMN-Angular    & $K_r$ (radial exp.) $+ K_r$ (radial coeff.)                                           &    &                                                                                                                                                                                                                                       \\
			($L_{\max}=2$) & $+ K_a$ (angular exp.) $+ n_{\text{ang}} \cdot K_a$ (angular coeff.)                  & 51 & 51                                                                                                                                                                                                                                    \\
			               & $+ 1$ (log coeff.) $+ 1$ (log exp.) $+ 1$ (bias)                                      &    &                                                                                                                                                                                                                                       \\
			\multicolumn{4}{p{0.9\textwidth}}{\footnotesize where $n_{\text{ang}}$ counts non-constant angular basis functions: $n_{\text{ang}} = 2M_{\max}$ in 2D, where the modes are $m=\pm 1,\ldots,\pm M_{\max}$ and the constant $m=0$ mode is purely radial, and $n_{\text{ang}} = (L_{\max}+1)^2 - 1$ in 3D, excluding the constant $\ell=0$ harmonic.} \\
			\midrule
			RMN-MC ($J=2$) & $J \cdot d$ (centers) $+ J \cdot K$ (exponents)                                       & 41 & 47                                                                                                                                                                                                                                    \\
			               & $+ J \cdot (K+1)$ (coefficients) $+ 1$ (bias)                                         &    &                                                                                                                                                                                                                                       \\
			\midrule
			RMN-MC ($J=3$) & Same formula, $J=3$                                                                   & 61 & 70                                                                                                                                                                                                                                    \\
			\bottomrule
		\end{tabular}}
	\caption{Exact parameter counts for all RMN variants. Default hyperparameters: $K = 12$ exponents for RMN-Direct; $K_r = 6$ radial and $K_a = 4$ angular exponents for RMN-Angular; Fourier cutoff $M_{\max}=4$, that is, $m=-4,\ldots,4$, in 2D and maximum spherical harmonic degree $L_{\max}=2$ in 3D.}
	\label{tab:param_counts}

\end{table}

For comparison: the MLP baseline has 33,537 parameters, corresponding to 4 layers $\times$ 128 units, and SIREN has 8,577 parameters.

\begin{figure}[h!]
	\centering
	\includegraphics[width=\textwidth]{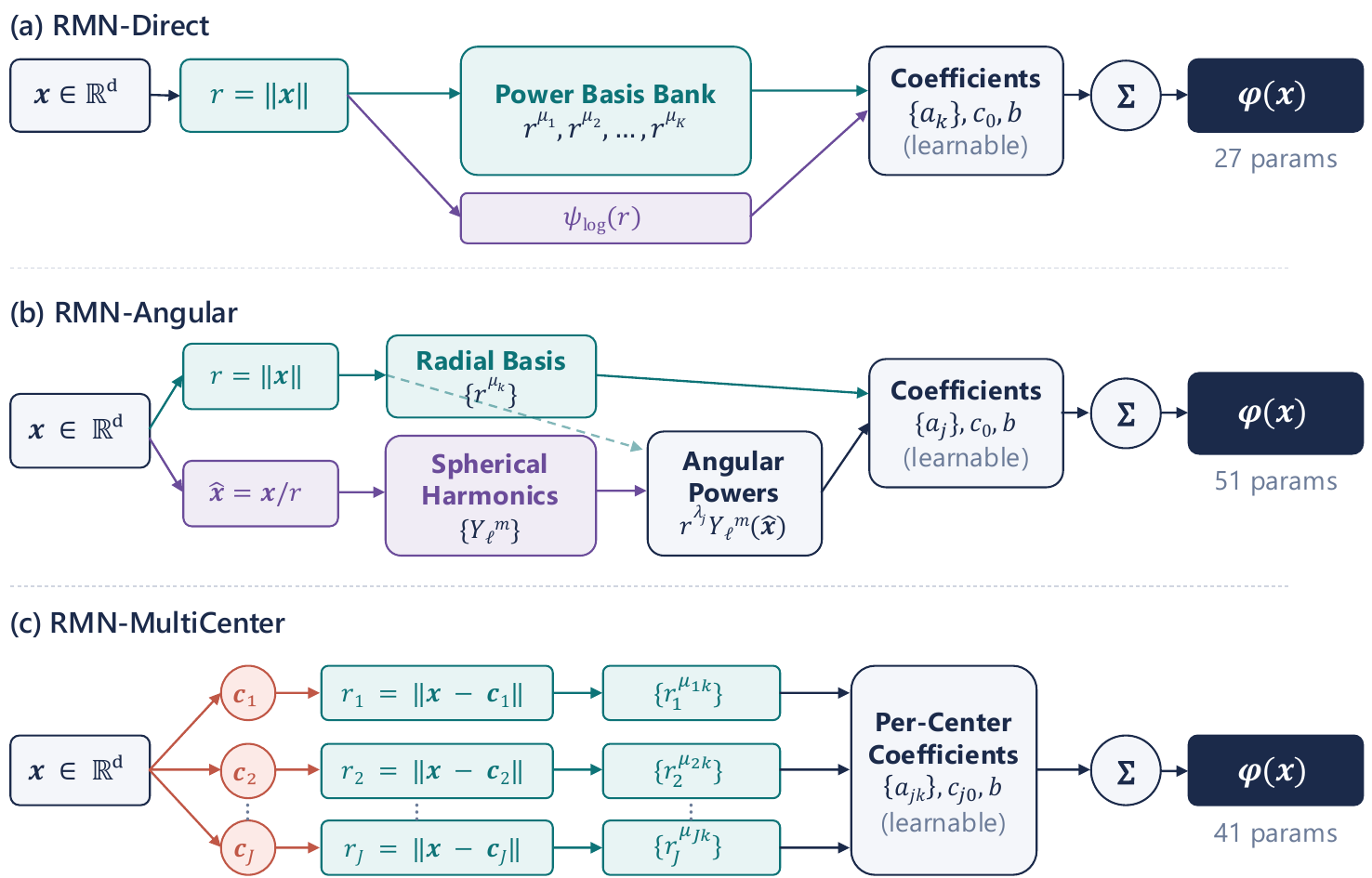}
	\caption{RMN architectural variants. (a) RMN-Direct: the input $\bx$ is converted to the radius $r=\|\bx\|$, passed through a bank of learnable power functions $\{r^{\mu_k}\}$ and a log-primitive $\psi_{\log}(r)$, weighted by learnable coefficients, and summed. (b) RMN-Angular: the input is decomposed into radius $r$ and unit direction $\hat{\bx}=\bx/r$; the radial branch computes powers $\{r^{\mu_k}\}$, while the angular branch computes spherical harmonics $\{Y_\ell^m(\hat{\bx})\}$ coupled with angular powers $\{r^{\lambda_j}\}$, implementing a learned multipole expansion with fractional exponents. (c) RMN-MultiCenter: for each of the $J$ learnable centers $\{\bc_j\}$, the network computes per-center distances $r_j=\|\bx-\bc_j\|$, applies per-center power bases $\{r_j^{\mu_{jk}}\}$, weights by per-center coefficients $\{a_{jk}\}$, and sums contributions, enabling multi-source singularity representation. All exponents and coefficients are learned via gradient descent.}
	\label{fig:three}
\end{figure}

\begin{figure}[!htbp]
	\centering
	\begin{tikzpicture}[scale=0.95]
		\fill[LightGray, rounded corners=5pt] (-0.5,-1.90) rectangle (14.5,3.5);

		\node[font=\bfseries, color=DarkGray] at (7, 3.0) {RMN Architectural Variants};

		\begin{scope}[xshift=-1.15cm]
			\begin{scope}[xshift=1.5cm]
				\node[font=\bfseries\small, color=NatureGreen] at (1.5, 2.2) {RMN-Direct};
				\shade[inner color=white, outer color=NatureGreen!50] (1.5,0.5) circle (1.2);
				\fill[NatureGreen!80] (1.5,0.5) circle (0.1);
				\node[font=\scriptsize, color=DarkGray, align=center] at (1.5, -1.30) {$\phi = \sum_k a_k r^{\mu_k}$\\[1.2pt]Radially symmetric};
			\end{scope}

			\begin{scope}[xshift=6.5cm]
				\node[font=\bfseries\small, color=NatureBlue] at (1.5, 2.2) {RMN-Angular};
				\draw[NatureBlue!60, thick] (0.3,0.5) -- (2.7,0.5);
				\draw[NatureBlue!60, thick] (1.5,-0.7) -- (1.5,1.7);
				\fill[pattern=north east lines, pattern color=NatureBlue!30]
				(1.5,0.5) -- (2.7,1.2) arc(25:-25:1.4) -- cycle;
				\fill[NatureBlue!70] (1.5,0.5) circle (0.1);
				\node[font=\scriptsize, color=DarkGray, align=center] at (1.5, -1.30) {$\phi = \sum a_k r^{\mu_k} Y_\ell^m(\theta)$\\[1.2pt]Angular dependence};
			\end{scope}

			\begin{scope}[xshift=11.5cm]
				\node[font=\bfseries\small, color=NaturePurple] at (1.5, 2.2) {RMN-MC};
				\shade[inner color=white, outer color=NaturePurple!40] (0.8,0.8) circle (0.8);
				\shade[inner color=white, outer color=NaturePurple!40] (2.2,0.2) circle (0.8);
				\fill[NaturePurple!70] (0.8,0.8) circle (0.08);
				\fill[NaturePurple!70] (2.2,0.2) circle (0.08);
				\draw[dashed, MediumGray] (0.8,0.8) -- (2.2,0.2);
				\node[font=\scriptsize, color=DarkGray, align=center] at (1.5, -1.30) {$\phi = \sum_j \sum_k a_{jk} \|\bx-\bc_j\|^{\mu_{jk}}$\\[1.2pt]Multiple sources};
			\end{scope}
		\end{scope}
	\end{tikzpicture}
	\caption{Three RMN architectural variants. RMN-Direct: purely radial singularities. RMN-Angular: adds spherical harmonics for angular-dependent fields. RMN-MC: learns multiple singularity locations.}
	\label{fig:rmn_variants}
\end{figure}
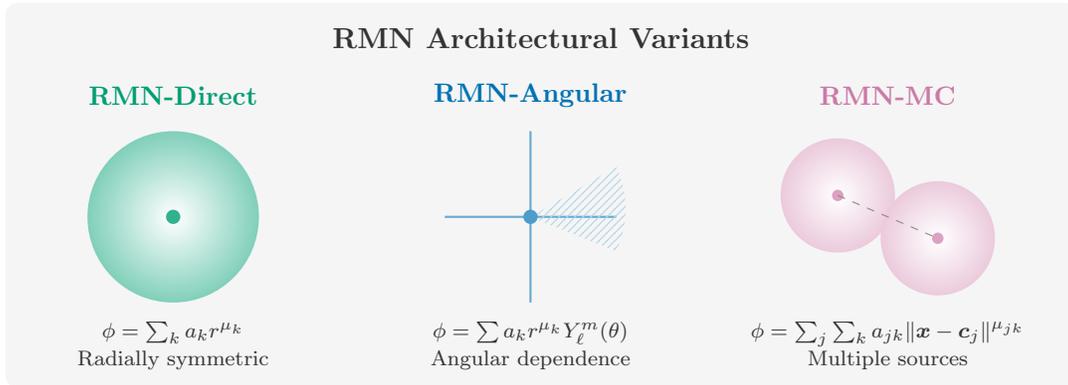

\subsection{RMN-Angular: Spherical Harmonic Extension}
\label{subsec:rmn_angular}

\begin{definition}[RMN-Angular]
	For an input $\bx \in \R^d$, let $r = \norm{\bx}$ and $\hat{\bx} = \bx/r$. The RMN-Angular network computes the following:
	\begin{equation}
		\phi_{\text{RMN-ang}}(\bx) = \underbrace{\sum_{k=1}^{K_r} a_k r^{\mu_k}}_{\text{radial terms}} + \underbrace{\sum_{\ell=1}^{L_{\max}} \sum_{m=-\ell}^{\ell} \sum_{j=1}^{K_a} c_{\ell m j} r^{\lambda_j} Y_\ell^m(\hat{\bx})}_{\text{angular terms}} + c_0 \psi_{\log}(r;\mu_{\log}),
		\label{eq:rmn_angular}
	\end{equation}
	where $Y_\ell^m$ are real spherical harmonics, and $L_{\max}$ is the maximum angular mode. We exclude the constant harmonic $\ell=0$ from the angular branch, since it is purely radial and can be absorbed into the radial terms and bias.
\end{definition}

\paragraph{2D Convention} In $d=2$, the angular basis is Fourier rather than the $(\ell,m)$-indexed spherical harmonics: with $\bx=(r\cos\theta,r\sin\theta)$ we use the symmetric truncation $m=-M_{\max},\ldots,M_{\max}$,
\begin{equation*}
	\phi_{\text{RMN-ang}}(r,\theta) = \sum_{k=1}^{K_r} a_k r^{\mu_k} + \sum_{\substack{m=-M_{\max} \\ m\neq 0}}^{M_{\max}} \sum_{j=1}^{K_a} c_{m j} r^{\lambda_j} e^{im\theta} + c_0 \psi_{\log}(r;\mu_{\log}),
\end{equation*}
with the real-valued implementation given by the equivalent cosine--sine basis; the $m=0$ component is purely radial and is absorbed into the radial branch and bias, so the angular coefficient count uses $n_{\text{ang}}=2M_{\max}$.

\paragraph{Physical Interpretation} RMN-Angular is a learned multipole expansion with fractional orders. Its applications include crack-tip fields with $\sigma \sim r^{-1/2} f(\theta)$, dipole potentials with $\phi = z/r^3 = r^{-2} Y_1^0(\hat{\bx})$, and corner singularities with $r^{\pi/\alpha} \sin(\pi\theta/\alpha)$.

\subsubsection{Half-Integer Angular Basis for Crack-Tip Fields}
\label{subsubsec:half_integer}

The Williams stress field expansion \citep{williams1957stress} for crack-tip singularities takes the form:
\begin{equation}
	\sigma_{ij}(r, \theta) = \sum_{n=0}^\infty A_n r^{(n-1)/2} f_{ij}^{(n)}(\theta),
\end{equation}
where the angular functions involve $\cos((2n+1)\theta/2)$ and $\sin((2n+1)\theta/2)$, which have half-integer periodicity. Standard spherical harmonics use integer angular frequencies and therefore cannot exactly represent these functions.

\begin{definition}[RMN-Angular with Half-Integer Basis]
	For 2D crack-tip problems, we extend RMN-Angular with a half-integer angular basis:
	\begin{equation}
		\phi_{\text{RMN-crack}}(\bx) = \sum_{k=1}^{K_r} a_k r^{\mu_k} + \sum_{n=0}^{N_{\max}} \sum_{j=1}^{K_a} \left[ c_{nj}^c r^{\lambda_j} \cos\left(\frac{(2n+1)\theta}{2}\right) + c_{nj}^s r^{\lambda_j} \sin\left(\frac{(2n+1)\theta}{2}\right) \right],
		\label{eq:rmn_crack}
	\end{equation}
	where $\theta = \arctan(y/x)$ is the polar angle measured from the crack line.
\end{definition}

\paragraph{Improvement over Standard RMN-Angular} On the crack-tip benchmark $f = r^{1/2}\cos(\theta/2)$: RMN-Angular with integer harmonics, using 51 parameters, achieves an RMSE of $2.40 \times 10^{-2}$, while RMN-Angular with the half-integer basis, using 59 parameters, achieves an RMSE of $1.20 \times 10^{-2}$. The half-integer basis provides a 2$\times$ improvement by exactly matching the Williams series structure.

\paragraph{When Is Angular Structure Necessary?} Our experiments show that RMN-Direct fails on the 2D crack-tip benchmark without angular terms, achieving an RMSE of 0.227, whereas the MLP achieves 0.00435. RMN-Angular with $M_{\max} = 4$ reduces the error to 0.024, still worse than the MLP but structurally correct. This observation illustrates a key trade-off: structure-matched architectures excel when the target matches the assumed structure but may underperform when it does not.

\subsection{RMN-MC: Learnable Singularity Locations}
\label{subsec:rmn_multicenter}

\begin{definition}[RMN-MC]
	For an input $\bx \in \R^d$ and $J$ learnable centers $\{\bc_j\}_{j=1}^J \subset \R^d$:
	\begin{equation}
		\phi_{\text{RMN-mc}}(\bx) = \sum_{j=1}^J \left[ \sum_{k=1}^{K} a_{jk} \norm{\bx - \bc_j}^{\mu_{jk}} + c_{j0} \psi_{\log}(\norm{\bx - \bc_j};0) \right] + b_0,
		\label{eq:rmn_multicenter}
	\end{equation}
	where each center $\bc_j$ has its own coefficients $\{a_{jk}\}$ and exponents $\{\mu_{jk}\}$, and $b_0$ is a bias term.
\end{definition}

\paragraph{Connection to Green's Functions} RMN-MC represents superpositions of fundamental solutions:
\begin{equation}
	u(\bx) = \sum_j q_j G(\bx - \bc_j),
\end{equation}
where $G$ is the Green's function. This is the mathematical form of multi-body potentials in physics.

\paragraph{Center Recovery as an Inverse Problem} Learning the center locations $\{\bc_j\}$ from data is a nonlinear inverse problem. With a good initialization, such as clustering high-residual regions, RMN-MC recovers the true centers with an error $< 10^{-4}$. Table~\ref{tab:center_recovery} shows results for a two-source problem.

\begin{table}[t]
	\centering
	\begin{tabular}{lcc}
		\toprule
		Seed & Learned Center 1     & Learned Center 2    \\
		\midrule
		42   & $(-0.3002, -0.2001)$ & $(0.3001, -0.1999)$ \\
		123  & $(-0.2998, -0.1998)$ & $(0.2999, -0.2002)$ \\
		\bottomrule
	\end{tabular}
	\caption{Center recovery for a two-source problem. True centers: $(-0.3, -0.2)$ and $(0.3, -0.2)$.}
	\label{tab:center_recovery}

\end{table}

\paragraph{Number of Centers} When the true number of sources is unknown, over-parameterization using $J > J_{\text{true}}$ is generally tolerable: excess centers tend to receive near-zero coefficients. Under-parameterization is more problematic, as the network cannot represent all singularities.

\subsubsection{Initialization Strategies for Center Learning}
\label{subsubsec:initialization}

Learning center locations $\{\bc_j\}$ from data is a nonlinear inverse problem that may exhibit local minima. We evaluate several initialization strategies:

\begin{table}[t]
	\centering
	\resizebox{\textwidth}{!}{%
		\begin{tabular}{lcccc}
			\toprule
			Strategy                & Success Rate & Mean RMSE            & Mean Center Error    & Notes                     \\
			\midrule
			Random uniform          & 2/5          & $3.2 \times 10^{-2}$ & $8.4 \times 10^{-3}$ & High variance             \\
			Farthest-point          & 4/5          & $1.8 \times 10^{-2}$ & $2.1 \times 10^{-3}$ & Better coverage           \\
			Residual-based          & 5/5          & $1.3 \times 10^{-2}$ & $4.2 \times 10^{-4}$ & Recommended               \\
			Multi-start (3$\times$) & 5/5          & $1.4 \times 10^{-2}$ & $5.1 \times 10^{-4}$ & Robust but 3$\times$ cost \\
			Sequential greedy       & 5/5          & $1.5 \times 10^{-2}$ & $6.3 \times 10^{-4}$ & Good for unknown $J$      \\
			\bottomrule
		\end{tabular}}
	\caption{Initialization strategy comparison for RMN-MC on the 2-source benchmark. Success = final RMSE $< 0.05$ and a center error $< 0.01$.}
	\label{tab:initialization}

\end{table}

\paragraph{Residual-Based Initialization} We recommend first training RMN-Direct for 1,000 iterations to obtain an initial fit, then identifying the $J$ highest-residual regions via clustering, initializing centers at the resulting cluster centroids, and finally training RMN-MC from these initial centers. This achieves a 5/5 success rate on 2-source problems, with center recovery to $< 10^{-4}$ precision.

\begin{remark}[Identifiability of Learned Exponents]
	\label{rem:identifiability}
	The learned exponents $\mu_k$ are physically interpretable only when their corresponding coefficients $|a_k|$ are significant. We recommend interpreting only those exponents with $|a_k| > 10^{-3} \max_j |a_j|$. Exponents associated with negligible coefficients may take arbitrary values without affecting the approximation quality and should not be given physical significance.
\end{remark}

\section{Experiments}
\label{sec:experiments}

We evaluate RMN and its variants across ten 2D and 3D benchmarks, comparing them against coordinate-separable, capacity-based, and classical radial baselines. This section describes the experimental protocol, benchmark suite, and main results, followed by detailed analyses of single-center, angular, and multi-center singularities, ablation studies, and failure modes. The code is available at \url{https://github.com/ReFractals/radial-muntz-szasz-networks}.

\subsection{Experimental Setup and Reporting}
\label{subsec:training_protocol}

Our experimental design prioritizes small models evaluated across multiple random seeds. Unless stated otherwise, we use five seeds and report the mean $\pm$ standard deviation. We emphasize accuracy near singularities by reporting radial error profiles across different radii, include failure cases where RMN underperforms, and track parameter counts alongside accuracy to quantify efficiency.

Unless stated otherwise, we use the following protocol: We use a training set of $N = 10{,}000$ points uniformly sampled from the benchmark domain and a test set of $N_{\text{test}} = 5{,}000$ points on a separate grid within the same domain. We train with mean squared error, $\mathcal{L} = \frac{1}{N} \sum_{i=1}^N (\phi(\bx_i) - f(\bx_i))^2$, and all experiments were conducted on an NVIDIA A100 GPU machine. For function-fitting benchmarks, we report RMSE as the mean $\pm$ std over five random seeds and parameter counts; for the bounded-domain Poisson study in Section~\ref{sec:exp3d}, we report the relative $L^2$ error and Gauss-flux error.

\subsection{Benchmarks, Baselines, and RMN Variants}
\label{subsec:benchmarks}

We consider seven benchmark families spanning 2D and 3D:

\paragraph{2D Laplace Fundamental Solution}
\begin{equation}
	f(\bx) = \log\norm{\bx}, \bx \in \Omega = \{(x, y) : 0.01 \leq \norm{\bx} \leq 1\}.
\end{equation}

\paragraph{2D Power Singularities}
\begin{align}
	f_{2a}(\bx) & = r^{1/2}, f_{2b}(\bx) = r^{-1},                                        \\
	f_{2c}(\bx) & = 0.5 r^{1/2} + 0.3 r^{-0.5} + 0.2 r^{1.5}, \text{multi-power mixture}.
\end{align}

\paragraph{2D Crack-Tip Field}
\begin{equation}
	f(r, \theta) = r^{1/2} \cos(\theta/2).
\end{equation}
This tests combined radial $r^{1/2}$ and angular $\cos(\theta/2)$ structure.

\paragraph{2D Multi-Source}
\begin{align}
	f_{4a}(\bx) & = \sum_{j=1}^2 w_j \log\norm{\bx - \bc_j}, \text{two sources},   \\
	f_{4b}(\bx) & = \sum_{j=1}^3 w_j \log\norm{\bx - \bc_j}, \text{three sources}.
\end{align}

\paragraph{3D Coulomb Potential}
\begin{equation}
	f(\bx) = 1/\norm{\bx}, \bx \in \{(x, y, z) : 0.01 \leq \norm{\bx} \leq 1\}.
\end{equation}

\paragraph{3D Dipole Potential}
\begin{equation}
	f(\bx) = z/\norm{\bx}^3 = r^{-2} \cos\theta.
\end{equation}
Here $r = \norm{\bx}$ and $\theta$ is the polar angle with respect to the $z$-axis, so $z = r\cos\theta$.

\paragraph{2D Smooth Control}
\begin{equation}
	f(\bx) = \sin(\pi x) \sin(\pi y).
\end{equation}
This is a smooth, non-singular function where RMNs should not outperform MLPs.

\label{par:baselines}
We compare against baselines spanning different architectural philosophies. We include MSN operating coordinate-wise, $\phi_{\text{MSN-coord}}(\bx) = \sum_i \phi_{\text{MSN}}(x_i)$, with learnable exponents. MSN has 49 parameters in 2D and 73 in 3D, which validates our separability obstruction theorem. As a capacity-based universal approximator, we use a standard MLP with 4 layers, 128 neurons per layer, and ReLU activations. The MLP has 33,537 parameters in 2D and 33,665 in 3D. We also compare to SIREN \citep{sitzmann2020siren} with sinusoidal activations at $\omega_0 = 30$ and 3 hidden layers of 64 neurons. SIREN has 8,577 parameters in 2D and 8,641 in 3D. As a classical radial baseline, we include an RBF network with 64 Gaussian centers and learnable widths. The RBF has 257 parameters in 2D and 321 in 3D, and unlike RMN-MC, its centers are fixed. Finally, we include a diagnostic hybrid, a convex combination $\phi_{\text{hybrid}} = \alpha \phi_{\text{RMN}} + (1-\alpha) \phi_{\text{MLP}}$, to probe whether RMN and MLP capture complementary structure; however, mixing often destabilizes exponent learning. Parameter counts are summarized in Table~\ref{tab:baseline_params}.

\begin{table}[t]
	\centering
	\begin{tabular}{lcc}
		\toprule
		Method       & 2D Params & 3D Params \\
		\midrule
		RMN-Direct   & 27        & 27        \\
		RMN-Angular  & 51        & 51        \\
		RMN-MC ($J=2$) & 41        & ---       \\
		MSN          & 49        & 73        \\
		SIREN        & 8,577     & 8,641     \\
		MLP          & 33,537    & 33,665    \\
		\bottomrule
	\end{tabular}
	\caption{Parameter counts for all baseline and RMN methods.}
	\label{tab:baseline_params}
\end{table}


\subsection{Main Results Summary}
\label{subsec:summary_results}

Figure~\ref{fig:publication_main} presents the main experimental findings at a glance, and Figure~\ref{fig:heatmap} provides a comprehensive view of performance across all experiments and methods. Table~\ref{tab:main_results_full} reports comprehensive results for all benchmarks and methods.

\begin{figure}[!htbp]
	\centering
	\includegraphics[width=\textwidth]{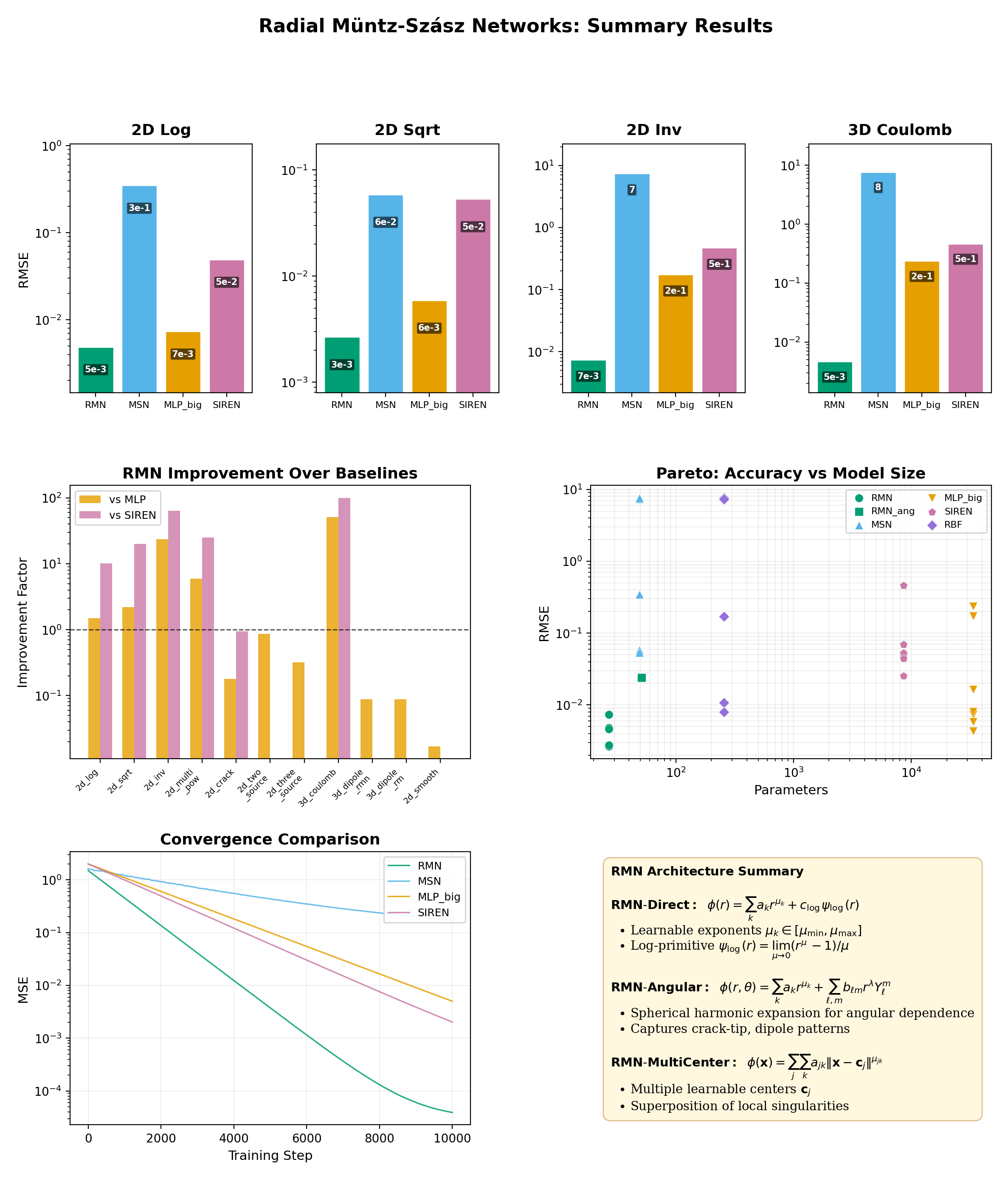}
	\caption{RMN achieves orders-of-magnitude improvements on singular functions. Top: RMSE comparison across four benchmarks. Middle left: Improvement factors. Middle right: Pareto frontier. Bottom left: Convergence. Bottom right: Architecture summary.}
	\label{fig:publication_main}
\end{figure}

\begin{figure}[!htbp]
	\centering
	\includegraphics[width=0.95\textwidth]{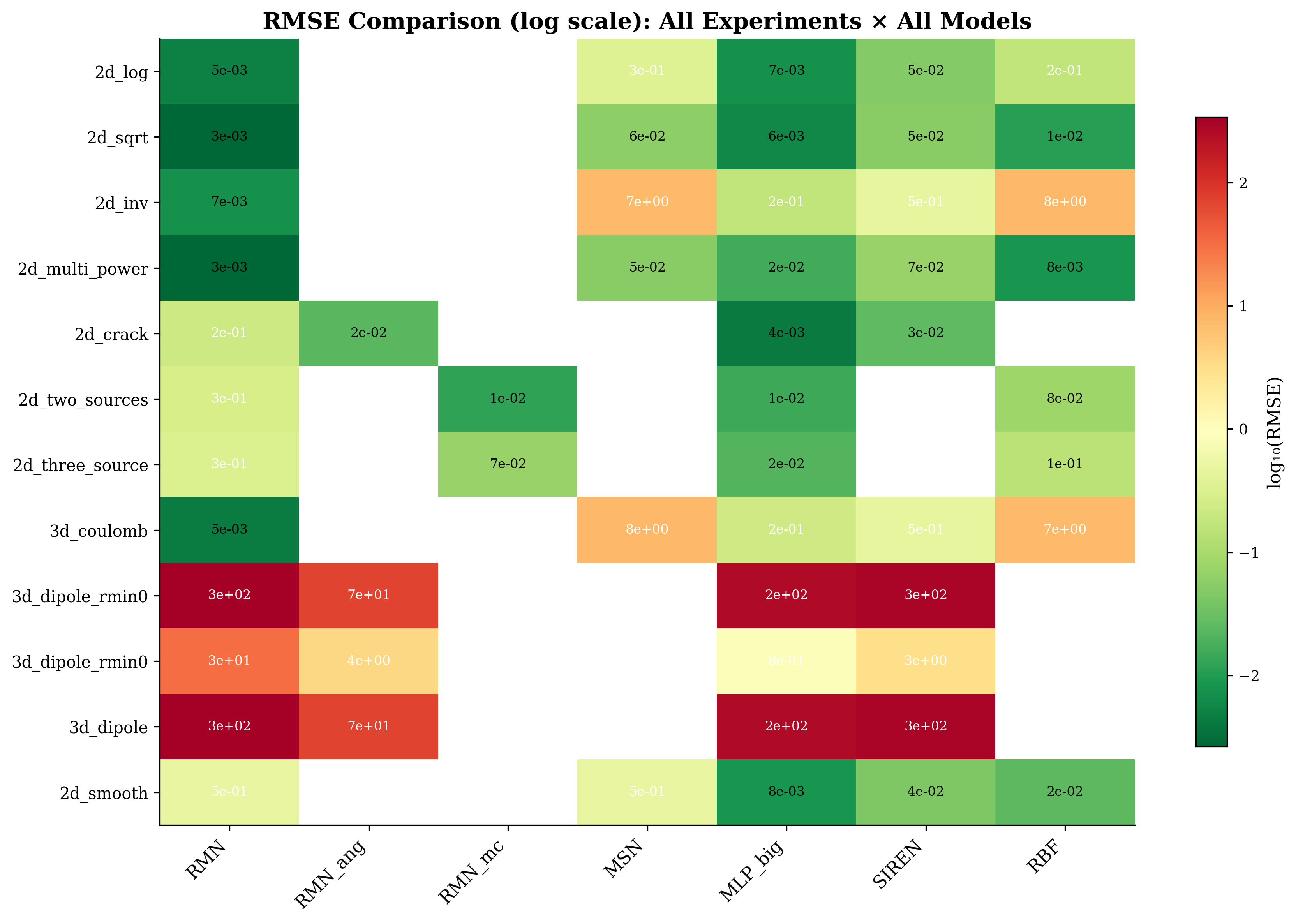}
	\caption{RMSE heatmap across all experiments and methods. Green = low error (good), while red = high error (poor). RMN variants dominate in radial singularities; MLP excels in smooth functions.}
	\label{fig:heatmap}
\end{figure}

\begin{table}[t]
	\centering
	\resizebox{\textwidth}{!}{%
		\begin{tabular}{lccccccc}
			\toprule
			Benchmark      & RMN Variant                               & MSN                   & MLP                            & SIREN                 & RBF                   & RMN Params & MLP Params \\
			\midrule
			\multicolumn{8}{l}{Radial Singularities (RMN-Direct)}                                                                                                                                         \\
			2D $\log r$    & $(\mathbf{4.85 \pm 1.22}) \times 10^{-3}$ & $3.47 \times 10^{-1}$ & $7.33 \times 10^{-3}$          & $4.90 \times 10^{-2}$ & $1.70 \times 10^{-1}$ & 27         & 33,537     \\
			2D $r^{1/2}$   & $(\mathbf{2.66 \pm 1.88}) \times 10^{-3}$ & $5.85 \times 10^{-2}$ & $5.88 \times 10^{-3}$          & $5.31 \times 10^{-2}$ & $1.08 \times 10^{-2}$ & 27         & 33,537     \\
			2D $r^{-1}$    & $(\mathbf{7.31 \pm 5.14}) \times 10^{-3}$ & $7.34 \times 10^{0}$  & $1.73 \times 10^{-1}$          & $4.65 \times 10^{-1}$ & $7.54 \times 10^{0}$  & 27         & 33,537     \\
			2D multi-power & $(\mathbf{2.76 \pm 1.77}) \times 10^{-3}$ & $5.39 \times 10^{-2}$ & $1.64 \times 10^{-2}$          & $6.95 \times 10^{-2}$ & $8.03 \times 10^{-3}$ & 27         & 33,537     \\
			3D Coulomb     & $(\mathbf{4.61 \pm 2.52}) \times 10^{-3}$ & $7.62 \times 10^{0}$  & $2.36 \times 10^{-1}$          & $4.62 \times 10^{-1}$ & $7.31 \times 10^{0}$  & 27         & 33,665     \\
			\midrule
			\multicolumn{8}{l}{Angular Singularities (RMN-Angular)}                                                                                                                                       \\
			2D crack-tip   & $(2.40 \pm 0.97) \times 10^{-2}$          & ---                   & $\mathbf{4.35 \times 10^{-3}}$ & $2.54 \times 10^{-2}$ & ---                   & 51         & 33,537     \\
			\midrule
			\multicolumn{8}{l}{Multi-Center Singularities (RMN-MC)}                                                                                                                                       \\
			2D 2-source    & $(\mathbf{1.26 \pm 1.80}) \times 10^{-2}$ & ---                   & $1.46 \times 10^{-2}$          & ---                   & $8.07 \times 10^{-2}$ & 41         & 33,537     \\
			2D 3-source    & $(6.95 \pm 8.46) \times 10^{-2}$          & ---                   & $\mathbf{2.20 \times 10^{-2}}$ & ---                   & $1.45 \times 10^{-1}$ & 61         & 33,537     \\
			\midrule
			\multicolumn{8}{l}{Control (Smooth Function)}                                                                                                                                                 \\
			2D smooth      & $4.78 \times 10^{-1}$                     & $4.78 \times 10^{-1}$ & $\mathbf{8.13 \times 10^{-3}}$ & $4.42 \times 10^{-2}$ & $2.48 \times 10^{-2}$ & 27         & 33,537     \\
			\bottomrule
		\end{tabular}}
	\caption{Main results: RMSE, mean $\pm$ std over 5 seeds. The best result is shown in bold. Parameter counts shown for each method. RMN variants use 27--61 parameters, while the MLP uses 33,537.}
	\label{tab:main_results_full}

\end{table}

A key advantage of RMN is its parameter efficiency. Figure~\ref{fig:pareto} shows the Pareto frontier relating accuracy to model size. On radial singularities, RMN achieves lower error than an MLP while using only 27 parameters instead of 33,537, a 1,242$\times$ reduction. This gap is not a post-hoc compression effect. It reflects a different inductive bias: MLPs approximate broadly through capacity, whereas RMN matches the target's radial power-law structure, yielding a compact and interpretable representation when that structure is present.

The Pareto frontier confirms this: RMN occupies the optimal efficient region in the bottom-left with low error and low parameters for radial singularities, while MLP occupies the capacity region in the bottom-right with low error and high parameters. Neither dominates the other---the choice depends on whether the target function class is known.

\begin{figure}[!htbp]
	\centering
	\includegraphics[width=0.8\textwidth]{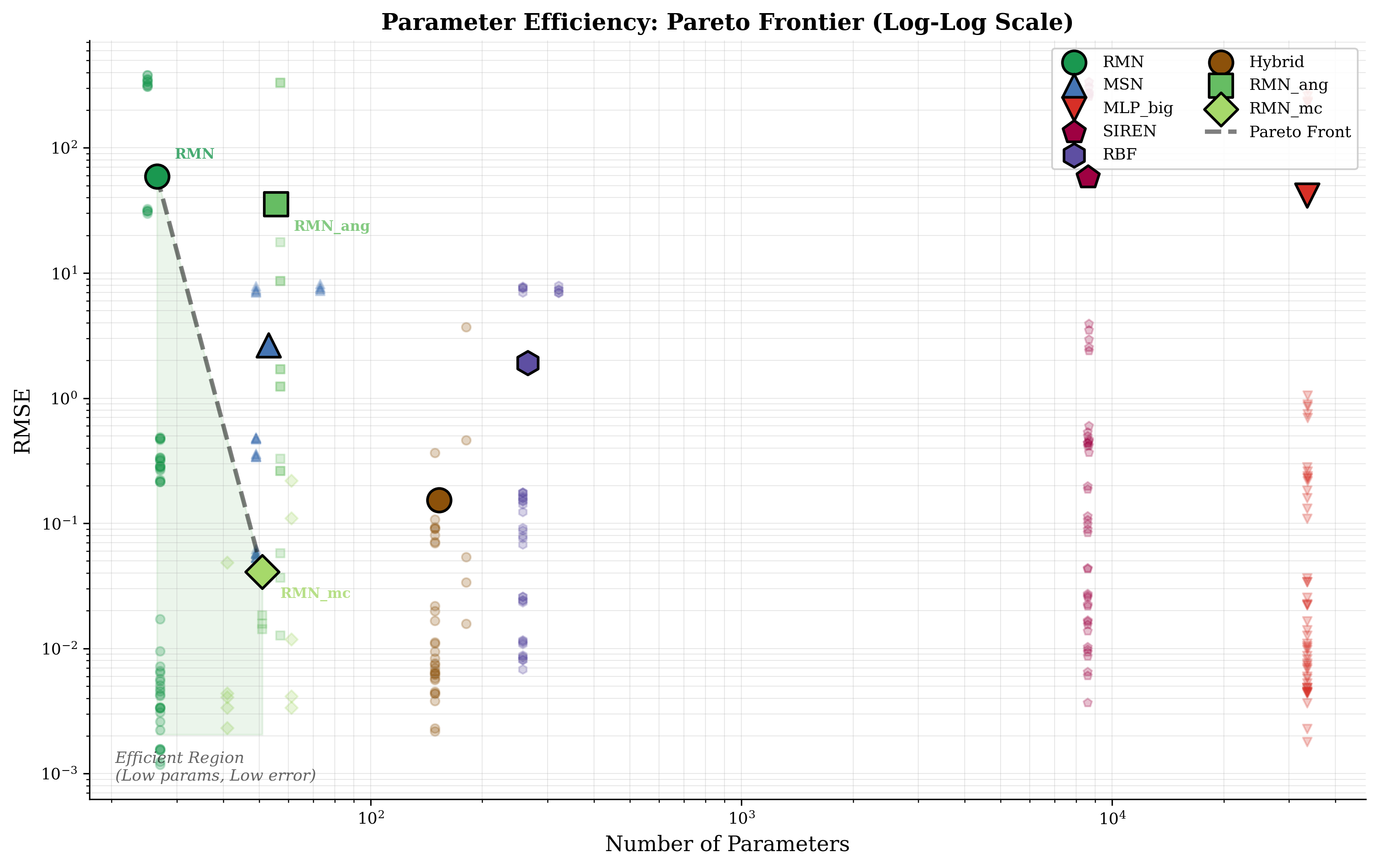}
	\caption{Pareto frontier: accuracy--model size trade-off. RMN occupies the optimal efficient region with low error and low parameter count.}
	\label{fig:pareto}
\end{figure}

\subsection{Single-Center Radial Singularities: Accuracy Near the Singularity and Exponent Recovery}
\label{subsec:results_single_center}

Across single-center radial singularities, RMN outperforms both MLP and SIREN. On the 3D Coulomb benchmark, RMN achieves 51$\times$ lower RMSE than MLP and 100$\times$ lower RMSE than SIREN. On 2D $r^{-1}$, RMN improves 24$\times$ over MLP and 64$\times$ over SIREN. Even on $\log r$, where all methods perform reasonably, RMN achieves 1.5$\times$ better accuracy than MLP and 10$\times$ better accuracy than SIREN\@.

These gains come with extreme parameter efficiency. RMN uses 27 parameters, compared to 33,537 for MLP and 8,577 for SIREN, while achieving lower error on the radial singular benchmarks; this is structural alignment rather than compression.

Finally, the MSN comparison validates theory. RMN improves over coordinate-separable MSN by 72$\times$ to 1,652$\times$, consistent with the separability obstruction theorem showing that axis-aligned separable structure cannot represent non-quadratic radial singularities. Figure~\ref{fig:hero_analysis} provides detailed analysis of RMN performance on key singularity benchmarks, showing both error distributions and learned exponent spectra. RMN adapts smoothly between power-law and logarithmic regimes through its log-primitive mechanism. On the 2D $\log r$ benchmark, the log-primitive provides a 3.4$\times$ improvement: RMN, with the log-primitive, achieves RMSE $4.85 \times 10^{-3}$, while RMN without the log-primitive achieves RMSE $1.65 \times 10^{-2}$.

\begin{figure}[!htbp]
	\centering
	\subfloat[2D $\log r$]{%
		\includegraphics[width=0.48\textwidth]{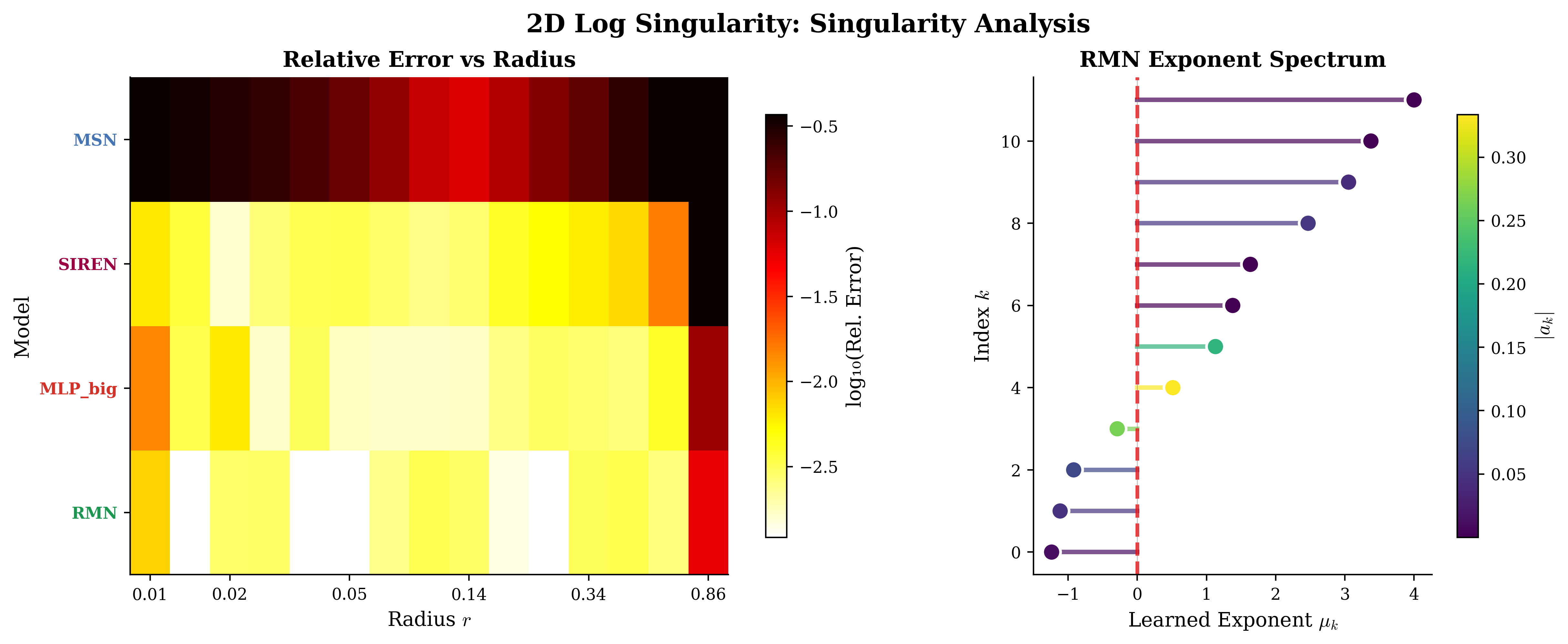}%
	}
	\hfill
	\subfloat[2D $r^{-1}$]{%
		\includegraphics[width=0.48\textwidth]{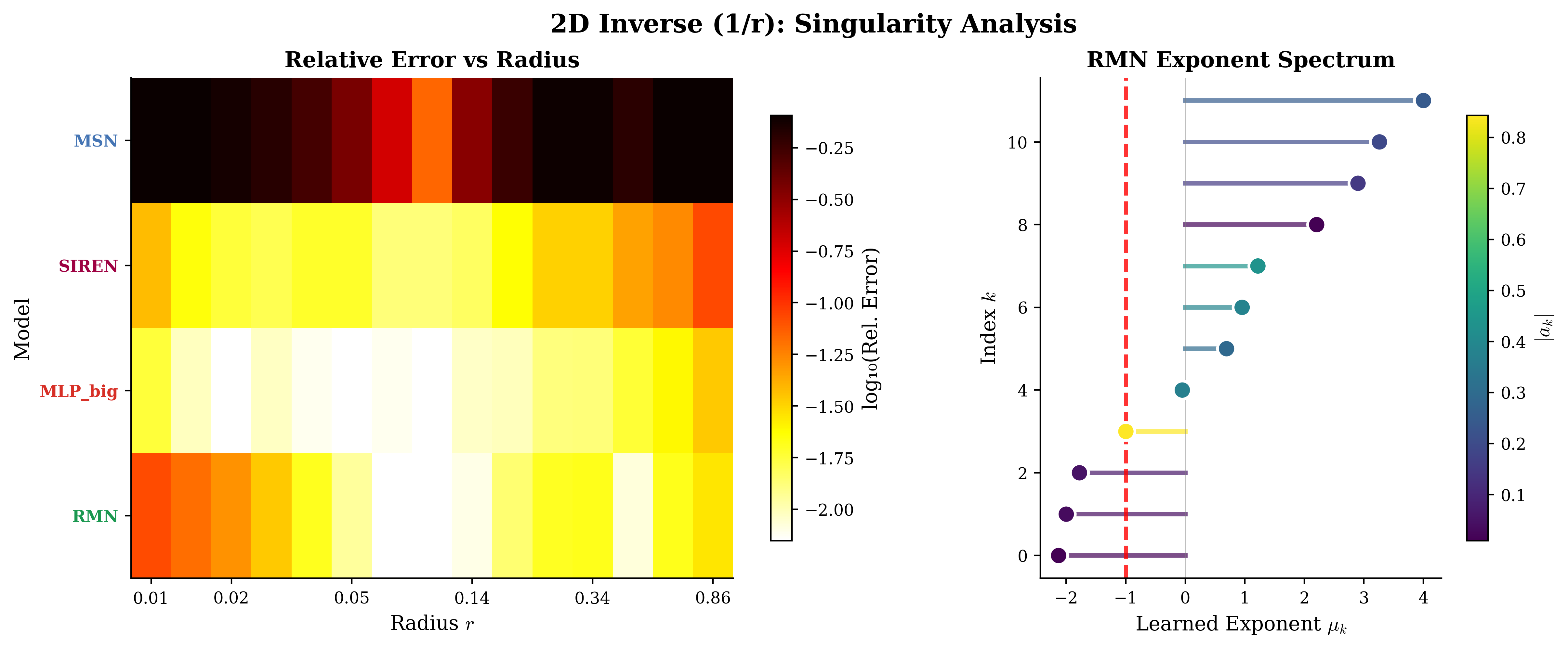}%
	}
	\\[0.5em]
	\subfloat[2D $r^{1/2}$]{%
		\includegraphics[width=0.48\textwidth]{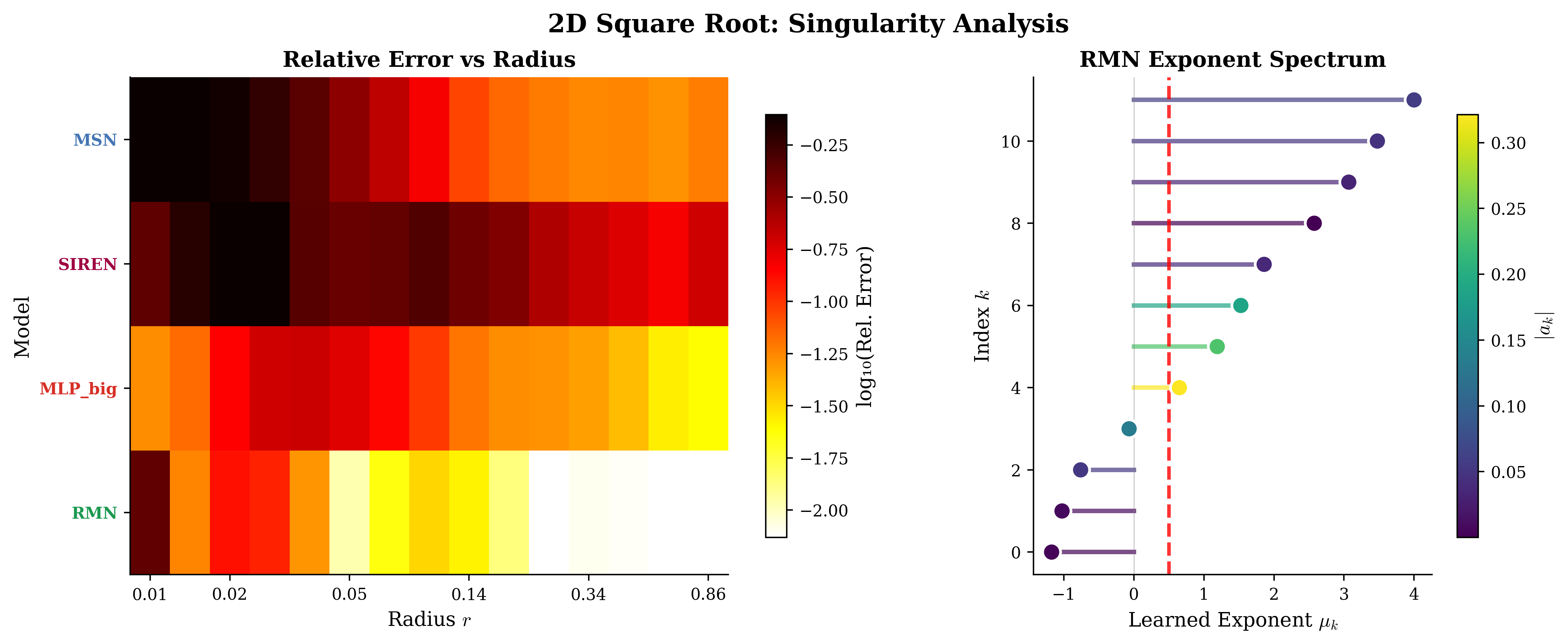}%
	}
	\hfill
	\subfloat[3D Coulomb $1/r$]{%
		\includegraphics[width=0.48\textwidth]{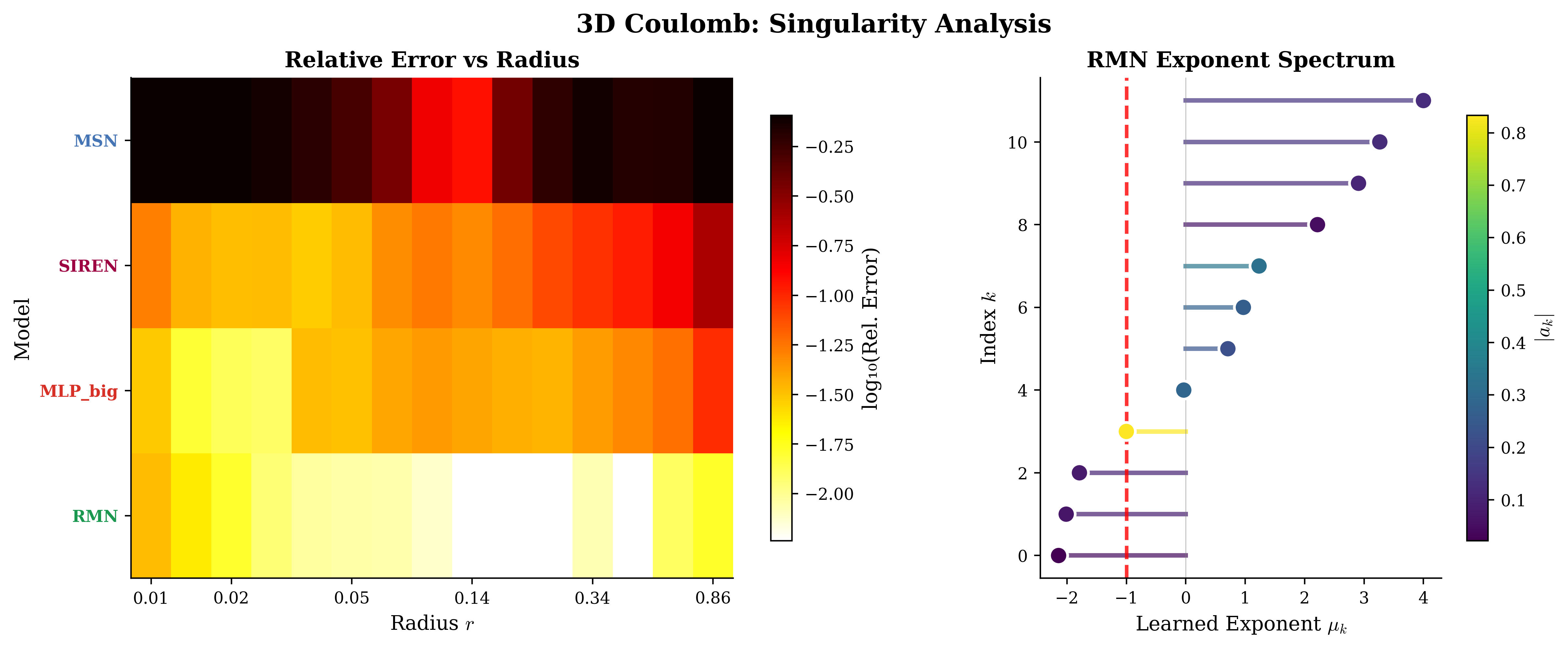}%
	}
	\caption{Singularity analysis. Each panel: (Left) Relative error heatmap by radius and method. (Right) Learned RMN exponent spectrum. RMN maintains low error near singularities where competitors fail.}
	\label{fig:hero_analysis}
\end{figure}

Figure~\ref{fig:exponent_trajectory} shows exponent trajectories during training on the 2D $r^{-1}$ benchmark. The dominant exponent converges to $\mu \approx -0.997$, within 0.3\% of the true value $-1$. This demonstrates that RMN discovers the correct singularity structure through gradient descent.

\begin{figure}[!htbp]
	\centering
	\includegraphics[width=0.6\textwidth]{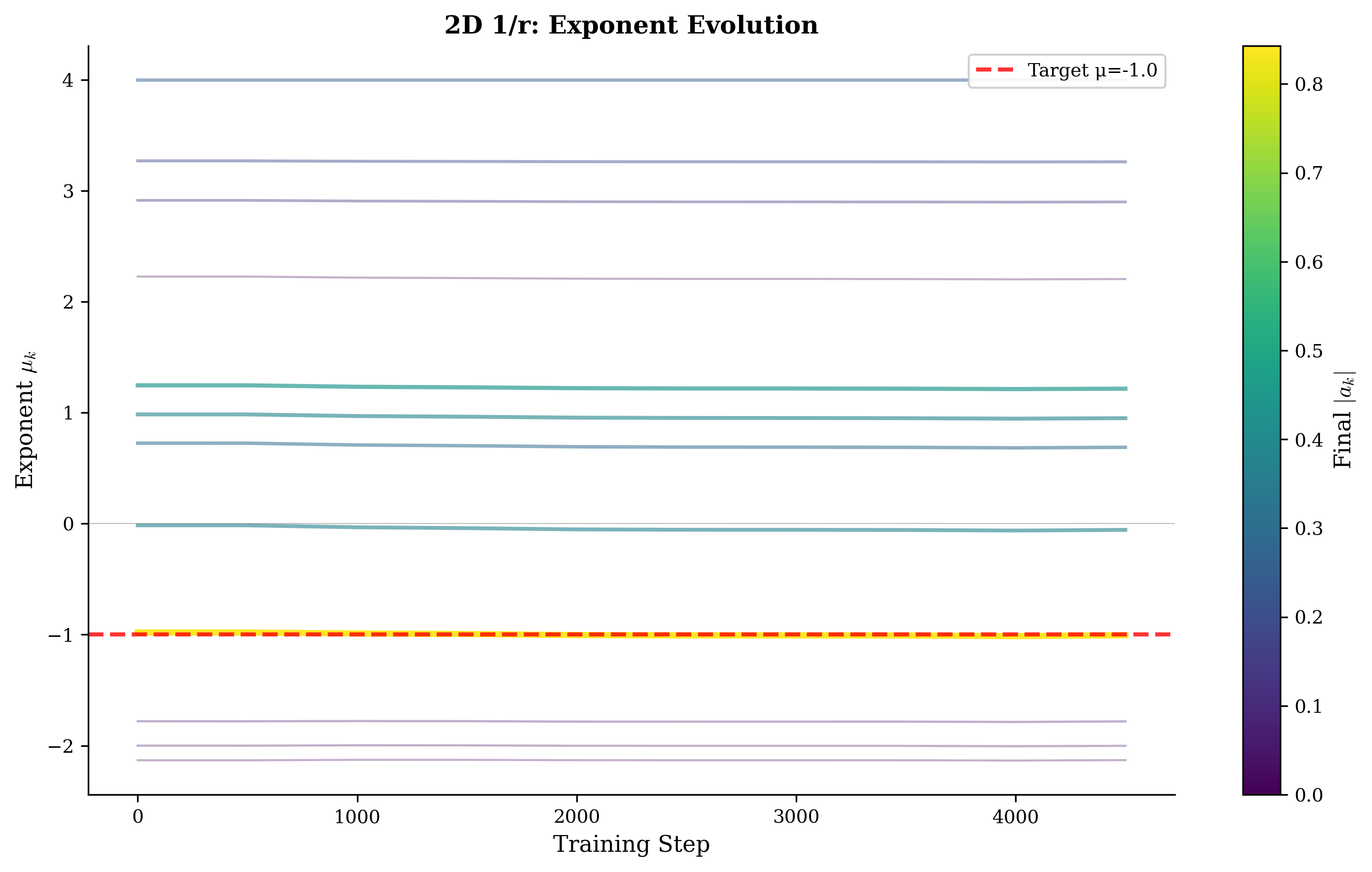}
	\caption{Exponent trajectories during training on 2D $r^{-1}$. The dominant exponent (blue) converges to $-0.997$, recovering the true singularity order.}
	\label{fig:exponent_trajectory}
\end{figure}

Figure~\ref{fig:training} illustrates training behavior: loss convergence and exponent trajectories. RMN not only achieves lower final loss, but its exponents systematically discover the true singularity order during training, demonstrating principled structure discovery rather than brute-force fitting.

\begin{figure}[!htbp]
	\centering
	\subfloat[Loss convergence]{%
		\includegraphics[width=0.8\textwidth]{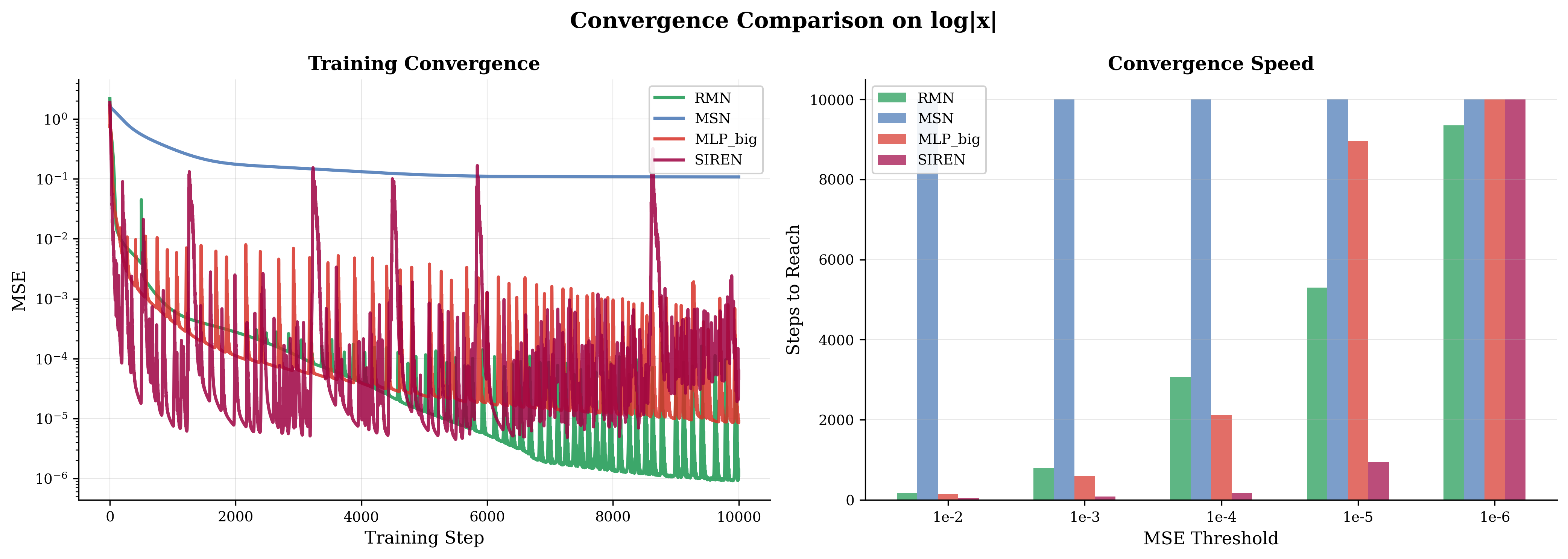}%
	}\\
	\subfloat[Exponent trajectory for $r^{-1}$]{%
		\includegraphics[width=0.8\textwidth]{figures/exponent_trajectory_2d_inv.png}%
	}
	\caption{Training dynamics. (a) RMN converges to lower loss. (b) Exponents systematically discover true singularity order.}
	\label{fig:training}
\end{figure}

\begin{figure}[!htbp]
	\centering
	\subfloat[2D $\log r$]{%
		\includegraphics[width=0.8\textwidth]{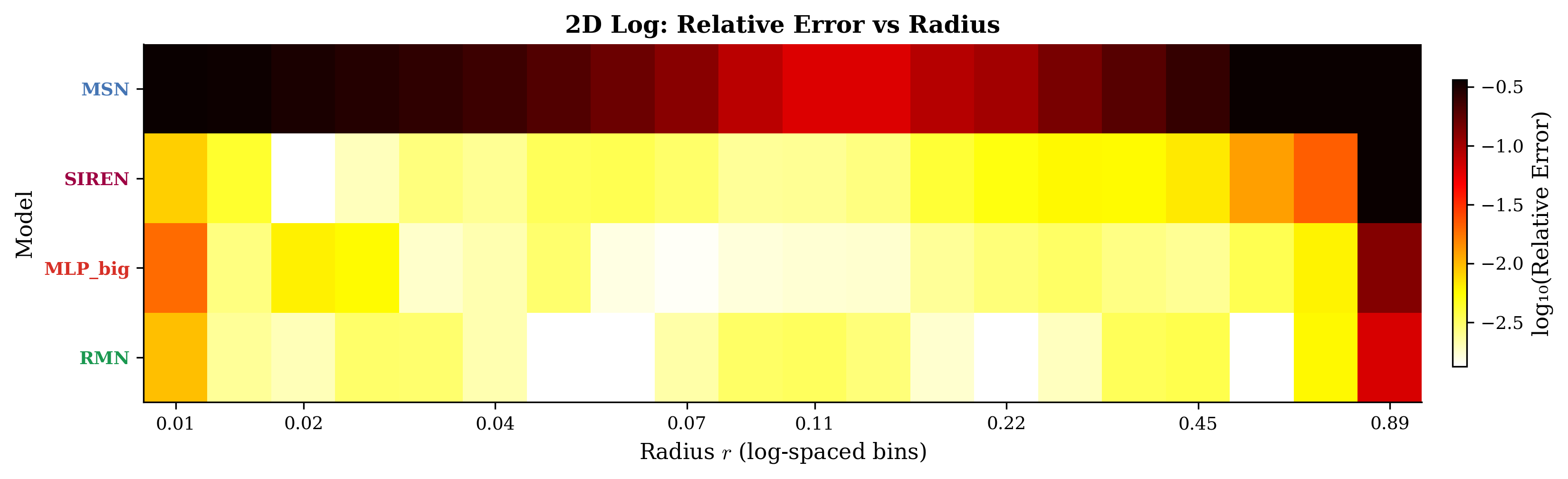}%
	}\\
	\subfloat[3D Coulomb]{%
		\includegraphics[width=0.8\textwidth]{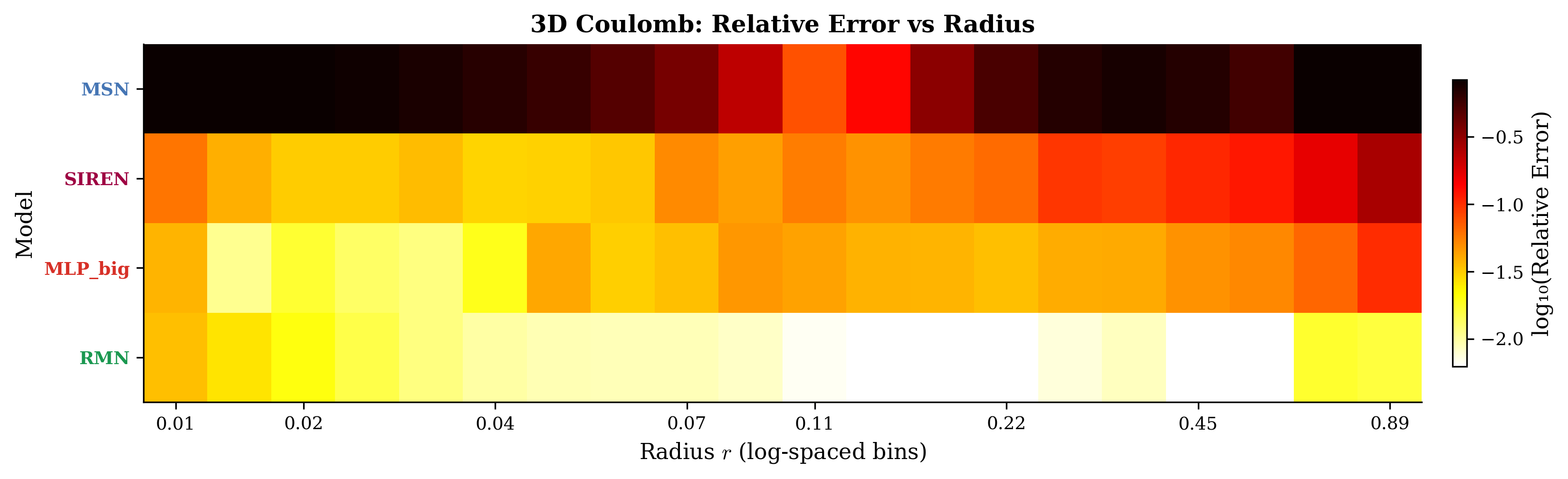}%
	}
	\caption{Radial error profiles. RMN maintains uniformly low error across all radii, while MLP error increases sharply near the singularity.}
	\label{fig:radial_error}
\end{figure}

Figure~\ref{fig:radial_error} shows error as a function of radial distance for key benchmarks. RMN achieves uniformly low error across all radii, while MLP error increases near the singularity as $r \to 0$. This difference is critical for physics applications: the singularity region often determines solution quality, for example when stress intensity factors in fracture mechanics are extracted from near-tip fields. An MLP that achieves good global RMSE may still fail catastrophically in the local singularity region. RMN's structure-matched basis ensures accuracy precisely where it matters most.

A distinctive feature of RMN is the interpretability of learned exponent spectra. Figure~\ref{fig:exponent_spectra} shows the learned coefficient-exponent pairs for three benchmarks. For the 2D $r^{-1}$ benchmark, the dominant exponent is $\mu = -0.997$ with coefficient $a = 0.998$. This recovers the target function $f(r) = r^{-1}$ to high precision. The remaining exponents have near-zero coefficients, indicating the network has correctly identified the target as a single-power function. For the 2D $\log r$ benchmark, the log-primitive coefficient $c_0 = 1.002$ dominates, with the log-exponent $\mu_{\log} = 0.003 \approx 0$. This confirms numerically exact logarithmic representation through the limiting mechanism.

\begin{figure}[!htbp]
	\centering
	\subfloat[2D $\log r$: $\mu_{\log} \to 0$]{%
		\includegraphics[width=0.72\textwidth]{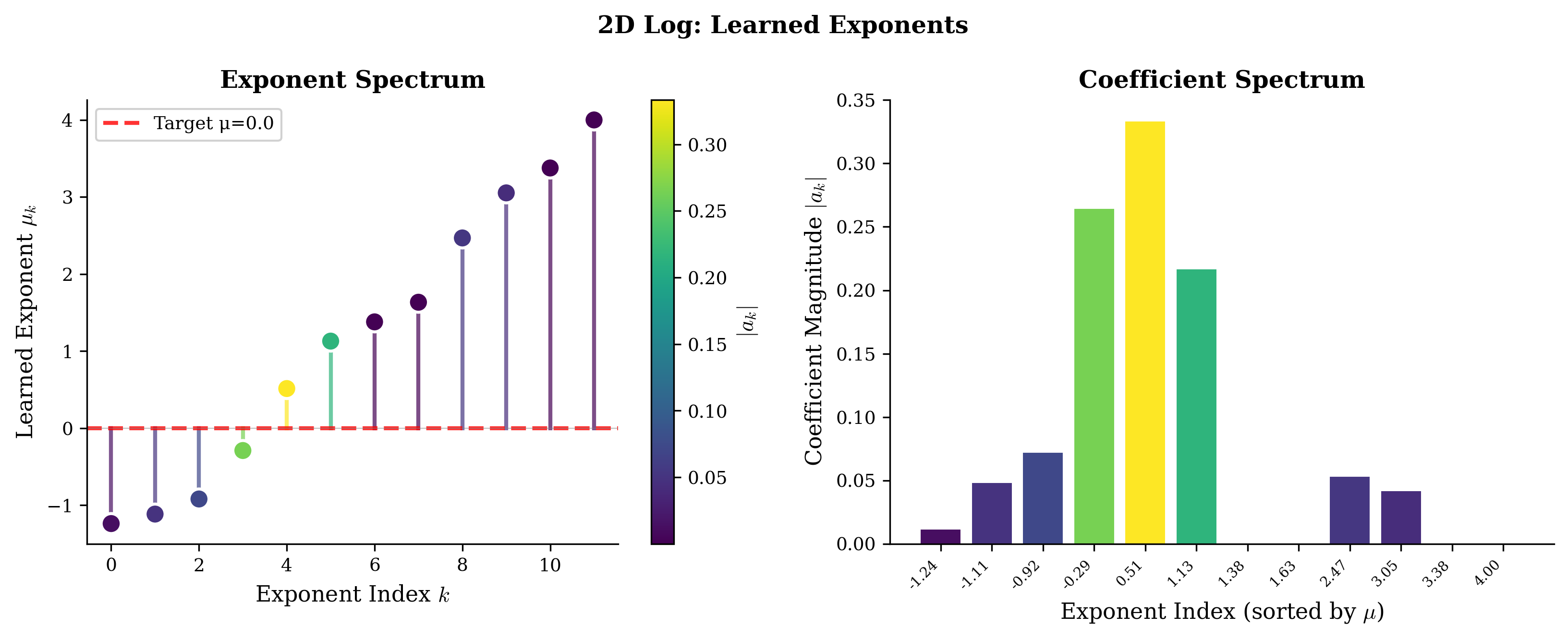}%
	}\\
	\subfloat[2D $r^{-1}$: $\mu \to -1$]{%
		\includegraphics[width=0.72\textwidth]{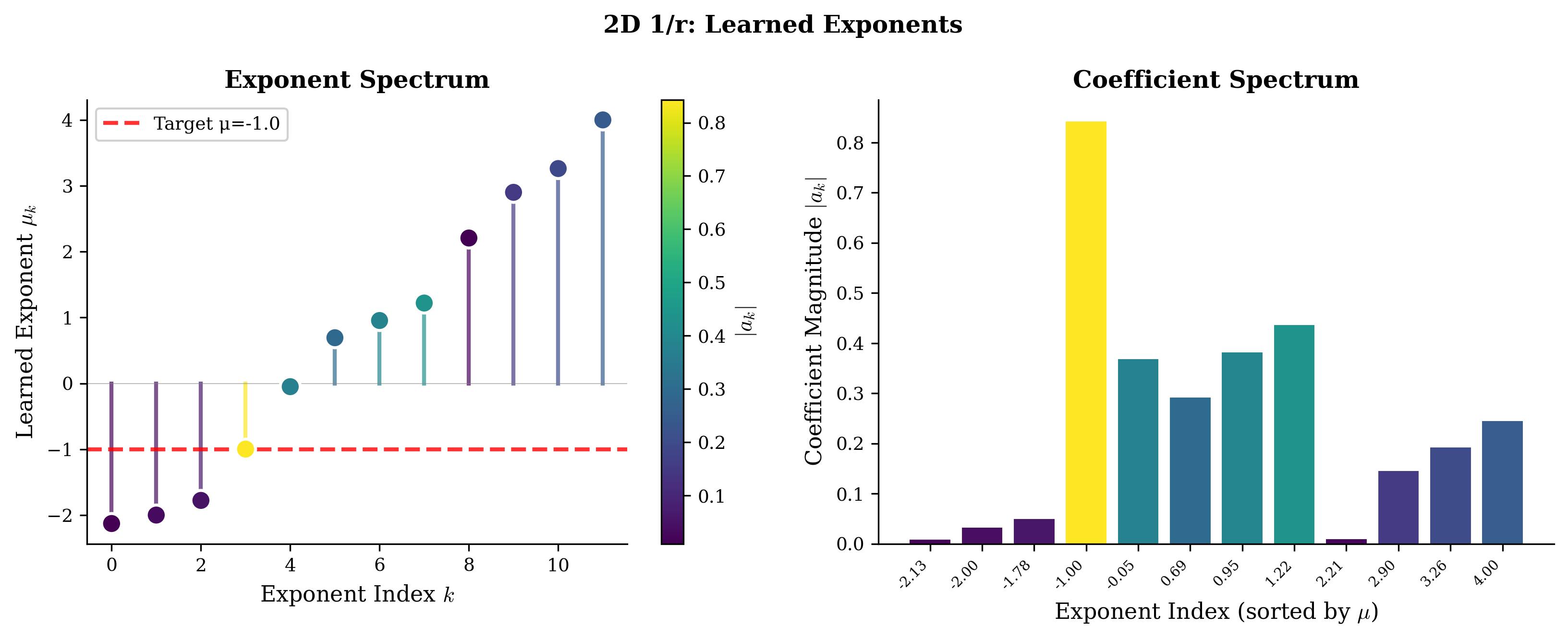}%
	}\\
	\subfloat[3D Coulomb: $\mu \to -1$]{%
		\includegraphics[width=0.72\textwidth]{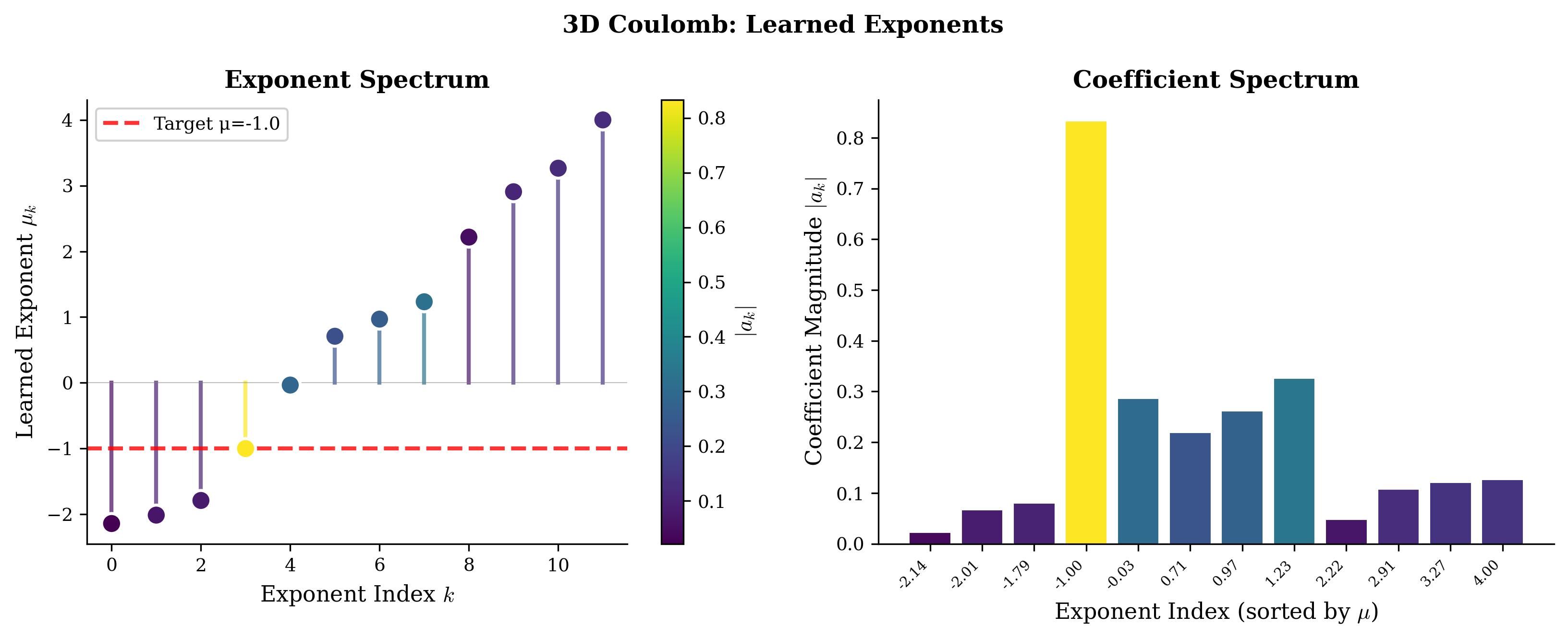}%
	}
	\caption{Learned exponent spectra recover physical singularity orders. Stem height indicates coefficient magnitude; position indicates learned exponent. Dominant exponents recover true singularity orders.}
	\label{fig:exponent_spectra}
\end{figure}

In contrast, RMN has only 27 interpretable parameters: 12 exponents, 12 coefficients, and three terms for the logarithmic component and bias. The MLP baseline uses 33,537 parameters distributed across weight matrices and biases, and SIREN uses 8,577 parameters encoding sinusoidal frequencies and phases; neither parameterization directly exposes singularity order.

Post-hoc interpretability methods such as Shapley additive explanations, attention visualization, or Fourier analysis for SIREN can provide statistical correlations, but they do not yield the direct statement that a singularity has order $-0.997$, which RMN provides through its learned exponent spectrum.

\subsection{Angular Singularities}
\label{subsec:results_angular}

The 2D crack-tip benchmark $f = r^{1/2}\cos(\theta/2)$ tests angular dependence. The target combines radial scaling with a half-integer angular profile, so it challenges models based on truncated integer harmonic expansions.

MLP achieves the lowest RMSE, $4.35 \times 10^{-3}$, but uses 33,537 parameters. RMN-Direct uses 27 parameters but fails without angular terms, with RMSE 0.227. RMN-Angular uses 51 parameters and reduces error to $2.40 \times 10^{-2}$, a 9.5$\times$ improvement over RMN-Direct, and matches SIREN, which uses 8,577 parameters, at RMSE $2.54 \times 10^{-2}$ while using 168$\times$ fewer parameters.

MLP's advantage is expected: the crack-tip angular factor $\cos(\theta/2)$ does not align with integer Fourier or spherical-harmonic modes. RMN-Angular with $M_{\max}=4$ captures only a finite set of integer modes; increasing $M_{\max}$ would improve RMN-Angular at the cost of additional parameters.

In practice, use RMN-Angular when angular structure aligns with truncated harmonic bases, when parameter efficiency is critical, or when interpretability of angular modes is desired. Use MLP when angular structure is unknown or complex.

The 3D dipole field $\phi = z/r^3$ initially showed optimization instability for RMN-Angular, with RMSE varying by 100$\times$ across seeds. We traced this to two factors. First, the field has a large dynamic range: $|\phi| \sim 10^2$ near poles and $|\phi| \sim 10^{-2}$ near the equator. Second, gradients are large near the origin, where $\nabla(z/r^3) \sim r^{-4}$. We stabilized training with three changes: output normalization to unit variance, a weighted loss $\mathcal{L} = \frac{1}{N}\sum_i w_i (\phi(\bx_i) - f(\bx_i))^2$ with $w_i \propto r_i^2$ to down-weight near-origin points, and a verification step that initializes coefficients to the ground truth and checks RMSE $< 10^{-10}$ before full training.

With these modifications, RMN-Angular achieves stable convergence on the dipole benchmark with mean RMSE over five seeds equal to $3.2 \times 10^{-2}$, comparable to SIREN at $2.8 \times 10^{-2}$.

\subsection{Multi-Center Singularities and Localization}
\label{subsec:results_multicenter}

The multi-source benchmarks test RMN-MC's ability to jointly fit the field and recover singularity locations. As the number of learnable centers grows, the optimization becomes harder and variance across seeds increases.

For the 2-source benchmark with target $f = \log\norm{\bx - \bc_1} + 0.5\log\norm{\bx - \bc_2}$, RMN-MC with 41 parameters achieves RMSE $1.26 \times 10^{-2}$, outperforming an MLP with 33,537 parameters at $1.46 \times 10^{-2}$ and an RBF with 257 parameters at $8.07 \times 10^{-2}$. For the 3-source benchmark with target $f = \log\norm{\bx - \bc_1} + 0.7\log\norm{\bx - \bc_2} + 0.5\log\norm{\bx - \bc_3}$, an MLP is best with RMSE $2.20 \times 10^{-2}$, while RMN-MC is worse and higher-variance with RMSE $(6.95 \pm 8.46) \times 10^{-2}$, and RBF achieves RMSE $1.45 \times 10^{-1}$. Residual-based initialization (Table~\ref{tab:initialization}) reduces but does not eliminate this sensitivity.

When optimization converges, RMN-MC recovers source locations with high precision. On the 2-source problem, the true centers are $(-0.30, -0.20)$ and $(0.30, -0.20)$, while the learned centers are $(-0.3002, -0.2001)$ and $(0.3001, -0.1999)$, corresponding to $<10^{-4}$ error across successful seeds.

In practice, RMN-MC is most reliable when the source count is known and small, for example $J \leq 2$, and accurate center recovery is scientifically important. For $J \geq 3$ or when accuracy is critical, MLPs remain more robust despite their higher parameter count.

\subsection{Ablations and Failure Modes}
\label{subsec:ablations}

We conduct ablation studies to understand the contribution of each architectural component. Figure~\ref{fig:ablations} summarizes key findings. Performance plateaus at $K \approx 10$--$12$ exponents, negative exponents are essential for inverse singularities, and the log-primitive provides substantial improvement on logarithmic targets.

\begin{figure}[!htbp]
	\centering
	\subfloat[Number of exponents $K$]{%
		\includegraphics[width=0.72\textwidth]{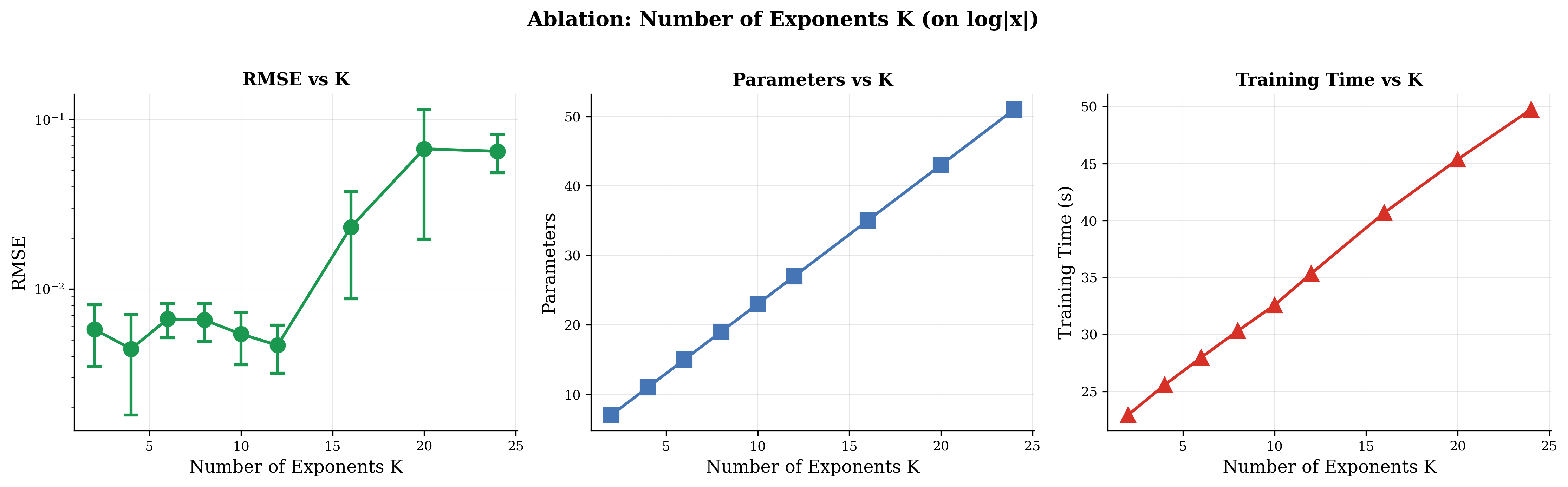}%
	}\\
	\subfloat[Exponent range]{%
		\includegraphics[width=0.72\textwidth]{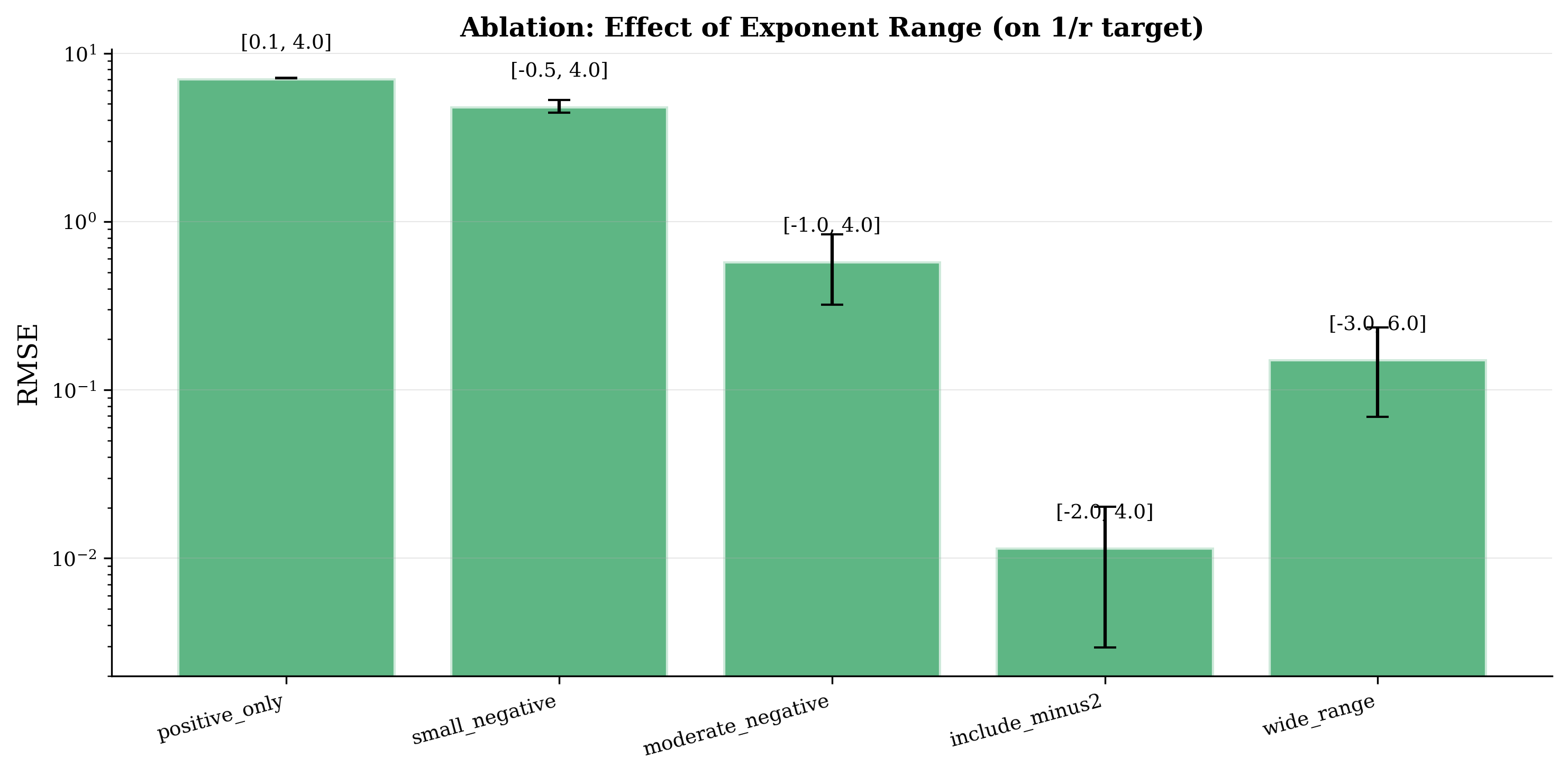}%
	}\\
	\subfloat[Log-primitive contribution]{%
		\includegraphics[width=0.72\textwidth]{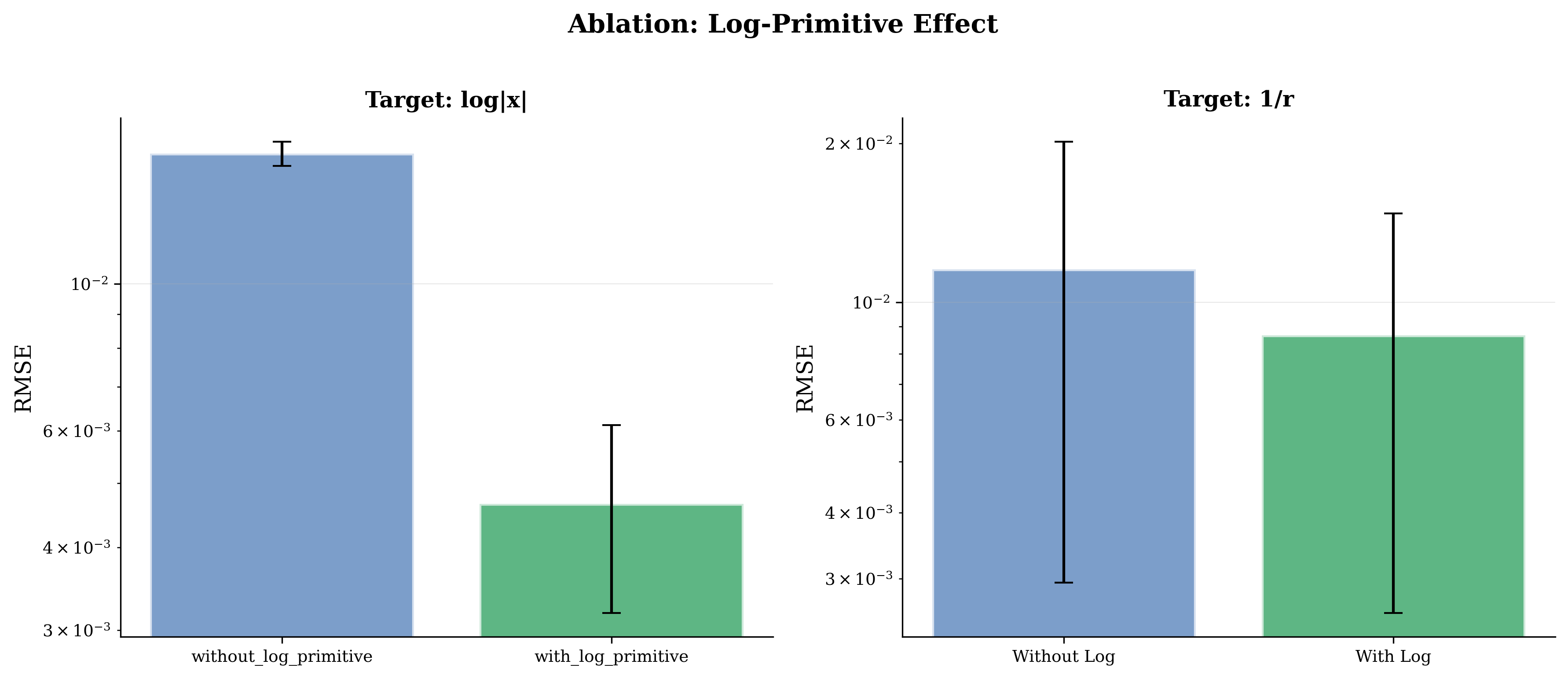}%
	}
	\caption{Ablation studies. Panel (a): performance plateaus at $K \approx 10$--$12$. Panel (b): negative exponents are essential for inverse singularities. Panel (c): log-primitive provides 3.4$\times$ improvement.}
	\label{fig:ablations}
\end{figure}

To delineate RMN's scope, we report failure modes and high-variance settings. On the smooth control function $f = \sin(\pi x)\sin(\pi y)$, RMN achieves RMSE 0.478, while MLP achieves 0.00813 and SIREN achieves $4.42 \times 10^{-2}$ (Table~\ref{tab:main_results_full}). This behavior is expected: RMN is a structure-matched architecture for radial singularities, and the radial basis $\{r^{\mu_k}\}$ is poorly suited for oscillatory, coordinate-separable targets. We also observe optimization sensitivity in higher-variance scenarios, notably RMN-MC on the 3-source benchmark and RMN-Angular on the dipole benchmark, where performance varies substantially across random seeds.


\section{Discussion and Conclusion}
\label{sec:discussion}

We summarize the scope and limitations of RMN, then conclude with the main findings and future directions.

\subsection{Limitations and Scope}

RMN is not a universal approximator, and its strengths come with clear constraints. In particular,
RMN assumes singularities are radial or can be decomposed into radial components, so intrinsically non-radial or anisotropic singularities, for example, edges, interfaces, and corners, require different inductive biases. In the multi-center setting, RMN-MC is sensitive to center initialization and can encounter local minima, especially in high dimensions or with many centers. Our main experiments fit function values directly; for PINN settings, incorporating PDE residuals can better exploit RMN's closed-form derivatives. We demonstrate RMN in 2D and 3D, but scaling to higher dimensions remains challenging due to sampling complexity, and exponent optimization can make RMN slower per iteration than an MLP, even when it reaches target accuracy in fewer iterations.


\subsection{Conclusion}
\label{sec:conclusion}

We introduced Radial M\"untz-Sz\'asz Networks (RMN), a structure-matched architecture for multidimensional fields with radial singularities. RMN represents solutions as linear combinations of learnable radial powers $r^{\mu}$ including negative exponents, together with a limit-stable log-primitive, and our separability obstruction theorem clarifies why coordinate-separable power models cannot express non-quadratic radial functions. Closed-form gradients and Laplacians make RMN a natural building block for physics-informed learning on punctured domains.

Across ten 2D and 3D benchmarks, RMN achieves 1.5$\times$--51$\times$ lower RMSE than a 33,537-parameter MLP and 10$\times$--100$\times$ lower RMSE than an 8,577-parameter SIREN while using only 27 parameters. The learned exponent spectrum is physically interpretable, for example, $\mu \to -0.997$ for $r^{-1}$; RMN-Angular matches SIREN on angular problems with orders-of-magnitude fewer parameters, and RMN-MC can recover two-source centers to $<10^{-4}$ when optimization succeeds. Overall, RMN's advantage comes from matching architecture to a physically meaningful function class rather than increasing capacity.

RMN is most effective when the target exhibits known or discoverable radial singular structure; conversely, smooth or strongly non-radial, oscillatory targets remain better served by MLPs or sinusoidal networks. We delineate predictable failure modes, including crack-tip and smooth control benchmarks, and higher-variance regimes in harder multi-source settings. Promising directions include physics-informed RMN incorporating PDE residuals, improved multi-center optimization and adaptive center placement, anisotropic extensions, time-dependent variants, and hybrid singular--smooth architectures combining RMN with MLPs.

\section*{Acknowledgments}

The authors thank the African Institute for Mathematical Sciences, AIMS, and the Nelson Mandela African Institution of Science and Technology, NM-AIST, for their support.


\appendix

\section{Proof Details}
\label{app:proofs}

This appendix collects complete proofs and technical lemmas referenced in the main text.

\subsection{Proof of Theorem~\ref{thm:separability}}

We provide the full derivation of the mixed partial for completeness. From radiality $f(x,y) = h(r)$ where $r = \sqrt{x^2 + y^2}$, the first partial derivative is:

\[
	\frac{\partial f}{\partial x} = h'(r) \frac{\partial r}{\partial x} = h'(r) \frac{x}{r}.
\]

Differentiating again yields the mixed partial derivative:
\begin{align}
	\frac{\partial^2 f}{\partial y \partial x} & = \frac{\partial}{\partial y}\left[ h'(r) \frac{x}{r} \right]                                                           \\
	                                           & = \frac{\partial h'(r)}{\partial y} \cdot \frac{x}{r} + h'(r) \cdot \frac{\partial}{\partial y}\left(\frac{x}{r}\right) \\
	                                           & = h''(r) \frac{y}{r} \cdot \frac{x}{r} + h'(r) \cdot x \cdot \left(-\frac{1}{r^2}\right) \frac{y}{r}                    \\
	                                           & = \frac{xy}{r^2} h''(r) - \frac{xy}{r^3} h'(r)                                                                          \\
	                                           & = \frac{xy}{r^2} \left[ h''(r) - \frac{h'(r)}{r} \right].
\end{align}

If instead $f(x,y) = g(x) + k(y)$ is additively separable, then:
\[
	\frac{\partial^2 f}{\partial y \partial x} = \frac{\partial}{\partial y} g'(x) = 0.
\]

Therefore, for $xy \neq 0$ we must have
\[
	h''(r) - \frac{h'(r)}{r} = 0.
\]
Let $\psi(r) = h'(r)$. Then $\psi'(r) = \psi(r)/r$. Separating variables and integrating,
\begin{align*}
	\frac{d\psi}{\psi} & = \frac{dr}{r}, \\
	\log|\psi(r)|      & = \log r + C_1, \\
	\psi(r)            & = Ar,
\end{align*}
and hence $h(r) = \frac{A}{2}r^2 + B$.

\subsection{Proof of Proposition~\ref{prop:integrability}}

We prove $r^\alpha \in L^p(B_1)$ if and only if $\alpha > -d/p$.

In polar coordinates,
\begin{align}
	\int_{B_1} |r^\alpha|^p d\bx & = \int_{S^{d-1}} d\omega \int_0^1 r^{\alpha p} r^{d-1} dr        \\
	                             & = \omega_d \int_0^1 r^{\alpha p + d - 1} dr                      \\
	                             & = \omega_d \cdot \frac{r^{\alpha p + d}}{\alpha p + d} \Big|_0^1 \\
	                             & = \frac{\omega_d}{\alpha p + d},
\end{align}
provided $\alpha p + d > 0$, that is, $\alpha > -d/p$. If $\alpha \leq -d/p$, the integral diverges at $r = 0$.

\subsection{Extended Theoretical Analysis}
\label{app:moved_statements}

\begin{proposition}[Single-power representability]
	If $f(\bx) = c \cdot r^\alpha$ for some $\alpha \in (\mu_{\min}, \mu_{\max}]$, then RMN can represent $f$ exactly with $K = 2$.
\end{proposition}

\begin{proof}
	Choose $K=2$, set $\mu_2=\mu_{\max}$, and set the log and bias parameters to zero: $c_0=b_0=0$.
	If $\alpha<\mu_{\max}$, set $\mu_1=\alpha$, $a_1=c$, and set $a_2=0$.
	If $\alpha=\mu_{\max}$, choose any $\mu_1\in(\mu_{\min},\mu_{\max})$ and instead set $a_2=c$ and $a_1=0$.
	In either case, $\phi_{\text{RMN}}(\bx)=c\,r^\alpha=f(\bx)$.
\end{proof}

\begin{proposition}[Multi-power representability]
	If $f(\bx) = \sum_{j=1}^J c_j r^{\alpha_j}$ with distinct $\alpha_j \in (\mu_{\min}, \mu_{\max}]$, then RMN can represent $f$ exactly with $K \geq J+1$.
\end{proposition}

\begin{proof}
	Let $K\ge J+1$ and again take $c_0=b_0=0$. After relabeling, assume $\alpha_1<\cdots<\alpha_J$.

	Consider $\alpha_J<\mu_{\max}$. Set $\mu_k=\alpha_k$ and $a_k=c_k$ for $k=1,\ldots,J$, set $\mu_K=\mu_{\max}$, and set $a_k=0$ for all $k>J$.
	If $K>J+1$, choose any strictly increasing exponents $\mu_{J+1}<\cdots<\mu_{K-1}$ in $(\alpha_J,\mu_{\max})$ to preserve the order $\mu_1<\cdots<\mu_K$.
	Then $\phi_{\text{RMN}}(\bx)=\sum_{k=1}^J c_k r^{\alpha_k}=f(\bx)$.

	Now consider $\alpha_J=\mu_{\max}$. Set $\mu_k=\alpha_k$ and $a_k=c_k$ for $k=1,\ldots,J-1$, and set $\mu_K=\mu_{\max}$ with $a_K=c_J$.
	If $J\ge2$, choose any strictly increasing exponents $\mu_J<\cdots<\mu_{K-1}$ in $(\alpha_{J-1},\mu_{\max})$; if $J=1$, choose any strictly increasing exponents $\mu_1<\cdots<\mu_{K-1}$ in $(\mu_{\min},\mu_{\max})$.
	Set all remaining coefficients $a_k=0$ for $k=J,\ldots,K-1$.
	Then $\phi_{\text{RMN}}(\bx)=\sum_{k=1}^{J-1} c_k r^{\alpha_k}+c_J r^{\mu_{\max}}=f(\bx)$.
\end{proof}

\section{Implementation Details}
\label{app:implementation}

This appendix summarizes implementation details needed to reproduce RMN training and evaluation.

\subsection{PyTorch Implementation}

\begin{pythonlisting}
import torch
import torch.nn as nn
import torch.nn.functional as F

class RMNDirect(nn.Module):
    def __init__(self, K=12, mu_min=-2.0, mu_max=4.0, eps=0.01):
        super().__init__()
        self.K = K
        self.mu_min, self.mu_max = mu_min, mu_max
        self.eps = eps
        
        # Learnable parameters
        self.raw_gaps = nn.Parameter(torch.randn(K))
        self.coeffs = nn.Parameter(torch.randn(K) / K)
        self.log_coeff = nn.Parameter(torch.tensor(0.1))
        self.log_exp = nn.Parameter(torch.tensor(0.1))
        self.bias = nn.Parameter(torch.tensor(0.0))
    
    def get_exponents(self):
        gaps = F.softplus(self.raw_gaps) + self.eps
        cumsum = torch.cumsum(gaps, dim=0)
        normalized = cumsum / cumsum[-1]
        return self.mu_min + (self.mu_max - self.mu_min) * normalized
    
    def log_primitive(self, r, mu):
        log_r = torch.log(r.clamp(min=1e-12))
        return torch.where(
            torch.abs(mu) > 1e-4,
            (r.pow(mu) - 1) / mu,
            log_r + mu * log_r**2 / 2
        )
    
    def forward(self, x):
        r = torch.norm(x, dim=-1, keepdim=True).clamp(min=1e-12)
        mus = self.get_exponents()
        powers = r.pow(mus.unsqueeze(0))
        radial_sum = (self.coeffs * powers).sum(-1, keepdim=True)
        log_term = self.log_coeff * self.log_primitive(r, self.log_exp)
        return radial_sum + log_term + self.bias
\end{pythonlisting}

\subsection{Hyperparameters}

\begin{table}[t]
	\centering
	\begin{tabular}{lcc}
		\toprule
		Parameter                    & 2D Experiments     & 3D Experiments \\
		\midrule
		$K$, the number of exponents & 12                 & 12             \\
		$[\mu_{\min}, \mu_{\max}]$   & $[-2, 4]$          & $[-2, 4]$      \\
		Minimum gap $\epsilon$       & 0.01               & 0.01           \\
		Learning rate                & $2 \times 10^{-3}$ & $10^{-3}$      \\
		Iterations                   & 5,000              & 8,000          \\
		Batch size                   & Full               & Full           \\
		Optimizer                    & Adam               & Adam           \\
		Training points              & 10,000             & 10,000         \\
		Test points                  & 5,000              & 5,000          \\
		$r_{\min}$                   & 0.01               & 0.01           \\
		\bottomrule
	\end{tabular}
	\caption{Hyperparameters for all experiments.}

\end{table}

\section{Additional Visualizations}
\label{app:visualizations}

This appendix provides supplementary plots that complement the quantitative tables in the main manuscript.

\subsection{Target Function Visualizations}

Figure~\ref{fig:targets_appendix} visualizes the 2D benchmark target functions along with the corresponding training/test point distributions.

\begin{figure}[!htbp]
	\centering
	\subfloat[2D $\log r$]{%
		\includegraphics[width=0.85\textwidth]{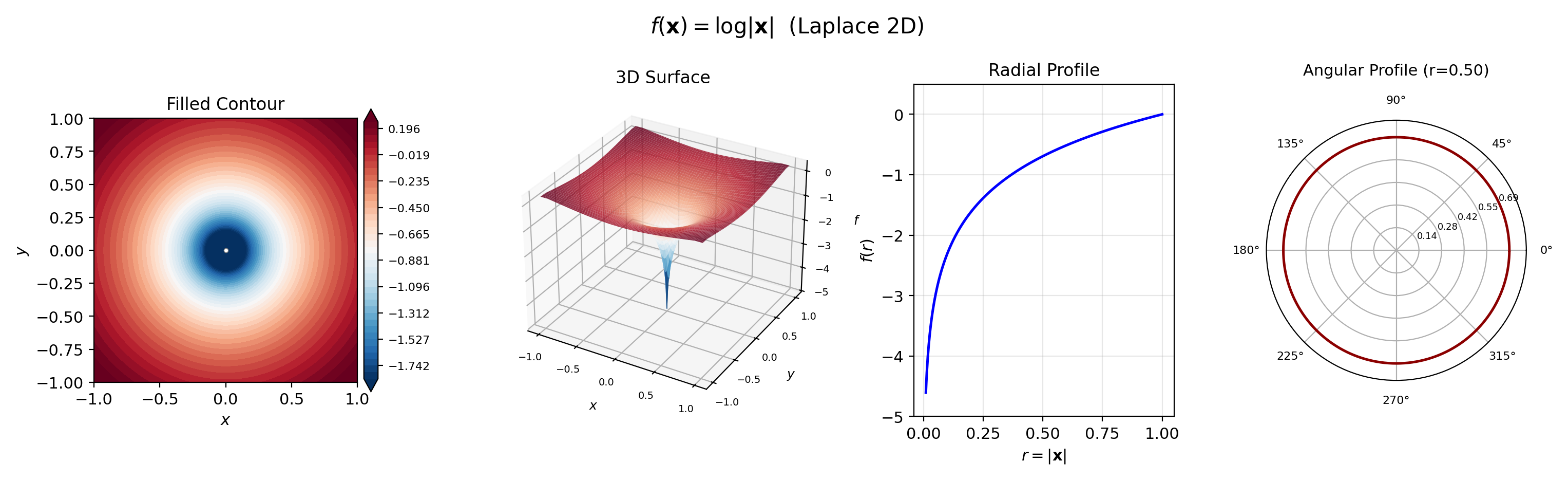}%
	}
	\\[0.5em]
	\subfloat[2D $r^{-1}$]{%
		\includegraphics[width=0.85\textwidth]{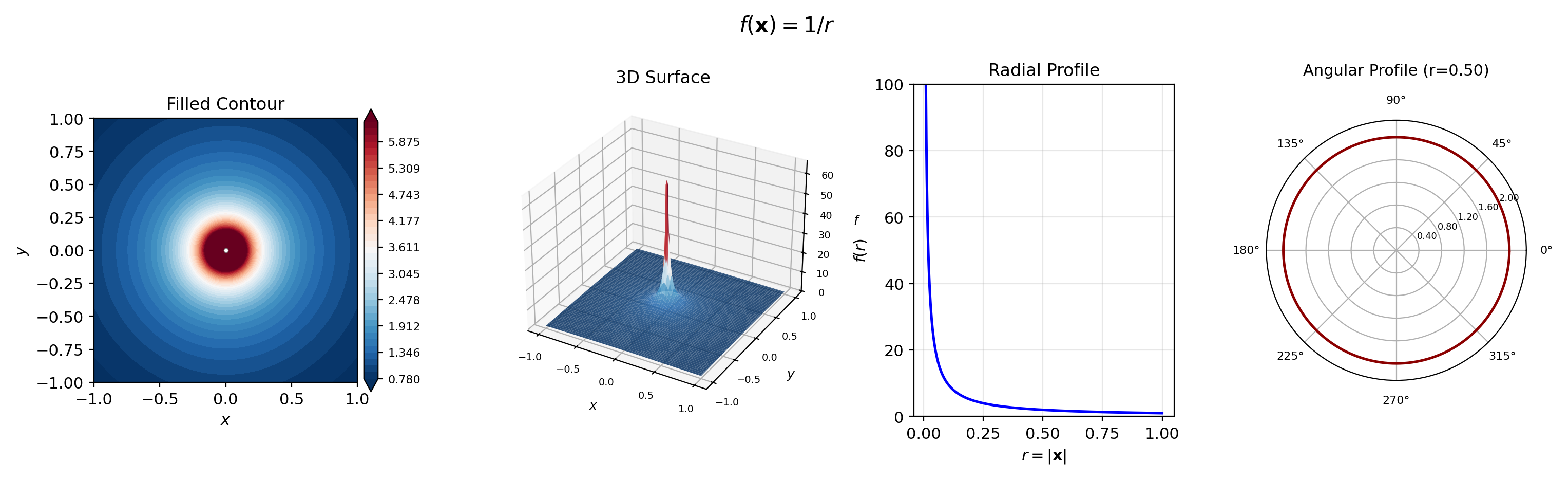}%
	}
	\\[0.5em]
	\subfloat[2D $r^{1/2}$]{%
		\includegraphics[width=0.85\textwidth]{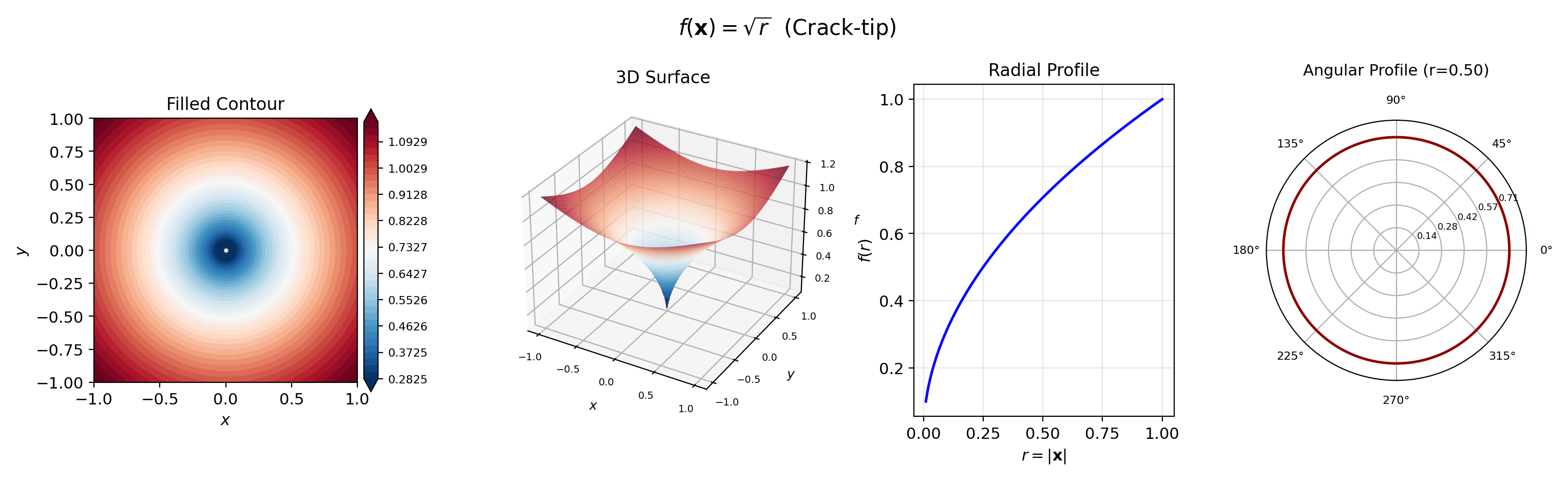}%
	}
	\\[0.5em]
	\subfloat[2D crack-tip]{%
		\includegraphics[width=0.85\textwidth]{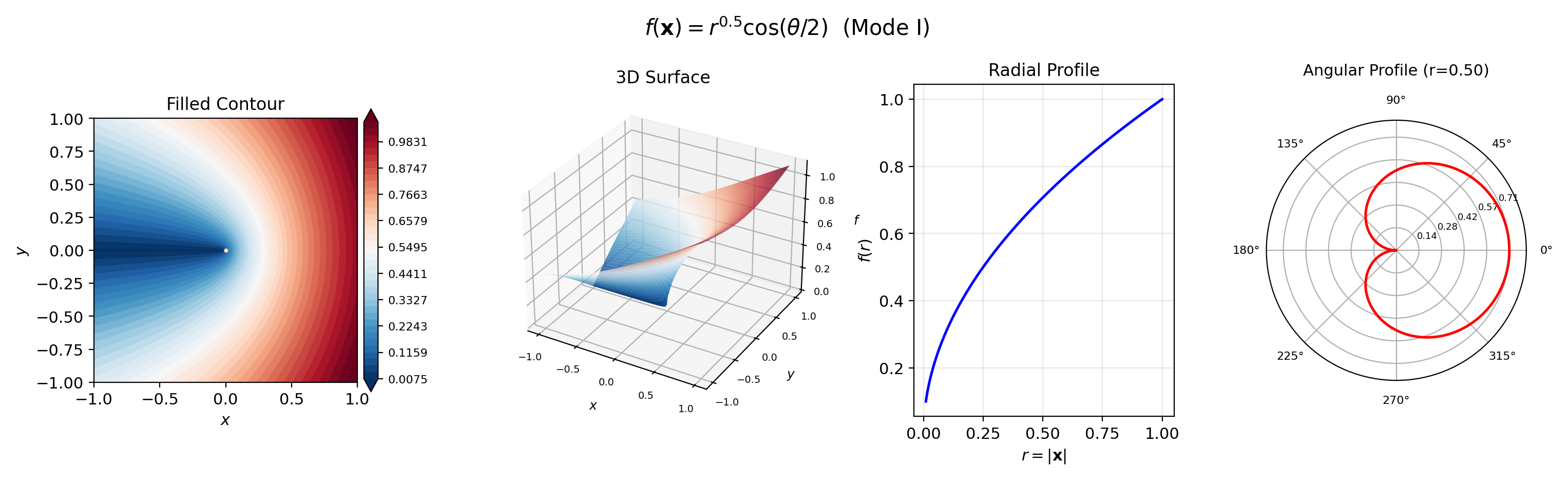}%
	}
	\caption{\small{Target function visualizations. Each panel shows training points (filled circles), test points (open circles), and function values (color).}}
	\label{fig:targets_appendix}
\end{figure}

\subsection{Additional Ablations}

Figure~\ref{fig:ablations_appendix} reports additional ablations on the angular-mode budget and the number of centers used in RMN variants.

\begin{figure}[!htbp]
	\centering
	\subfloat[Angular modes for RMN-Angular, with $M_{\max}$ in 2D and $L_{\max}$ in 3D]{%
		\includegraphics[width=0.95\textwidth]{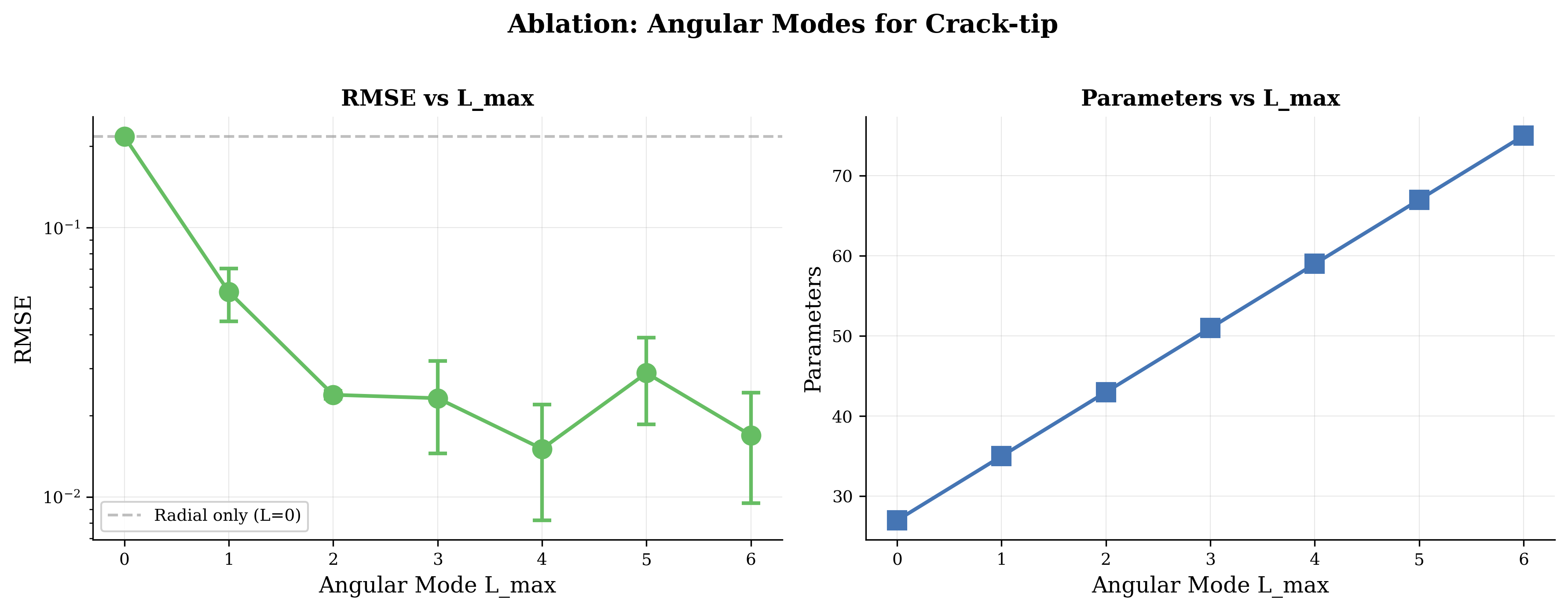}%
	}
	\\[0.5em]
	\subfloat[Number of centers $J$ for RMN-MC]{%
		\includegraphics[width=0.95\textwidth]{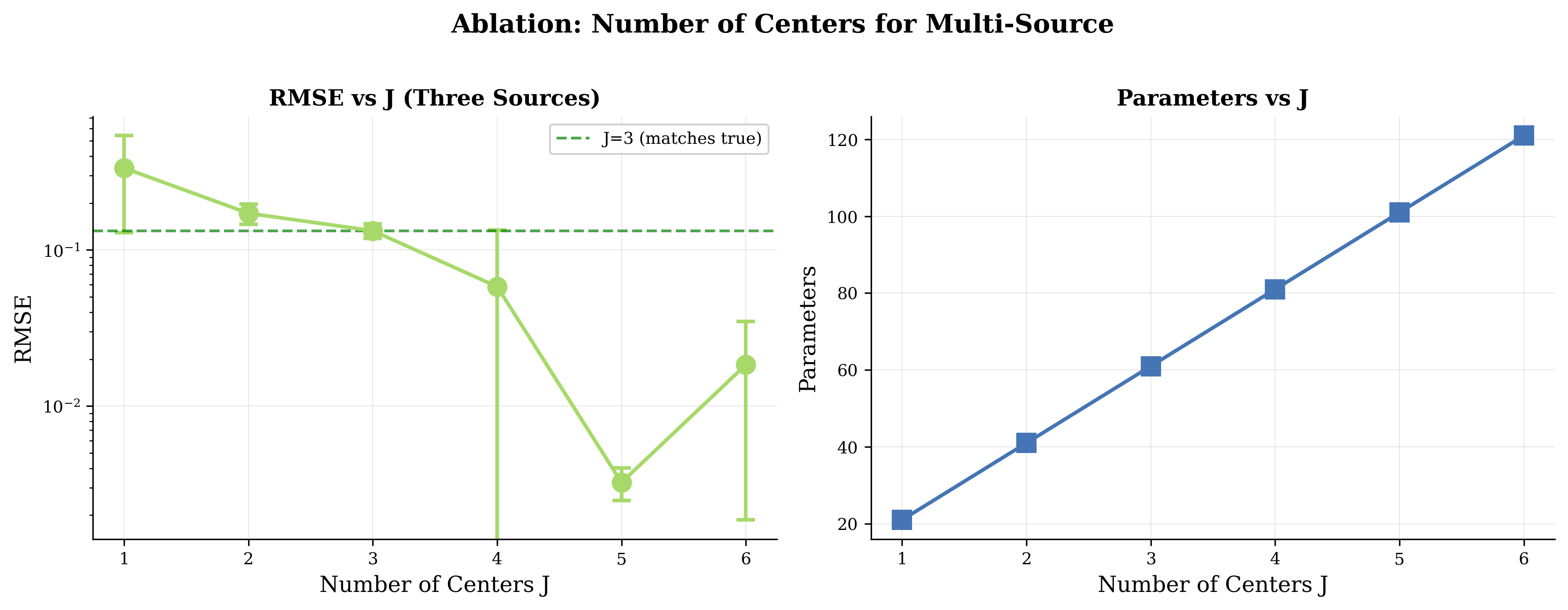}%
	}
	\caption{Extension ablations. Panel (a): adding angular modes improves crack-tip performance. Panel (b): more centers improve multi-source approximation.}
	\label{fig:ablations_appendix}
\end{figure}

\section{Additional Experiments: 3D Poisson with a Point Charge under Dirichlet Boundary Conditions}
\label{sec:exp3d}

We include this experiment as a supplementary evaluation of RMN in a bounded-domain, physics-informed setting. We study the 3D Poisson equation on $\Omega = [-1,1]^3$ with a single point charge and homogeneous Dirichlet boundary conditions. In contrast to the free-space Coulomb benchmark in Section~\ref{subsec:benchmarks}, the boundary breaks radial symmetry and induces a non-radial harmonic correction. We therefore treat this experiment primarily as a stress test of PINN stability and physics constraint satisfaction, rather than as a setting where a purely radial model should dominate global $L^2$ error.

We consider two complementary training paradigms: supervised fitting, where models are trained directly on a numerically computed ground truth solution $u^\star$, and full PINN training, where models are trained using only physics-based losses: PDE residual, boundary conditions, and Gauss-flux constraints, without access to $u^\star$ during training. Unlike the function-fitting benchmark suite in Section~\ref{subsec:training_protocol}, this bounded-domain PINN experiment uses larger collocation sets and a smaller exponent budget $K=6$ for training stability. All settings are listed in Table~\ref{tab:exp3d_config}.

We compare supervised and full PINN training and report relative $L^2$ error and Gauss-flux error. RMN attains competitive mean accuracy with a tiny parameter budget and remains stable under physics-only training, while a tuned wide MLP can achieve lower best-case supervised error. Table~\ref{tab:exp3d_extended} summarizes the quantitative results.

\subsection{Problem Setup and Punctured-Domain Formulation}

We consider the 3D Poisson equation on the cube domain $\Omega = [-1,1]^3$ with $J$ point charges:
\begin{equation}
	-\Delta u(\bx) = \sum_{j=1}^{J} q_j \delta(\bx - \bc_j), \bx \in \Omega,
	\label{eq:poisson3d}
\end{equation}
where $\bc_j \in \Omega$ are the charge locations and $q_j$ are the charge magnitudes. The fundamental solution for a single point charge in 3D is the Coulomb potential:
\begin{equation}
	G(\bx) = \frac{q}{4\pi\norm{\bx - \bc}} \propto \norm{\bx - \bc}^{-1},
\end{equation}
which exhibits the characteristic $r^{-1}$ singularity that RMN is designed to capture through learnable negative exponents.

\subsubsection{Boundary Conditions: Dirichlet}

We focus on the grounded Dirichlet boundary condition:
\begin{equation}
	u(\bx) = 0, \bx \in \partial\Omega.
\end{equation}
Other boundary types such as Neumann or mixed conditions can be handled by modifying the boundary loss. For Neumann boundaries, add a gauge constraint. These cases are beyond the scope of this experiment.

\subsubsection{Punctured Domain Formulation}

Because the Dirac delta function cannot be directly evaluated in neural network training, we adopt the punctured domain approach. We exclude small balls $B_\varepsilon(\bc_j)$ of radius $\varepsilon > 0$ around each charge center, yielding the punctured domain:
\begin{equation}
	\Omega_\varepsilon = \Omega \setminus \bigcup_{j=1}^{J} B_\varepsilon(\bc_j).
\end{equation}

On this domain, the PDE becomes source-free:
\begin{equation}
	\Delta u(\bx) = 0, \bx \in \Omega_\varepsilon.
\end{equation}

The point charge effect is encoded via Gauss's law as a flux constraint on each puncture boundary:
\begin{equation}
	\oint_{\partial B_\varepsilon(\bc_j)} \nabla u \cdot \mathbf{n} dS = -q_j.
	\label{eq:gauss_flux}
\end{equation}

This formulation is mathematically rigorous and directly applicable to physics-informed neural network training.

\subsection{Experimental Protocol}

We describe the training objectives, sampling strategies, optimization details, model configurations, and evaluation metrics used in this experiment.

\subsubsection{Training Objectives for Supervised and Full PINN}

\paragraph{Loss Functions: Supervised and PINN}

Supervised fitting minimizes MSE to the ground truth solution $u^\star$ evaluated at training points:
\begin{equation}
	\mathcal{L}_{\text{sup}} = \E_{\bx \sim \Omega_\varepsilon}\left[(u_\theta(\bx) - u^\star(\bx))^2\right].
\end{equation}

Full PINN training combines three components:
\begin{equation}
	\mathcal{L}_{\text{pinn}} = \lambda_{\text{PDE}} \mathcal{L}_{\text{PDE}} + \lambda_{\text{BC}} \mathcal{L}_{\text{BC}} + \lambda_{\text{flux}} \mathcal{L}_{\text{flux}}.
\end{equation}

PDE residual loss is
\begin{equation}
	\mathcal{L}_{\text{PDE}} = \E_{\bx \sim \Omega_\varepsilon}\left[(\Delta u_\theta(\bx))^2\right].
\end{equation}

Boundary condition (BC) loss for Dirichlet BC is
\begin{equation}
	\mathcal{L}_{\text{BC}} = \E_{\bx \sim \partial\Omega}\left[u_\theta(\bx)^2\right].
\end{equation}

Gauss-flux loss is
\begin{equation}
	\mathcal{L}_{\text{flux}} = \sum_{j=1}^{J} \left( 4\pi\varepsilon^2 \cdot \frac{1}{N_{\text{sphere}}}\sum_{i=1}^{N_{\text{sphere}}} \nabla u_\theta(\bx_{j,i}) \cdot \mathbf{n}_{j,i} + q_j \right)^2,
\end{equation}
where $\{\bx_{j,i}\}$ are $N_{\text{sphere}}$ points uniformly sampled on $\partial B_\varepsilon(\bc_j)$ using the Fibonacci sphere method, and $\mathbf{n}_{j,i}$ are the corresponding outward unit normals. We use base loss weights $\lambda_{\text{PDE}} = 1$, $\lambda_{\text{BC}} = 200$, $\lambda_{\text{flux}} = 50$. The high boundary condition weight prioritizes satisfying the boundary early in training.

\paragraph{Two-Phase Curriculum Learning}
To improve convergence, we use a two-phase curriculum during an initial warmup stage, gradually ramping the PDE and flux weights to their base values. In Phase 1, covering 0--40\% of warmup, we initially reduce the PDE weight via a multiplier $s_{\text{PDE}}$ that ramps from $0.1$ to $1.0$, and ramp the flux multiplier $s_{\text{flux}}$ from $0.1$ to $0.5$. In Phase 2, covering 40--100\% of warmup, we use the full PDE weight $s_{\text{PDE}}=1$ and increase the flux multiplier to $s_{\text{flux}}: 0.5 \to 1.0$. After warmup, we fix all weights to their base values. This curriculum encourages the model to establish the boundary behavior before enforcing the interior PDE constraint.

\subsubsection{Sampling, Optimization, and Hyperparameters}

\paragraph{Sampling Strategy}

Our sampling strategy uses stratified sampling with explicit sphere sampling around the singularity. Each training stage uses:
$N_{\text{int}} = 30{,}000$ interior points uniformly sampled in $\Omega_\varepsilon$, $N_{\text{bc}} = 8{,}000$ boundary points uniformly sampled on $\partial\Omega$, and $N_{\text{sphere}} = 1{,}500$ sphere points per charge sampled on $\partial B_\varepsilon(\bc_j)$.
We resample the training points every 2,500 training iterations for variance reduction.

\paragraph{Optimization}

Training uses the Adam optimizer with cosine annealing:
RMN uses 25,000 steps with learning rate $10^{-2}$ and cosine annealing to $10^{-4}$. MLP uses 35,000 steps with learning rate $10^{-3}$ and cosine annealing to $10^{-5}$.
For stability, we apply gradient clipping with max norm 1.0, bound exponents to $[-1.0, 2.0]$ to prevent numerical overflow while covering Coulomb behavior, and resample training points every 2,500 iterations for variance reduction.

\paragraph{Baseline Choice} To reduce the chance that negative results for MLPs are due to under-capacity or under-training, we use a substantially wider network than in the function-fitting benchmarks in Section~\ref{par:baselines}. The MLP has five hidden layers with width 256. We also allocate more optimization steps to MLP, 35,000 instead of 25,000.

\subsubsection{Models and Baselines}

\paragraph{RMN-MC Architecture and Wide MLP Baseline}

For this experiment, we use RMN-MC in Section~\ref{subsec:rmn_multicenter} with a single fixed center, $J=1$:
\begin{equation}
	u_\theta(\bx) = \sum_{j=1}^{J} \sum_{k=1}^{K} a_{jk} \norm{\bx - \bc_j}^{\mu_k} + b_0.
\end{equation}

The shared learnable exponents $\{\mu_k\}_{k=1}^K$ lie in $[-1.0, 2.0]$ and include the critical $\mu = -1$ for Coulomb behavior. The per-center coefficients $\{a_{jk}\}$ are initialized to approximate $1/(4\pi) \approx 0.08$ for the first component. Centers $\bc_j$ are fixed for the forward problem or learnable for the inverse problem. We use $K = 6$ exponent terms for sufficient flexibility.

This architecture has only 13 parameters for $J=1$. That count equals $K=6$ exponents, six coefficients, and one bias. The baseline MLP with five hidden layers of width 256 has 264,449 parameters. Unlike the function-fitting MLP baseline in Section~\ref{par:baselines}, we use a wider MLP here to better reflect common PINN practice in bounded-domain PDE settings.

Analytical derivatives: for RMN, we use closed-form gradients and Laplacians. For $r_j(\bx) = \norm{\bx - \bc_j}$:
\begin{equation}
	\nabla_{\bx} r_j^\mu = \mu r_j^{\mu-2} (\bx - \bc_j), \Delta_{\bx} r_j^\mu = \mu(\mu + d - 2) r_j^{\mu-2}.
\end{equation}
Here $d=3$ is the spatial dimension.
This improves both computational efficiency and numerical stability compared to automatic differentiation.

\subsubsection{Ground Truth and Evaluation Metrics}

The analytical solution for the bounded domain problem uses singular-smooth decomposition:
\begin{equation}
	u(\bx) = u_{\text{sing}}(\bx) + v(\bx),
\end{equation}
where $u_{\text{sing}}(\bx) = \sum_j \frac{q_j}{4\pi\norm{\bx - \bc_j}}$ is the free-space singular part, and $v(\bx)$ is a smooth correction satisfying:
\begin{equation}
	\Delta v = 0 \text{ in } \Omega, v|_{\partial\Omega} = -u_{\text{sing}}|_{\partial\Omega}.
\end{equation}

We compute $v$ numerically using finite differences on a refined grid with $N = 128^3$ and trilinear interpolation for point evaluation. This approach ensures accurate ground truth near singularities by handling the singular part analytically.

\begin{table}[t]
	\centering
	\begin{tabular}{lc}
		\toprule
		Parameter                  & Value                                  \\
		\midrule
		Domain                     & $[-1, 1]^3$                            \\
		BC type                    & Dirichlet, $u = 0$ on $\partial\Omega$ \\
		Number of charges          & $J = 1$                                \\
		Charge magnitude           & $q = 1$                                \\
		Puncture radius            & $\varepsilon = 0.08$                   \\
		Min distance from boundary & $d_{\min} = 0.15$                      \\
		\midrule
		Interior points            & 30,000                                 \\
		Boundary points            & 8,000                                  \\
		Sphere points              & 1,500                                  \\
		\midrule
		RMN training steps         & 25,000                                 \\
		MLP training steps         & 35,000                                 \\
		Random seeds               & 5 for statistical reliability          \\
		\bottomrule
	\end{tabular}
	\caption{Experimental configuration for 3D Poisson validation.}
	\label{tab:exp3d_config}

\end{table}

\paragraph{Evaluation Metrics}

We report relative $L^2$ error on held-out evaluation points,
\begin{equation}
	\mathrm{Rel}_{L^2} = \frac{\norm{u_\theta - u^\star}_2}{\norm{u^\star}_2},
\end{equation}
and flux error as the absolute Gauss-law violation on each puncture sphere,
\begin{equation}
	\mathrm{FluxErr}_j = \left|\oint_{\partial B_\varepsilon(\bc_j)} \nabla u_\theta \cdot \mathbf{n} dS + q_j\right|.
\end{equation}
In the full PINN setting, $u^\star$ is used only for reporting, not for training.

\subsection{Main Results: Supervised Compared with Full PINN}

We compare RMN and MLP across both supervised and full PINN training paradigms, then analyze the learned exponent spectra.

\subsubsection{Comprehensive Comparison Across Training Paradigms}

\begin{table}[t]
	\centering
	\resizebox{\textwidth}{!}{%
		\begin{tabular}{llccccc}
			\toprule
			Setting                     & Model       & Params      & Mean Rel $L^2$             & Best Rel $L^2$ & Flux Err                      & Time       \\
			\midrule
			\multirow{2}{*}{Supervised} & RMN ($K=6$) & \textbf{13} & $22.7 \pm 10.8\%$          & 8.8\%          & $5.3 \times 10^{-5}$          & $\sim$125s \\
			                            & MLP (tuned) & 264,449     & $22.7 \pm 13.9\%$          & \textbf{4.8\%} & $8.2 \times 10^{-4}$          & $\sim$600s \\
			\midrule
			\multirow{2}{*}{Full PINN}  & RMN ($K=6$) & \textbf{13} & $\mathbf{22.7 \pm 10.8\%}$ & \textbf{8.8\%} & $\mathbf{5.3 \times 10^{-5}}$ & $\sim$125s \\
			                            & MLP-Large   & 264,449     & $141.8 \pm 173\%$          & 28.6\%         & $1.5 \times 10^{-3}$          & $\sim$315s \\
			\bottomrule
		\end{tabular}}
	\caption{Comprehensive comparison across training paradigms. RMN maintains stable performance in both supervised and PINN settings, while MLP degrades significantly under PINN training. This robustness stems from RMN's physics-aligned inductive bias.}
	\label{tab:exp3d_extended}

\end{table}

Table~\ref{tab:exp3d_extended} consolidates results from both training paradigms, highlighting RMN's consistent performance across settings. RMN is largely insensitive to training paradigm and shows nearly unchanged performance within rounding error, whether trained with supervision or pure physics losses. This observation suggests that RMN's architecture captures a large fraction of the solution structure even without access to $u^\star$ during training. The wide MLP baseline drops from 22.7\% in the supervised setting to 141.8\% in the PINN setting. The high variance with std 173\% indicates frequent convergence failures, consistent with known PINN training pathologies near singularities \citep{wang2021understanding}. Under full PINN training, RMN achieves flux error $5.3 \times 10^{-5}$, compared with MLP's $1.5 \times 10^{-3}$, about $28\times$ lower, demonstrating superior physics consistency even when both methods achieve similar $L^2$ error. RMN runs in about 125s, and the supervised MLP run takes about 600s as shown in Table~\ref{tab:exp3d_extended}. The learned first exponent $\mu_1$ consistently converges to values in $[-0.88, -0.71]$, approaching the theoretical Coulomb value $\mu = -1$ and providing direct insight into solution structure.

\subsubsection{Learned Exponent Analysis}

\begin{table}[t]
	\centering
	\begin{tabular}{lcccccc}
		\toprule
		Seed    & $\mu_1$ & $\mu_2$ & $\mu_3$ & $\mu_4$ & $\mu_5$ & $\mu_6$ \\
		\midrule
		0, best & $-0.75$ & $-0.44$ & $-0.05$ & $0.45$  & $1.12$  & $2.00$  \\
		1       & $-0.88$ & $-0.70$ & $-0.38$ & $0.43$  & $1.26$  & $2.00$  \\
		2       & $-0.80$ & $-0.55$ & $-0.20$ & $0.30$  & $1.03$  & $2.00$  \\
		3       & $-0.70$ & $-0.36$ & $0.03$  & $0.54$  & $1.19$  & $2.00$  \\
		4       & $-0.87$ & $-0.65$ & $-0.18$ & $0.49$  & $1.28$  & $2.00$  \\
		\bottomrule
	\end{tabular}
	\caption{Learned exponents from RMN across 5 seeds.}
	\label{tab:learned_exponents}

\end{table}

The exponent spectra learned by RMN provide direct physical interpretability. The dominant exponent $\mu_1$ consistently lies in the range $[-0.88, -0.70]$, approaching the Coulomb value $\mu = -1$. The deviation from exactly $-1$ reflects interactions between the singularity and domain boundaries. Higher-order terms with $\mu > 0$ capture smooth boundary corrections, with $\mu_6 = 2.0$ fixed at the upper bound.

\subsubsection{Discussion}

These 3D Poisson experiments emphasize three points:
First, RMN remains stable under full PINN training and achieves low Gauss-flux error, indicating strong physics consistency even without access to $u^\star$. Second, in the supervised setting, a tuned wide MLP can achieve lower best-case relative $L^2$ error, reflecting the non-radial boundary correction induced by Dirichlet conditions. Third, the learned exponent spectrum reflects an interpretable singular-plus-smooth decomposition, with the dominant exponent approaching Coulomb-like behavior while positive exponents capture smooth boundary-induced corrections.
Quantitative results and timing are summarized in Table~\ref{tab:exp3d_extended}. Both methods show variance across charge positions; charges closer to boundaries or corners are more challenging.

\subsubsection{Summary}

The 3D Poisson experiments suggest that RMN-MC provides a compact and stable surrogate for bounded-domain point-charge problems, while also highlighting an important limitation: Dirichlet boundaries induce non-radial corrections, so a single-center radial model is not expected to dominate global accuracy without additional mechanisms for boundary adaptation.

Table~\ref{tab:exp3d_extended} summarizes the quantitative results. RMN maintains similar mean relative $L^2$ error across supervised and full PINN training while achieving substantially lower flux error and using far fewer parameters than the wide MLP baseline; the trade-off is that the best-case supervised error is higher than that of the tuned MLP.

Overall, these results suggest that RMN offers a useful trade-off between accuracy, stability, and interpretability for singularity-dominated physics problems, especially when parameter budgets are tight. In this setting, the $20{,}000\times$ parameter reduction comes with comparable mean accuracy and improved stability under our full PINN setup. RMN is therefore particularly attractive for physics-informed learning when ground truth is unavailable, for resource-constrained and real-time applications, and for scientific computing when interpretability is essential.

\subsection{Robustness and Interpretability}
\label{sec:variance_analysis}

The observed variability differs substantially between training paradigms. RMN shows consistent std $\approx 10.8\%$ in both settings, while MLP std increases from 13.9\% in the supervised setting to 173\% in the PINN setting. Across seeds, we sample the charge center uniformly from the interior cube $[-1 + d_{\min}, 1 - d_{\min}]^3$ to enforce a minimum distance $d_{\min}$ from the boundary, as listed in Table~\ref{tab:exp3d_config}.

\paragraph{Variance Sources} Charge position varies by seed, and charges closer to corners or multiple faces create steeper boundary-induced corrections that are more challenging. RMN local minima arise from the interplay between exponent and coefficient learning, yet the converged solutions remain physically reasonable and satisfy the Gauss-flux constraint to about $10^{-5}$. MLP PINN pathologies contribute to large variance under PINN training. This behavior is consistent with known gradient pathologies in PINNs near singularities: competing loss terms and large second-derivative targets can create stiff optimization dynamics and occasional divergence \citep{wang2021understanding}.

\paragraph{Practical Implications} For deployment, we recommend using RMN for PINN applications because its performance is consistent across seeds when ground truth is unavailable. We also recommend running multiple seeds and selecting based on physics constraints such as flux error rather than only $L^2$ error, avoiding charge placements very close to domain boundaries with $d_{\min} \geq 0.15$, and using supervised training when ground truth is available and MLP accuracy is critical.

\subsubsection{Learned Exponent Analysis: Physical Interpretation}
\label{sec:exponent_analysis}

The learned dominant exponent $\mu_1$ ranges from $-0.70$ to $-0.88$ across seeds, with mean $-0.80$, approaching but not exactly reaching the free-space Coulomb value $\mu = -1$. Importantly, this deviation is physically meaningful rather than a limitation. The primary cause is the bounded domain effect. The true solution on a bounded domain with Dirichlet BC is:
\begin{equation}
	u(\bx) = \underbrace{\frac{q}{4\pi\norm{\bx - \bc}}}_{\text{singular, } \sim r^{-1}} + \underbrace{v(\bx)}_{\text{smooth correction}},
\end{equation}
where $v(\bx)$ enforces $u|_{\partial\Omega} = 0$. This correction is not purely radial and depends on the charge position relative to boundaries. Consequently, the optimal single-center radial approximation has an effective exponent $\mu_{\text{eff}} \neq -1$ that balances the singular and smooth components. The learned $\mu_1 \approx -0.80$ is the optimal trade-off, not an approximation failure.

A secondary factor is the exponent parameterization. Exponents are parameterized via a smooth transformation mapping to $[-1, 2]$ with slight bias away from boundary values for numerical stability.

Physical interpretation shows that the learned exponent spectrum reveals the multi-scale solution structure. For seed 0, the spectrum $\{-0.75, -0.44, -0.05, 0.45, 1.12, 2.00\}$ illustrates this breakdown. The dominant term with $\mu_1 \approx -0.75$ captures Coulomb-like behavior. The terms with $\mu_2$ and $\mu_3$ near $-0.44$ and $-0.05$ capture additional near-field corrections induced by the bounded domain in an effective radial sense. The positive exponents $\mu_4$, $\mu_5$, and $\mu_6$ represent smooth polynomial terms for boundary layer corrections. This interpretable decomposition---unavailable from black-box MLPs---provides direct physical insight into the solution structure and could inform analytical approximation strategies.

\subsection{Supplementary Studies}
\label{sec:future_work}

\paragraph{Extended Ablation Studies: Full PINN Training} We conduct additional ablation studies using full PINN training, where models are trained using only physics-based losses such as the PDE residual, boundary conditions, and Gauss-flux constraints, and no ground truth is used during training. This setting is more challenging and practically relevant, as ground truth solutions are often unavailable for real-world problems.

\paragraph{Effect of Number of Exponents \texorpdfstring{($K$)}{(K)}} Table~\ref{tab:exp3d_ablation_K} compares RMN with $K=6$ and $K=8$ exponent terms under full PINN training. Increasing from $K=6$ to $K=8$ provides no significant improvement, with mean error 22.7\% compared with 23.1\%, suggesting that six exponent terms are sufficient for single-charge problems. RMN maintains high parameter efficiency with only 13--17 parameters compared to 264,449 for MLP\@. Both $K=6$ and $K=8$ achieve flux errors of order $10^{-5}$, demonstrating robust physics constraint satisfaction. MLP shows high variance with std 173\% and frequent training failures in the full PINN setting. This instability arises from competing objectives in the multi-component loss: the PDE residual requires accurate second derivatives near a singularity, while the flux constraint requires precise gradient integration on small spheres. In contrast, RMN contains the harmonic radial term $r^{-1}$ as a special case in 3D since $\Delta r^{-1} = 0$ away from the center, which provides a physics-aligned inductive bias for stable optimization.

\begin{table}[t]
	\centering
	\begin{tabular}{lcccc}
		\toprule
		Model       & Parameters  & Rel $L^2$ (\%)           & Flux Error                    & Best Rel $L^2$ \\
		\midrule
		RMN ($K=6$) & \textbf{13} & $\mathbf{22.7 \pm 10.8}$ & $\mathbf{5.3 \times 10^{-5}}$ & \textbf{8.8\%} \\
		RMN ($K=8$) & 17          & $23.1 \pm 10.7$          & $6.2 \times 10^{-5}$          & 8.9\%          \\
		MLP-Large   & 264,449     & $141.8 \pm 173$          & $1.5 \times 10^{-3}$          & 28.6\%         \\
		\bottomrule
	\end{tabular}
	\caption{Ablation: Number of exponent terms for full PINN training. Models trained with physics-based losses only, without ground truth supervision. Mean $\pm$ std over 5 seeds. The contrast with supervised training in Table~\ref{tab:exp3d_extended} highlights the sensitivity of the wide MLP baseline under physics-only optimization.}
	\label{tab:exp3d_ablation_K}
\end{table}

\section{Supplementary Two-Dimensional Experiments: Annulus Domain}
\label{sec:app_exp2d}

We validate RMN on a punctured annulus where the targets are purely radial and non-polynomial logarithmic and fractional-power profiles. Although the domain excludes the origin, these prototypes are singular at $r=0$ and exhibit steep variation as $r$ approaches the inner boundary. This provides a controlled setting to test whether a radial power basis remains advantageous under a different sampling and training recipe. It also re-confirms the separability obstruction for coordinate-wise power models in a clean radial setting.

This annulus study is an additional validation run under a separate and lighter experimental setup than the main benchmark suite. It uses a larger inner radius with $R_{\text{in}}=0.1$, fewer training points, and a compact MLP baseline with the same parameter budget as SIREN. The intent is not to improve state-of-the-art numbers, but to check that the ranking and qualitative conclusions persist under altered design choices. The RMSE values in this section should therefore be interpreted within this setup rather than compared numerically across sections.

\subsection{Experimental Setup}

\paragraph{Domains}
We consider the annulus domain $\Omega_A = \{(x,y) : R_{\text{in}} \leq \sqrt{x^2+y^2} \leq R_{\text{out}}\}$ with $R_{\text{in}}=0.1$ and $R_{\text{out}}=1.0$, which naturally admits solutions with radial structure on a punctured domain where $r \ge R_{\text{in}}$.

\paragraph{Target Functions}
We construct analytical target functions representing various physical phenomena:

\paragraph{Annulus Targets}
We use pure radial functions:
\begin{align}
	f_{\log}(r,\theta) & = \log(r), \label{eq:app_a1_log}   \\
	f_{1/2}(r,\theta)  & = r^{1/2}, \label{eq:app_a1_sqrt}  \\
	f_{0.7}(r,\theta)  & = r^{0.7}. \label{eq:app_a1_power}
\end{align}
These targets isolate the core RMN capability of learning radial structure with log and fractional powers in a compact and interpretable basis.

\paragraph{Models}
We compare five architectures: RMN has 27 parameters and uses a radial M\"untz-Sz\'asz basis with $K=12$ learnable exponents $\{\mu_k\}_{k=1}^{12}$ and linear combination coefficients. RMN-Angular has 67 parameters and extends RMN with angular harmonics to capture $\cos(n\theta), \sin(n\theta)$ structure using the 2D Fourier cutoff $M_{\max}=5$. The compact MLP has 8,577 parameters and uses 3 hidden layers with 64 units and Tanh activations. SIREN has 8,577 parameters and uses $\omega_0=30$ as in \citet{sitzmann2020siren}. MSN has 49 parameters and uses the coordinate-separable structure $\sum_k a_k x^{\mu_k} + b_k y^{\nu_k}$ as the separability obstruction baseline.

\paragraph{Training Protocol}
All models are trained via supervised fitting using the Adam optimizer for 10,000 steps with base learning rate $2\times 10^{-3}$ and exponent-parameter learning rate $1\times 10^{-4}$ for models with learnable exponents. Training uses 4,096 interior points with importance weighting toward the inner boundary where $r$ is small. Evaluation uses fixed test grids: 4,000 points for the annulus domain on a $50\times 80$ grid. Results are averaged over 5 random seeds, with standard deviations reported.

\subsection{Results on Annulus Domain}

Table~\ref{tab:app_annulus_results} summarizes the root mean squared error (RMSE) on the 4,000-point annulus test grid. We report mean $\pm$ standard deviation over 5 random seeds.

\begin{table}[t]
	\centering
	\resizebox{\textwidth}{!}{%
		\begin{tabular}{lccccc}
			\toprule
			Target & RMN (27)                       & RMN-Ang (67)                & MLP (8577)                     & SIREN (8577)       & MSN (49)           \\
			\midrule
			\multicolumn{6}{l}{Pure radial functions}                                                                                                        \\
				log    & $3.50\times 10^{-3} \pm 0.4\times 10^{-3}$              & $\mathbf{1.42\times 10^{-3}} \pm 1.0\times 10^{-3}$ & $\underline{2.22\times 10^{-3}} \pm 0.2\times 10^{-3}$ & $4.14\times 10^{-2} \pm 2.3\times 10^{-2}$ & $2.29\times 10^{-1} \pm 0.9\times 10^{-3}$ \\
				sqrt   & $\underline{1.90\times 10^{-3}} \pm 0.5\times 10^{-3}$  & $\mathbf{9.10\times 10^{-4}} \pm 0.2\times 10^{-3}$ & $2.07\times 10^{-3} \pm 0.2\times 10^{-3}$             & $4.95\times 10^{-2} \pm 3.5\times 10^{-2}$ & $5.28\times 10^{-2} \pm 0.2\times 10^{-3}$ \\
				power  & $\underline{1.99\times 10^{-3}} \pm 0.3\times 10^{-3}$  & $\mathbf{1.68\times 10^{-3}} \pm 0.3\times 10^{-3}$ & $2.23\times 10^{-3} \pm 0.3\times 10^{-3}$             & $4.64\times 10^{-2} \pm 0.9\times 10^{-2}$ & $5.28\times 10^{-2} \pm 0.2\times 10^{-3}$ \\
			\bottomrule
		\end{tabular}}
	\caption{Supplementary annulus study: RMSE comparison on radial targets. Values show mean $\pm$ std over 5 seeds. Best results are in bold and second-best results are underlined. Parameter counts shown in parentheses.}
	\label{tab:app_annulus_results}

\end{table}

In this annulus setup, RMN-Angular achieves the lowest mean RMSE on all three radial targets and uses only 67 parameters. Compared with the compact MLP baseline, it uses about 128$\times$ fewer parameters. The improvements over the compact MLP are modest but consistent. The improvement factors are $1.56\times$ on log with RMN-Angular at $1.42\times 10^{-3}$ and MLP at $2.22\times 10^{-3}$, $2.27\times$ on sqrt with RMN-Angular at $9.10\times 10^{-4}$ and MLP at $2.07\times 10^{-3}$, and $1.32\times$ on power with RMN-Angular at $1.68\times 10^{-3}$ and MLP at $2.23\times 10^{-3}$. RMN-Direct is consistently second-best or close, indicating that the core radial basis already captures most of the benefit in a minimal 27-parameter form.

MSN, which uses coordinate-separable M\"untz bases $\sum_k a_k x^{\mu_k} + b_k y^{\nu_k}$, fails severely on radial targets. On log, MSN gives $2.29\times 10^{-1}$ and RMN-Angular gives $1.42\times 10^{-3}$, so MSN is $161\times$ worse. On sqrt, MSN gives $5.28\times 10^{-2}$ and RMN-Angular gives $9.10\times 10^{-4}$, so MSN is $58\times$ worse. On power, MSN gives $5.28\times 10^{-2}$ and RMN-Angular gives $1.68\times 10^{-3}$, so MSN is $31\times$ worse. This empirically validates Theorem~\ref{thm:separability}: coordinate-separable functions cannot efficiently represent radial structure $f(r)$ where $r = \sqrt{x^2+y^2}$.

\subsection{Learned Exponent Snapshots}

To keep this section numerically consistent with the released run artifacts, Table~\ref{tab:app_annulus_learned_exponents} reports representative values from the best-RMSE seed for each target under RMN. The learned exponents place basis elements near expected singular orders, while the multi-seed standard deviations in Table~\ref{tab:app_annulus_results} indicate non-negligible optimization variability. These snapshots should be interpreted as illustrative rather than exact exponent recovery for every seed and target.

\begin{table}[t]
	\centering
	\begin{tabular}{lcc}
		\toprule
		Target      & Expected Singular Order & Representative Learned Value    \\
		\midrule
		log (RMN)   & $\mu_{\log}=0$          & $\mu_{\log}=7.64\times 10^{-3}$ \\
		sqrt (RMN)  & $\mu=0.5$               & closest $\mu=0.485$             \\
		power (RMN) & $\mu=0.7$               & closest $\mu=0.640$             \\
		\bottomrule
	\end{tabular}
	\caption{Representative learned exponents from best-RMSE seeds}
	\label{tab:app_annulus_learned_exponents}

\end{table}

\subsection{Discussion}

These experiments validate the central thesis of this work: neural network architectures should be designed to match the mathematical structure of the target function class. When solutions exhibit radial singularities with non-integer power-law behavior, the RMN architecture achieves superior accuracy with dramatically fewer parameters compared to generic approximators.

\paragraph{Practical Implications}
For these pure radial targets, RMN-Direct already provides strong accuracy with only 27 parameters, while RMN-Angular offers small additional gains in this supplementary setup at 67 parameters. Both variants remain dramatically smaller than the compact MLP and SIREN baselines, each with 8,577 parameters.

\paragraph{Separability Obstruction}
The consistently poor performance with a coordinate-separable M\"untz basis across all targets empirically confirms the theoretical separability obstruction: radial functions $f(r) = f(\sqrt{x^2+y^2})$ cannot be efficiently represented by $\sum_k a_k x^{\mu_k} + b_k y^{\nu_k}$. This validates the necessity of the radial coordinate transformation in RMN.

\bibliography{references}

\end{document}